\documentclass[twoside,11pt]{article}

\usepackage{blindtext}

\usepackage{jmlr2e}

\usepackage{lastpage}
\jmlrheading{23}{2024}{}{1/21; Revised 5/22}{9/22}{21-0000}{Ahuja and Mansouri}

\ShortHeadings{On Provable Length and Compositional Generalization}{Ahuja and Mansouri}
\firstpageno{1}
\usepackage{xcolor}
\definecolor{mydarkblue}{rgb}{0,0.3,0.5}
\definecolor{lightblue}{rgb}{0.9, 0.95, 1.0}

\usepackage[normalem]{ulem}

\usepackage[utf8]{inputenc} % allow utf-8 input
\usepackage[T1]{fontenc}    % use 8-bit T1 fonts
\usepackage{hyperref}       % hyperlinks
\usepackage{url}            % simple URL typesetting
\usepackage{booktabs}       % professional-quality tables

\usepackage{colortbl}

\usepackage{amsfonts}       % blackboard math symbols
\usepackage{nicefrac}       % compact symbols for 1/2, etc.
\usepackage{microtype}      % microtypography
\usepackage{lipsum}
\usepackage[linesnumbered,ruled,vlined]{algorithm2e}
\usepackage{algorithmic}

\SetCommentSty{mycommfont}
\SetKwInput{KwInput}{Input}                % Set the Input
\SetKwInput{KwOutput}{Output}  
\SetKwInput{KwInitialization}{Initialization}
\usepackage{float}
\usepackage{caption}
\usepackage{subcaption}
\usepackage{mathtools}
\usepackage{soul}
\usepackage{enumitem}
\usepackage{amssymb}
\usepackage{tikz}
\usepackage{bm}
\usetikzlibrary{arrows}
\usepackage{dsfont}
\usepackage{graphicx, graphics}
\usepackage{wrapfig}
\usepackage{thmtools,thm-restate}

\usepackage{cleveref}
\usepackage{caption}
\usepackage{subcaption}

\newtheorem{assumption}{Assumption}
\DeclareMathAlphabet{\mathbsf}{OT1}{cmss}{bx}{n}% bold sans serif
\DeclareMathAlphabet{\mathssf}{OT1}{cmss}{m}{sl}% slanted sans serif
\crefname{lemma}{lemma}{lemmas}
\Crefname{lemma}{Lemma}{Lemmas}
\crefname{thm}{theorem}{theorems}
\Crefname{thm}{Theorem}{Theorems}
\crefname{prop}{proposition}{propositions}
\Crefname{prop}{Proposition}{Propositions}
\crefname{defn}{definition}{definitions}
\Crefname{defn}{Definition}{Definitions}
\creflabelformat{equation}{#1#2#3}
\crefname{equation}{equation}{equations}
\Crefname{equation}{Equation}{Equations}
\Crefname{section}{Section}{Sections}
\Crefname{appendix}{Appendix}{Appendices}
\crefname{figure}{figure}{figures}
\Crefname{figure}{Figure}{Figures}
\crefname{algorithm}{algorithm}{algorithms}
\Crefname{algorithm}{Algorithm}{Algorithms}
\crefname{assumption}{assumption}{assumptions}
\Crefname{assumption}{Assumption}{Assumptions}

\DeclareMathOperator*{\argmin}{arg\,min}
\usepackage[]{mdframed}

\begin{document}

\title{On Provable Length and Compositional Generalization}

\author{\name Kartik Ahuja \email kartikahuja@meta.com \\
       \addr FAIR (Meta)\\
       \AND
       \name Amin Mansouri \email amin.mansouri@mila.quebec \\
       \addr EPFL and Mila-Quebec AI Institute\\}

\editor{}

\maketitle

\begin{abstract}%   <- trailing '%' for backward compatibility of .sty file
Out-of-distribution generalization capabilities of sequence-to-sequence models can be studied from the lens of two crucial forms of generalization: length generalization --  the ability to generalize to longer sequences than ones seen during training, and compositional generalization: the ability to generalize to token combinations not seen during training. In this work, we provide first provable guarantees on length and compositional generalization for common sequence-to-sequence models -- deep sets, transformers, state space models, and recurrent neural nets -- trained to minimize the prediction error.  We show that \emph{limited capacity} versions of these different architectures achieve  both length and compositional generalization provided the training distribution is sufficiently diverse. In the first part, we study structured limited capacity variants of different architectures and arrive at the generalization guarantees with limited diversity requirements on the training distribution. In the second part, we study limited capacity variants with less structural assumptions and arrive at generalization guarantees but with more diversity requirements on the training distribution. Further, we also show that chain-of-thought supervision enables length generalization in higher capacity counterparts of the different architectures we study.
\end{abstract}


\section{ Introduction}

Large language models (LLMs), such as the GPT models \citep{achiam2023gpt} and the Llama models \citep{touvron2023llama}, have led to a paradigm shift in the development of future artificial intelligence (AI) systems. The accounts of their successes \citep{bubeck2023sparks, gunasekar2023textbooks} as well as their failures, particularly in reasoning and planning \citep{bubeck2023sparks, stechly2023gpt, valmeekam2023can}, continue to rise. The successes and failures of these models have sparked a debate about whether they actually learn general algorithms or if their success is mainly due to memorization and a superficial form of generalization \citep{dziri2024faith}.

The ability of a model to perform well across different distribution shifts highlights that it learns general algorithms. For models with fixed-dimensional inputs, considerable effort has led to methods with provable out-of-distribution (OOD) generalization guarantees \citep{rojas2018invariant,  chaudhuri2023does, wiedemer2023compositional, eastwood2024spuriosity}. For sequence-to-sequence models,  a large body of empirical works has investigated OOD generalization \citep{anil2022exploring, jelassi2023length}, but we lack efforts to study provable OOD generalization guarantees for these models. These provable guarantees provide a stepping stone towards explaining the success of the existing paradigm and also shine a light on where the existing paradigm fails.

OOD generalization capabilities of sequence-to-sequence models can be studied from the lens of two forms of generalization: length generalization --  the ability to generalize to longer sequences than ones seen during training, and compositional generalization -- the ability to generalize to token combinations not seen during training.  Although transformers \citep{vaswani2017attention} are the go-to sequence-to-sequence models for many applications, recently alternative architectures based on state-space models, as noted by \citet{gu2021efficiently}, \citet{orvieto2023resurrecting}, and \citet{gu2023mamba}, have shown a lot of promise. This motivates us to study a range of natural sequence-to-sequence hypothesis classes, including deep sets \citep{zaheer2017deep}, transformers, state space models (SSMs) and recurrent neural networks (RNNs). We divide the study into two parts: a) hypothesis classes with structural assumptions and b) general hypothesis classes with less structural assumptions. We summarize our main contributions as follows.

\paragraph{\textbf{Hypothesis classes with structural assumptions}}

\begin{itemize}
\item  For limited capacity versions of deep sets, transformers, SSMs, and RNNs with structural assumptions,  we find that a standard expected risk minimization-based learner provably achieves length and compositional generalization provided the training distribution satisfies certain diversity conditions.

\item Across different architectures, we find that the learned representations are linearly related to the representations generated by the true labeling function, which is also termed \emph{linear identification} \citep{khemakhem2020variational, roeder2021linear}.

 \item Through a range of experiments, we show the success in both forms of generalization, matching the predictions of the theory and even going beyond.
 
 \item For high capacity versions of deep sets, transformers, SSMs, and RNNs, we also show that chain-of-thought (CoT) supervision enables length generalization.
\end{itemize}

\paragraph{\textbf{General hypothesis classes with less structural assumptions}}

\begin{itemize}
\item For general hypothesis classes, we only focus on length generalization guarantees. We start with finite hypothesis classes and show that standard expected risk minimization based learner provably achieves length generalization provided the training sequences are sufficiently long.  

\item We prove that for hypothesis classes that are Lipschitz continuous with a Lipschitz constant that is independent of sequence length, approximate length generalization is achievable, provided training sequences are sufficiently long. We show that under certain parametric constraints, multi-block deep sets, transformers, RNNs, and SSMs satisfy the above stated Lipschitz condition.  To achieve length generalization, we use a constrained learner that seeks a model that simultaneously satisfies $\epsilon$ optimality at all sequence lengths in training. This learner closely resembles collaborative learning models from \citep{blum2017collaborative} and is also similar to group-DRO \citep{sagawa2019distributionally}.

\item  The results for general hypothesis classes capture more expressive classes than structured hypothesis classes, but they require access to more diverse training distributions. Although the constrained learner is learned via a procedure that is computationally not tractable, it suggests that there are other learning procedures possibly based on group-DRO that are both computationally tractable and better than expected risk minimization at length generalization. We believe this is an exciting future work.
\end{itemize}

  To the best of our knowledge, our provable guarantees for length and compositional generalization for sequence-to-sequence models are the first in the literature. The following sections are organized as follows.  In Section~\ref{sec: related_works}, we discuss the body of related works on length generalization and compositional generalization. In Section~\ref{sec: lgcg_framework}, we provide the basic framework, the definitions of both the length generalization and the compositional generalization, and the learning objective. In Section~\ref{sec: shc}, we operate with structured hypothesis classes and develop length and compositional generalization guarantees for different architectures -- deep sets, transformers, SSMs and RNNs. We also carried out experiments to verify the predictions of the theory. We further extend our results to incorporate CoT supervision. In Section~\ref{sec: ghc}, we develop results for hypothesis classes with fewer structural assumptions. Finally, in Section~\ref{sec: discussion}, we discuss the the limitations and important potential future works.

\section{Related Works} 
\label{sec: related_works}

\paragraph{Length generalization} 

Over the years, many important empirical insights on length generalization have come to the fore.  \cite{shaw2018self} discovered the drawbacks of absolute positional embeddings and suggested relative positional embeddings as an alternative. Subsequent empirical analysis, notably by \cite{anil2022exploring} and \cite{jelassi2023length}, explored length generalization in different settings for transformer-based models. Key findings revealed that larger model sizes do not necessarily enhance generalization and the effectiveness of relative positional embeddings appeared to depend on the task. In \cite{kazemnejad2024impact}, the authors conducted a comprehensive study of different positional embeddings and provided evidence to show that explicit use of positional encodings is perhaps not essential. In \cite{deletang2022neural}, the authors conducted experiments on tasks divided based on their placement in the Chomsky hierarchy and showed the importance of structured memory (stack, tape) in length generalization.  In a recent work, \cite{zhou2023algorithms} proposed the RASP conjecture, which delineates the tasks where transformers excel or fall short in length generalization, emphasizing the need for simplicity of tasks for the transformer and data diversity.  In this work, we take a step back and develop a first-principles-based approach to study length generalization for a range of architectures.  Our work is inspired by the experimental findings of \cite{zhou2023algorithms}.  Although \cite{zhou2023algorithms} provides empirical evidence for the conjecture, our work formalizes and proves length generalization for a range of architectures. 

On the theoretical side of length generalization, in \cite{abbe2023generalization}, the authors showed an implicit bias of neural network training toward minimum-degree interpolators. This bias was used to explain the failures of length generalization in the parity task from \cite{anil2022exploring}. In \cite{xiao2023conditions}, the authors use directed acyclic graphs (DAGs) to formulate the computation in reasoning tasks and characterize the conditions under which there exist functions that permit length generalization. Our results crucially differ; we show a range of conditions under which models learned via standard expected risk minimization achieve length and compositional generalization.  In  \cite{hou2024universal}, the authors  propose an interesting scratch pad strategy inspired by the operation of Turing machines. They call this strategy Turing programs. The scratch pad emulates the operation of a Turing machine. The authors argue that there exists a short RASP program ($O(n)$ length) that can simulate the operation of a Turing machine, which does not generate repeated $n$ grams, for a sufficiently long number of steps ($O(\exp(n))$).

\paragraph{Compositional generalization}

Compositionality has long been seen as a key piece to the puzzle of  human-level intelligence  \citep{fodor1988connectionism, hinton1990mapping, plate1991holographic, montague1970pragmatics}.  
 Compositionality is a broad umbrella term associated with several aspects \citep{hupkes2020compositionality}. In this work, we focus on systematicity, which evaluates a model's capability to understand known parts and combine them in new contexts.  The breadth of research on compositional generalization, encompassing studies like \cite{lake2018generalization, loula2018rearranging, gordon2019permutation, hupkes2020compositionality,  kim2020cogs, xu2022compositional,  arora2023theory}, is too expansive to address comprehensively here, refer to these surveys \citep{lin2023survey, sinha2024survey} for details.

In recent years, several works have taken first steps towards theoretical foundations of compositionality. 
% We borrow our formal definitions of compositionality from. 
We leverage the mathematical definition of compositionality from \cite{wiedemer2023compositional}, which focuses on generalization to the Cartesian product of the support of individual features. In \cite{dong2022first}, the authors analyze the conditions that provably guarantee generalization to the Cartesian product of the support of individual training features. \cite{dong2022first} studied additive models, i.e., labeling function is additive over individual features.  In \citep{wiedemer2023provable}, the authors focus on a more general model class than \cite{dong2022first}, where the labeling function is of the form $f(x_1,\cdots, x_n) = C(\phi_1(x_1),\cdots, \phi_n(x_n))$.  However, to guarantee compositional generalization, \cite{wiedemer2023compositional} require that the learner know the exact function $C$ that is used to generate the data. In our work, we do not make such an assumption; our data generation is dictated by the architecture in question, e.g. RNN, and we constrain the dimension of its hidden state. \cite{lachapelle2023additive, brady2023provably} extend these precursor results from \cite{dong2022first} from the supervised setting to the unsupervised setting. In particular, \cite{lachapelle2023additive, brady2023provably} are inspired by the success of object-centric models and show that additive decoder-based autoencoders achieve compositional generalization. 
In \cite{schug2023discovering}, the authors study compositionality in the context of meta-learning. Each task parameter in their setting is specified via a linear combination of some basis module parameters. They construct an approach that achieves provable compositional guarantees and outperforms meta-learning approaches such as MAML and ANIL.

\section{Length and Compositional Generalization: Definitions and Framework} 
\label{sec: lgcg_framework}

We are given a data set consisting of a sequence of inputs $\{x_1, \cdots, x_t\}$ and a corresponding sequence of labels $\{y_1, \cdots, y_t\}$, where each $x_i \in \mathbb{R}^{n}$ and $y_i \in \mathbb{R}^{m}$.   Observe that this formulation includes both standard downstream tasks such as arithmetic tasks, e.g., $y_i = \sum_{j=1}^{i} x_j$, $y_i = \Pi_{j=1}^{i}x_j$ etc.,  as well as next-token prediction task, where  $\{y_1, \cdots, y_t\}=\{x_2, \cdots, x_{t+1}\}$.  We denote a sequence $\{s_1, \cdots, s_t\}$ as $s_{\leq t}$, $X_k$ is random variable for token at $k^{th}$ position and its realization is $x_k$.   Consider a sequence $\{x_j\}_{j=1}^{\infty}$, which is sampled from $\mathbb{P}_{X}$, and a subsequence of this sequence $x_{\leq t}=\{x_j\}_{j=1}^{t}$, whose distribution is denoted as $\mathbb{P}_{X_{\leq t}}$. The label $y_t=f(x_{\leq t})$, where $f$ is the labeling function. The tuple of base distribution and the labeling function is denoted as $\mathcal{P} = \Big\{\mathbb{P}_{X}, f\Big\}$ and the tuple of base distribution up to length $t$ is denoted as $\mathcal{P}(t) = \Big\{\mathbb{P}_{X_{\leq t}}, f\Big\}$.  The support of the $k^{th}$ token $X_k$ in the sequence sampled from $\mathbb{P}_X$ is denoted $\mathsf{supp}(X_k)$. Given training sequences of length $T$ from $\mathcal{P}(T)$,  we are tasked to learn a model from the dataset that takes a sequence $x_{\leq t}$ as input and predicts the true label $y_t$ as well as possible. If the model succeeds to predict well on sequences that are longer than maximum training length $T$, then it is said to achieve length generalization (a more formal definition follows later). Further, if the model succeeds to predict well on sequences comprising of combination of tokens that are never seen under training distribution, then it is said to achieve compositional generalization (a more formal definition follows later.).  We next study both of these generalization forms.

\paragraph{Learning via expected risk minimization} Consider a map $h$ that accepts sequences of $n$-dimensional inputs to generate a $m$-dimensional output.  We measure the loss of predictions of $h$, i.e., $h(x_{\leq t})$,  against true labels as $\ell\big(h(x_{\leq t}), y_t\big)$, where $y_t$ is the true label for sequence $x_{\leq t}$. In what follows, we use the $\ell_2$ loss. Given sequences sampled from $\mathcal{P}(T)$, the expected risk in all instances of time up to the maximum length $T$ is defined as $R(h; T) \coloneq \sum_{t=1}^{T}\mathbb{E}\big[\ell(h(x_{\leq t}), y_t)\big]$.
The learner aims to find an $h^{*}$ that solves

\begin{equation}
   \mathcal{H}^{*}  = \argmin_{h \in \mathcal{H}} R(h; T), 
    \label{eqn: risk_min}
\end{equation}

where $\mathcal{H}$ is the hypothesis class of models and $\mathcal{H}^{*}$ is the set of all the solutions to equation \eqref{eqn: risk_min}. We seek to understand the properties of solutions to \eqref{eqn: risk_min} through the lens of the following questions. 
\begin{mdframed}
    When do common sequence-to-sequence models $\mathcal{H}$ succeed at length \& compositional generalization and when do they fail?  
\end{mdframed}

\begin{definition} \textbf{Length Generalization:} 
Consider the setting where a model is trained on sequences $(x_{\leq t}, y_{\leq t})$ of length up to $T$ drawn from $\mathcal{P}(T)$. If the model achieves zero error on sequences $(x_{\leq t}, y_{\leq t})$ of length up to $\tilde{T}$ drawn from $\mathcal{P}(\tilde{T}), \forall\; \tilde{T}\geq 1$,  then it achieves length generalization w.r.t. $\mathcal{P}$.
\end{definition}

 In the above definition of length generalization, we simply ask if the model is generalized to longer sequences.
We drop the phrase w.r.t $\mathcal{P}$ hereafter to avoid repetition. We now define a test distribution that evaluates compositional generalization capabilities. We consider sequences of fixed length $T$. Define a uniform distribution $\mathbb{Q}_{X_{\leq T}}$ such that the support of $\mathbb{Q}_{X_{\leq T}}$ is equal to the Cartesian product of the support of each token $X_k$ of $\mathbb{P}_{X}$, we write this joint support as $\Pi_{j=1}^{T} \mathsf{supp}(X_j)$. In this case as well, the labeling function continues to be $f$. Hence, we obtain the tuple $\mathcal{Q}(T)=\{\mathbb{Q}_{X_{\leq T}}, f\}$.

\begin{figure}
    \centering
    \includegraphics[width=0.8\textwidth, trim=0in 2.5in 0 3in ]{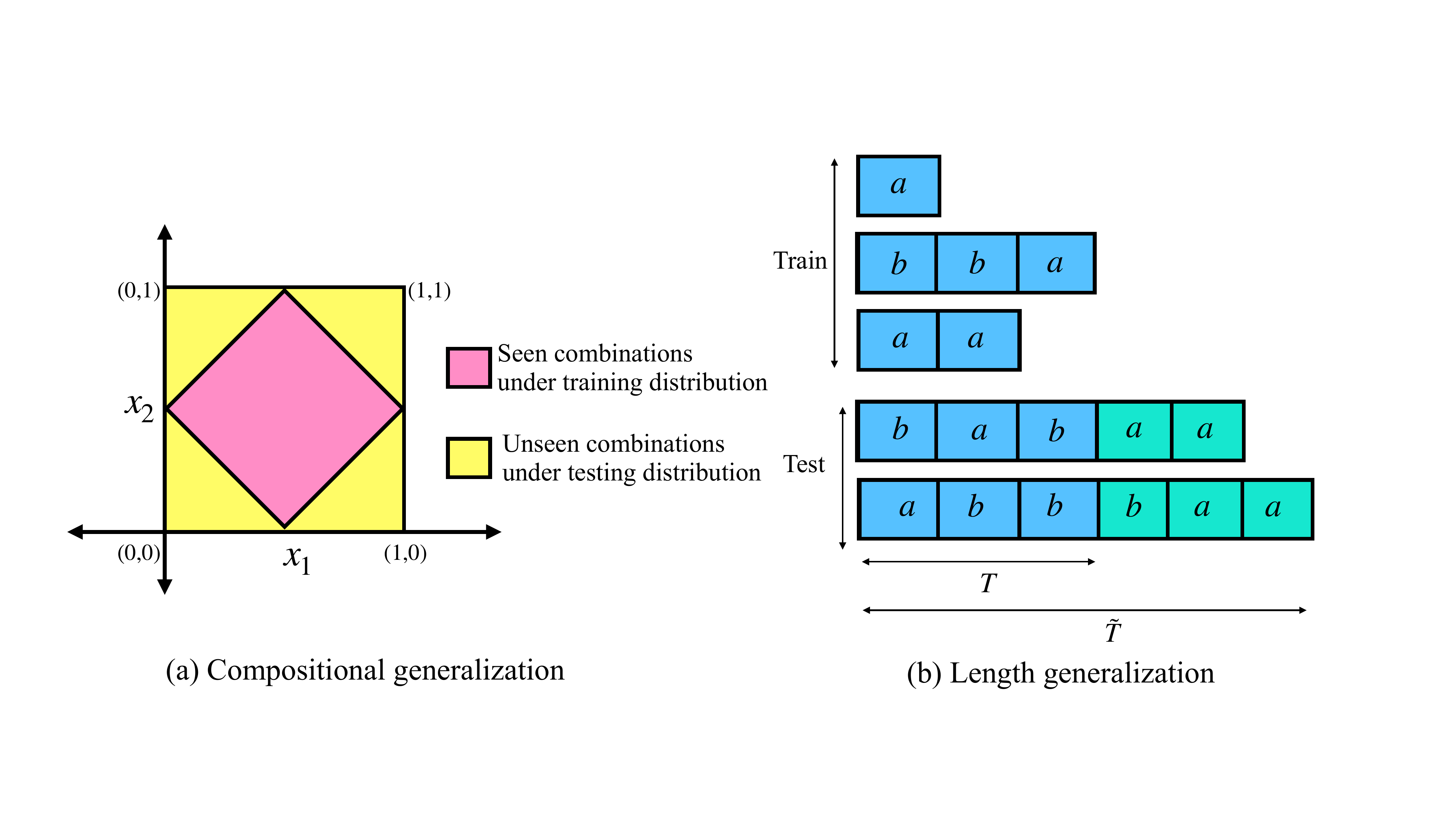}
    \caption{Illustrating support of train vs test distribution for (a) compositional generalization and (b) length generalization.}
    \label{fig1}
\end{figure}

\begin{definition} \textbf{Compositional Generalization:} 
Consider the setting where a model is trained on sequences $(x_{\leq t}, y_{\leq t})$ of length up to $T$ drawn from $\mathcal{P}(T)$. If the model achieves zero error on  sequences $(x_{\leq t}, y_{\leq t})$ of length up to  $T$ drawn from $\mathcal{Q}(T)$,  then it achieves compositional generalization.
\end{definition}

This definition of compositionality above is based on \cite{wiedemer2023compositional, brady2023provably}.  In this definition, we seek generalization to new combinations of seen tokens.

  \textit{Illustrative example} We teach the model multiplication on sequences of length $2$ , where each $x_j$ is a scalar, $y_i=\Pi_{j=1}^{i} x_{j}$. Say the support of the entire sequence drawn from $\mathbb{P}_{X}$ is $\big\{x \; \big| \; \|(x_1,x_2)-\frac{1}{2}\boldsymbol{1}\|_1\leq \frac{1}{2}, \; x_k \in [0,1], \; \forall k \geq 3 \big\}$. The support of the training distribution $\mathbb{P}_{X_{\leq 2}}$ is $\big\{x \; \big| \|(x_1,x_2)-\frac{1}{2}\boldsymbol{1}\|_1\leq \frac{1}{2}\big\}$ shown in the pink region in Figure~\ref{fig1}a. In Figure~\ref{fig1}a,  we illustrate compositional generalization, the model is trained on the pink region and asked to generalize to the yellow region.  Furthermore, if the model continues to correctly multiply on longer sequence lengths in $\mathbb{P}_{X_{\leq \tilde{T}}}$ for $\tilde{T}\geq T$, then it achieves length generalization (Figure~\ref{fig1}b).  

  \paragraph{Contrast with existing works on distribution shifts:}  Both notions of compositional generalization and length generalization introduced above involve testing on distributions whose support is not contained in the training distributions. The long line of work on distribution shifts \citep{sugiyama2007covariate, david2010impossibility, ben2014domain, rojas2018invariant, arjovsky2019invariant, ahuja2021invariance} assumes that the support of test is contained in the support of the train distribution. 
In recent years, there has been the development of a theory for distribution shifts under support mismatch \citep{dong2022first, abbe2023generalization, wiedemer2023compositional,netanyahu2023learning, shen2023engression}. Our work is closer to the latter line of work, but it comes with its own \emph{technical challenges}, which involve building new proofs different from the above line of work, as we  study a new family of models, i.e., sequence-to-sequence models, and a new form of generalization, i.e., length generalization. There is an important remark that is worth mentioning here. In line with most works that study out-of-distribution generalization \cite{dong2022first, wiedemer2023compositional, arjovsky2019invariant}, we also assume access to infinite data from the training distribution. This is a reasonable assumption as it isolates the study of generalization behaviors of ideal estimators beyond the training distribution from the impact of imperfect estimation arising due to access to finite samples from the training distribution.

\paragraph{RASP-L conjecture} \cite{zhou2023algorithms} propose a conjecture backed by empirical evidence, which delineates the conditions that suffice for length generalization for transformers. 
The conjecture places three requirements -- a) realizability: the task of interest is realizable on the transformer, b) simplicity: the task can be expressed as a short program in RASP-L language, c) diversity: the training data is sufficiently diverse such that there is no shorter program that achieves in-distribution generalization but not OOD generalization. We leverage assumptions similar to a) and b). We assume realizability, which means labeling function $f$ is in the hypothesis class $\mathcal{H}$. As to simplicity, we consider hypothesis class $\mathcal{H}$ with limited capacity. In terms of limiting the capacity, we divide our analysis into two parts. First,  we study structured hypothesis classes, where we make assumptions for example, on the number of blocks in the transformer. Next, we study less structured hypothesis classes, where we make assumptions that allow for bounded covering of the class.  We find c) from \citep{zhou2023algorithms}, which states that no shorter program achieves the generalization in-distribution but not the generalization out-of-distribution, to be a strong assumption. We will also make assumptions on data diversity, but our assumptions will directly place constraints on more natural primitives, such as the support of the training distributions and the length of the training sequences.

% In our setting, we do not invoke it and instead, we require that the support of test distribution is not larger than the Cartesian product of the marginal distribution of the tokens. We now move to proving simplified versions of this conjecture for different architectures. 

\section{Structured Hypothesis Classes}
\label{sec: shc}
In this section, we study different sequence-to-sequence hypothesis classes under certain structural assumptions. 

\subsection{Deep sets}
\label{sec: deep sets}
Deep sets are a natural first choice of architecture to study here. These take sets as inputs, and thus handle inputs of arbitrary lengths. These were introduced in \cite{zaheer2017deep}. Informally stated, \cite{zaheer2017deep} show that a large family of permutation-invariant functions can be decomposed as $\rho \big(\sum_{x\in \mathcal{X}} \phi(x)\big)$. Consider examples of the sum operator or the multiplication operator, which take $\{x_1, x_2, \cdots, x_k\}$ as input and return the sum $y = \sum_{j=1}^{k} x_{j}$ or the product $y=\Pi_{j=1}^{k} x_j$. These operations are permutation invariant and can be expressed using the decomposition above.  For the sum operator $\rho$ and $\phi$ are identity and for the multiplication operator $\rho = \exp$ and $\phi=\log$.  Consider another example from language. We construct a bag of words sentiment classifier, where $\{x_1,x_2,\cdots, x_i\}$ is the set of words that appear in the sentence, $\phi(x_j)$ is the feature embedding for word $j$. $\sum_{j\leq i}\phi(x_j)$ is the representation of the entire sentence which is passed to the final layer $\rho$ that generates the sentiment label. In what follows, we aim to understand when such a classifier generalizes to sentences beyond training lengths and to new sentences comprised of unseen word combinations. 

Before we state the first result, we start with a lemma which is used in several of the results that follow. 
In all the results that follow, we work with standard topology in $\mathbb{R}^{nt}$, where $n$ is dimension of each token and $t$ is the sequence length. We remind the reader of the definition of a regular closed set --  if a set is equal to the closure of its interior, then it is said to be a regular closed set.  In all the results that follow, we either work with continuous random variables for which the Radon-Nikodym derivative of $X_{\leq t}$ is absolutely continuous w.r.t Lebesgue measure $\forall t$ or we work with discrete random variables for which the Radon-Nikodym derivative of $X_{\leq t}$ is absolutely continuous w.r.t counting measure $\forall t$. 

\begin{assumption}
Each function in the hypothesis class $\mathcal{H}$ takes a sequence $\{x_1,\cdots, x_i\}$ as input and outputs $h(x_1,\cdots, x_i) =  \omega \Big(\sum_{j\leq i} \psi(x_j) \Big)$, where $\omega$ is a single layer perceptron with a continuously differentiable bijective activation (e.g., sigmoid) and $\psi$ is a map that is differentiable. 
\label{assm: perm_inv}
\end{assumption}

A simple mathematical example of a function from the above family when $\psi(x_j) = [x_j, x_j^2, x_j^3]$ is a polynomial map of degree $3$ and each $x_j$ is a scalar  -- $\sigma\big(a \sum_{j \leq i} x_j + b \sum_{j\leq i} x_j^2 + c\sum_{j \leq i} x_j^3\big)$.  In the assumption that follows, we assume that the support of the sequences is regular closed. 

\begin{assumption}
\label{assm: reg_closed}
    The joint support $\mathsf{supp}(X_{\leq i})$ is non-empty and regular closed for all $i\leq T$. 
\end{assumption}

In most of our results in the main body, we invoke Assumption~\ref{assm: reg_closed}. This assumption is satisfied in many cases if the tokens are continuous random variables, but it is not satisfied for discrete random variables. Several of our key results can be extended to the setting with discrete tokens, we provide the key non-trivial extensions (Theorem~\ref{thm4:disc}) in the Appendix.

\paragraph{Linear identification} Each architecture that we study in this work relies on a hidden representation that is passed on to a non-linearity to generate the label. Under the realizability condition for deep sets, the labeling function takes the form $f(\mathcal{X}) = \rho (\sum_{x\in \mathcal{X}} \phi(x))$, where $\phi(x)$ is the hidden representation. If the learned deep set is denoted by $\omega (\sum_{x\in \mathcal{X}} \psi(x))$, then the learned hidden representation is $\psi(x)$. If $\psi(x) = A \phi(x)$, then the learned representation is said to \emph{linearly identify} the true data generating representation $\phi(x)$.  We borrow this definition from the representation identification literature \citep{khemakhem2020variational, roeder2021linear}.

\begin{lemma}
\label{lemma1}
Let $\mathcal{X}\subseteq \mathbb{R}^{n}$. If $f:\mathcal{X} \rightarrow \mathbb{R}^{m}$ and $g:\mathcal{X}\rightarrow \mathbb{R}^{m}$ are continuously differentiable functions that satisfy $f(x)=g(x)$ almost everywhere in $\mathcal{X}$, where $\mathcal{X}$ is a regular non-empty closed set, then $f(x)=g(x), \forall x \in \mathcal{X}$  and $\nabla f(x) = \nabla g(x), \forall x \in \mathcal{X}$, where $\nabla$ is the Jacobian w.r.t $x$. 
\end{lemma}

The proof is provided in Appendix Section~\ref{sec: proofoflemma1}. 

\begin{theorem}
\label{thm1} If $\mathcal{H}$ follows Assumption~\ref{assm: perm_inv}, the realizability condition holds, i.e., $f\in \mathcal{H}$,  $\mathsf{supp}(X_j)=[0,1]^{n},\; \forall j \geq 1$, and the regular closedness condition in Assumption~\ref{assm: reg_closed} holds, then the model trained to minimize the risk in \eqref{eqn: risk_min} with $\ell_2$ loss generalizes to all sequences in the hypercube $[0,1]^{nt}, \;\forall t \geq 1$ and thus achieves length and compositional generalization. 
\end{theorem}

 % Due to space constraints, we only give proof sketch for some of the results in the main body. 
 
\begin{proof}
  Consider any $h$ that solves \eqref{eqn: risk_min}. Since $\ell$ is $\ell_2$ loss and realizability condition holds, $f$ is a solution to \eqref{eqn: risk_min}. For all $x_{\leq T} \in \mathsf{supp}(X_{\leq T})$ except over a set of measure zero the following condition holds

\begin{equation}
 h(x_{\leq T}) =  f(x_{\leq T}). 
\end{equation}

The above follows from the fact that $h$  solves \eqref{eqn: risk_min}, i.e., $\mathbb{E}[\|h-f\|^2]=0$ and from Theorem 1.6.6. \citep{ash2000probability}.  Since $\mathsf{supp}(X_{\leq T})$ is regular closed, $f,h$ are both continuously differentiable, we can use Lemma \ref{lemma1}, it follows that the above equality holds for all $x_{\leq T} \in \mathsf{supp}(X_{\leq T})$.
From realizability condition it follows that true $f(x_{\leq T}) = \rho\Big(\sum_{j\leq T} \phi(x_j)\Big)$.  We substitute the functional form from Assumption~\ref{assm: perm_inv} to get

\begin{equation}
    \begin{split}
         \omega \Big(\sum_{j\leq T} \psi(x_j) \Big) =  \rho\Big(\sum_{j\leq T}\phi(x_j)\Big) .\\  
    \end{split}
\end{equation}

 $\omega$ and $\rho$ are both single layer perceptron with a bijective activation $\sigma$. We substitute the parametric form of $\omega$ and $\rho$ to obtain

\begin{equation}
    \begin{split}
        & \sigma \Big(A \sum_{j\leq T} \psi(x_j)\Big) = \sigma\Big(B \sum_{j\leq T}\phi(x_j)\Big) \implies   A \sum_{j\leq T} \psi(x_j) = B \sum_{j\leq T}\phi(x_j). \\ 
    \end{split}
\end{equation}

The second equality in the above simplification follows from the fact that the activation $\sigma$ is bijective, the inputs to $\sigma$ are equal.  We take the derivative of the expressions above w.r.t $x_r$ to get the following condition and equate them (follows from Lemma~\ref{lemma1}). For all $x_r \in \mathsf{supp}(X_r)$, i.e., $x_r \in [0,1]^{n},$  

\begin{equation}
    \begin{split}
         \nabla_{x_{r}} \Big(A \sum_{j\leq T} \psi(x_j)\Big) =  \nabla_{x_{r}} \Big(B \sum_{j\leq T}\phi(x_j)\Big).\\ 
    \end{split}
\end{equation}

% We can take derivative w.r.t all $j\leq T$ to obtain the following. 
We drop the subscript $r$ to simplify the notation. Therefore, for all $x \in [0,1]^{n}$,  $A \nabla_{x} \psi(x)  =  B \nabla_{x} \phi(x)$, where $\nabla_{x} \psi(x)$ is the Jacobian of $\psi(x)$ w.r.t $x$ and $\nabla_{x} \phi(x)$ is the Jacobian of $\phi(x)$ w.r.t $x$.
We now take the derivative w.r.t some component $x^{k}$ of vector $x=[x^{1}, \cdots, x^{n}]$.  Denote the components other than $k$ as $x^{-k} = x\setminus x^{k}$.  From the above condition, it follows that for all $x \in [0,1]^{n}$, $A \frac{\partial \psi(x)}{\partial x^{k}}  =  B \frac{\partial\phi(x)}{\partial x^{k}}$.

Using fundamental theorem of calculus, we can integrate both sides for fixed $x^{-k}$ and obtain the following for all $x^{k} \in [0,1],$

\begin{equation}
    \begin{split}
        & A \psi(x^{k}, x^{-k}) =  B \phi(x^{k}, x^{-k}) + C_k(x^{-k}) \implies  A \psi(x)   - B \phi(x) =  C_k(x^{-k}).
    \end{split}
\end{equation}

The above condition is true of all $k \in \{1, \cdots, n\}$. Hence, we can deduce that for all $x \in [0,1]^n$ and  for $k\not=j$, where $j,k \in \{1, \cdots, d\}$,

\begin{equation}
    A \psi(x)  - B \phi(x) = C_k(x^{-k}) = C_{j}(x^{-j}).
\end{equation}

  Take the partial derivative of $C_k(x^{-k})$ and $C_{j}(x^{-j})$  w.r.t $x^{j}$ to obtain, for all $ x^{j} \in [0,1] $, 

\begin{equation}
\begin{split}
    \frac{\partial C_k(x^{-k})}{\partial x^j} = \frac{ \partial C_{j}(x^{-j})}{\partial x^j} = 0. 
\end{split}
\end{equation}

In the above simplification, we use the fact that $\forall x^{j} \in [0,1], \frac{ \partial C_{j}(x^{-j})}{\partial x^j} = 0$. Therefore, $C_k(x^{-k})$ cannot depend on $x^j$. We can apply the same condition on all $j \not=k$. As a result, $C_{k}(x^{-k})$ is a fixed constant vector denoted as $C$. We write this as 

\begin{equation}
     A \psi(x)  = B \phi(x) + C. 
\end{equation}

Substitute the above into $ A \sum_{j\leq T} \psi(x_j) = B \sum_{j\leq T}\phi(x_j)$ to obtain 

\begin{equation}
     B \sum_{j\leq T}\phi(x_j)+ CT = B \sum_{j\leq T}\phi(x_j) \implies C=0.
\end{equation}

Therefore, we get  
\begin{equation}
\forall x \in [0,1]^{n},  A \psi(x) =  B \phi(x).
\label{eqn: psi_phi_eq_deepset}
\end{equation}

We now consider any sequence $x_{\leq \tilde{T}}$ from $[0,1]^{n\tilde{T}}$. The prediction made by $h$ is

\begin{equation}
    \begin{split}
      h(x_{\leq \tilde{T}}) =   \sigma \Big(A \sum_{j\leq \tilde{T}} \psi(x_j)\Big) =  \sigma \Big(B \sum_{j\leq \tilde{T}} \phi(x_j)\Big) = f(x_{\leq \tilde{T}}).
    \end{split}
\end{equation}

We use \eqref{eqn: psi_phi_eq_deepset} in the simplification above. From the above, we can conclude that $h$ continues to be optimal for distribution $\mathbb{P}_{X_{\leq \tilde{T}}}$.

\end{proof}

 In the above result,  we require the support of the marginal distribution of each token to be $[0,1]^{n}$. The support of $T$ token length sequence under the joint training distribution can still be a much smaller subset of $[0,1]^{nT}$, as illustrated in Figure~\ref{fig1}a. Despite this the model generalizes to all sequences in $[0,1]^{nt}$ for all $t$. An important insight from equation \eqref{eqn: psi_phi_eq_deepset}  is that the hidden representation learned by the model is a linear transform of the true hidden representation, i.e., it achieves linear identification $\psi = A^{-1}B\phi$, if the matrix in the output layer $A$ is left-invertible. In Theorem \ref{thm1}, we observe all the labels from $t=1$ to $T$, i.e., $y_1$ to $y_T$. The result continues to hold if we observe only the label at length $T$, i.e., $y_T$, for any $T\geq 1$.  We believe Theorem~\ref{thm1} is simple and helps build intuition about the problem. We now extend Theorem~\ref{thm1} to more general settings.

\begin{assumption}
\label{assm: perm_inv_gn_ds}
 Each function in $\mathcal{H}$ is expressed as $h(x_1,\cdots, x_i)=\omega ( \sum_{j=1}^{i} \psi(x_j) )$, where $\omega$ is a $C^{1}$-diffeomorphism.
\end{assumption}

In contrast to Assumption~\ref{assm: perm_inv}, where $\omega$ is a single layer perceptron, the above assumption considers $\omega$ that are $C^1$-diffeomorphisms. Recall that a $C^1$-diffeomorphism is a continuously differentiable map that has a continuously differentiable inverse. 
 
\begin{assumption}
\label{assm: supp2_ds}

 The support of all tokens is equal, that is, $\mathsf{supp}(X_{j}) = [0,1]^{n}$, where $j\geq 1$.  The support of  $[\phi(X_1), \phi(X_2)]$ is  $\mathbb{R}^{2m}$, where $\phi$ is the embedding function for the labeling function $f(x_1,\cdots, x_i) = \rho (\sum_{j=1}^{i} \phi(x_j))$.
\end{assumption}

% We provide a remark on the assumption and where it is used following the proof of the next theorem. 

\begin{theorem}
\label{thm_ds_gn}
 If $\mathcal{H}$ follows Assumption~\ref{assm: perm_inv_gn_ds}, the realizability condition holds, i.e., $f\in \mathcal{H}$, and a further assumptions on the support (Assumption~\ref{assm: reg_closed}, \ref{assm: supp2_ds}) holds, then the model trained to minimize the risk in \eqref{eqn: risk_min} (with $T\geq 2$) with $\ell_2$ loss generalizes to all sequences in $[0,1]^{nt}, \forall t\geq 1$ and thus achieves length and compositional generalization.
\end{theorem}

\begin{proof}
We start with the same steps as earlier proofs and equate the prediction of $h$ and $f$. We first use the fact $h(x_{\leq i}) = f(x_{\leq i})$ almost everywhere in the support. We can use the continuity of $h,f$ and regular closedness of the support to extend the equality to all points in the support (follows from the first part of Lemma~\ref{lemma1}) to obtain the following. For all $x_{\leq i} \in \mathsf{supp}(X_{\leq i})$ , $\omega \Big(\sum_{j \leq i}  \psi(x_j) \Big) =  \rho\Big(\sum_{j \leq i}  \phi(x_j) \Big)$, which implies

\begin{equation}
    \begin{split}
       &   \sum_{j \leq i}  \psi(x_j) = \omega^{-1} \circ \rho  \Big(\sum_{j \leq i}   \phi(x_j) \Big) \implies   \sum_{j \leq i}  \psi(x_j) = a  \Big(\sum_{j \leq  i} \phi(x_j) \Big), 
    \end{split}
    \label{eqn1_ds_nl}
\end{equation}

where $a= \omega^{-1}\circ \rho$.  In the above simplification, we used the parametric form for the true labeling function and the learned labeling function and use the invertibility of $\omega$. Let us consider the setting when $i=1$. In that case summation involves only one term. Substitute $x_1=x$. We obtain $\forall x\in [0,1]^n,$

\begin{equation}
    \psi(x) = a( \phi(x)). 
     \label{eqn2_ds_nl}
\end{equation}

The above expression implies that $\psi$ bijectively identifies $\phi$. Let us consider the setting when $i=2$.  Substitute $x_1=x$ and $x_2=y$. We obtain

\begin{equation}
   a (\phi(x)) +a (\phi(y))  = a   \big(\phi(x) + \phi(y)\big) .   
      \label{eqn3_ds_nl}
\end{equation}

We now use the assumption that $[\phi(x), \phi(y)]$ spans $\mathbb{R}^{2m}$, where $\phi(x)$ and $\phi(y)$ individually span $\mathbb{R}^{m}$. Substitute $\phi(x) = \alpha$ and $\phi(y) = \beta$. We obtain $\forall \alpha \in \mathbb{R}^{m}$, $\forall \beta \in \mathbb{R}^{m}$

\begin{equation}
   a (\alpha) +a (\beta)  = a  \big( \alpha  + \beta\big) . 
   \label{eqn: convex_a_comb_ds}
\end{equation}

Observe that $a(0)=0$ (substitute $\alpha=\beta=0$ in the above).

We use \eqref{eqn: convex_a_comb_ds} to show that $a$ is linear. To show that, we need to argue that $a(c \alpha) = c a(\alpha)$ as we already know $a$ satisfies additivity condition.

From the identity above, we want to show that $a(p \alpha) = p a(\alpha)$, where $p$ is some integer. 

Substitute $\beta=-\alpha$ in $a(\alpha + \beta) = a(\alpha) + a(\beta)$. We obtain $a(0) = a(\alpha) + a(-\alpha) \implies a(-\alpha) = -a(\alpha)$. Suppose $p$ is a positive integer. We simplify $a(p \alpha)$ as follows $a(\alpha + (p-1)\alpha) = a(\alpha) + a((p-1)\alpha)$. Repeating this simplification, we get $a(p \alpha ) = p a(\alpha)$. Suppose $p$ is a negative integer. We can write $a( p \alpha) = a(-p \times -\alpha) = -p a(-\alpha)$. Since $a(-\alpha) = -a(\alpha)$, we get $a(p\alpha) = p a(\alpha)$. 

% This follows directly. 

Suppose $c$ is some rational number, i.e., $c = p/q$, where $p$ and $q$ are non-zero integers. We already know $a(p \alpha) = p a(\alpha)$. Further, we obtain

$a(q \frac{1}{q} \alpha) = q a(\frac{1}{q}\alpha) \implies a(\frac{1}{q}\alpha) = \frac{1}{q}a(\alpha)$, where $q$ is some integer.

Now combine these $a(p/q \alpha) = p a(1/q\alpha ) = \frac{p}{q}a(\alpha)$. We have established the homogeneity condition for rationals. 

We will now use the continuity of the function $a$ and density of rationals to extend the claim for irrationals. Suppose $c$ is some irrational. Define a sequence of rationals that approach $c$ (this follows from the fact that rationals are dense in $\mathbb{R}$). 

$a(c \alpha) = a(\lim_{n\rightarrow \infty} q_n \alpha) = \lim_{n \rightarrow \infty}  a(q_n \alpha). $

In the second equality above, we use the definition of continuity ($a$ is continuous since composition of continuous functions is continuous). We can also use the property that we already showed for rationals to further simplify

$\lim_{n \rightarrow \infty}  a(q_n \alpha)  = a(\alpha) \lim_{n \rightarrow \infty}  q_n  = ca(\alpha). $

Observe that $a:\mathbb{R}^{m}\rightarrow \mathbb{R}^{m}$ and for any $\alpha, \beta \in \mathbb{R}^{m}$ $a(\alpha+\beta) = a(\alpha) + a(\beta)$ and $a(c\alpha) = c a(\alpha)$.  From the definition of a linear map it follows that $a$ is linear. As a result, we can write $\forall x \in [0,1]^{n}$

\begin{equation}
     \psi(x) = a(\phi(x)) = A( \phi(x))
     \label{eqn: ds_nl_lin_id}
\end{equation}

Observe that $a$ is invertible because both $\rho$ and $\omega$ are invertible. Since $a$ is both invertible and linear, we know that $A$ is an invertible matrix. From this we get 

\begin{equation}
     \phi(x) = a^{-1}(\psi(x)) =  A^{-1}\psi(x)= C\psi(x), 
\end{equation}
where $C=A^{-1}$.  For all $z\in \mathbb{R}^{m}$, we obtain

$$\rho^{-1}\circ \omega (z) = Cz  \implies \omega(z) = \rho(Cz)$$

Let us consider any sequence $x_{\leq \tilde{T}} \in [0,1]^{n\tilde{T}}$. 
We use the above conditions  
$$\omega\big( \sum_{j \leq \tilde{T}} \psi( x_j) \big) = \rho ( C\sum_{j \leq \tilde{T}} \psi( x_j) ) = \rho \big( \sum_{j \leq \tilde{T}} \phi( x_j) \big) $$ 

 Thus we obtain length and compositional generalization. 

\end{proof}

At this point, some remarks are in order. Firstly,  from equation \eqref{eqn: ds_nl_lin_id}, it follows that the learned representation linearly identifies the representation from the labeling function. The proofs in this section have the flavor of the proofs in the literature on provable identifiability and out-of-distribution generalization \citep{khemakhem2020variational, wiedemer2023compositional}. There are some crucial aspects in which the above proofs differ from \cite{wiedemer2023compositional}. The authors consider labeling functions of the form $C(\phi_1(x_1),\cdots, \phi_n(x_n))$.  However, to guarantee compositional generalization, \cite{wiedemer2023compositional} require that the learner know the exact function $C$ that is used to generate the data and there exists some input value $x_i$ for which $\phi_i(x_i)$ is known to the learner. We do not rely on these assumptions and use different proof strategies to arrive at our results. 

\paragraph{High capacity deep sets} In the above results,  we operated with some constraints on the deep sets. In Theorem~\ref{thm1}, we use a limited capacity $\omega$ which is represented via a single-layer perceptron. In Theorem~\ref{thm_ds_gn}, we use $\omega$ that are represented by $C^1$ diffeomorphisms, which implies that the output dimension of $\psi$ equals the label dimension $m$ and cannot be larger.  What happens when we work with deep sets with arbitrary capacity, i.e., no constraints on $\omega$ and $\psi$? These models then express a large family of permutation-invariant maps \citep{zaheer2017deep}.  Suppose $\mathcal{H}$ is the class of all permutation invariant maps and the labeling function $f\in \mathcal{H}$. Consider a map $h$  such that $h=f$ for all sequences of length up to $T$, and $h=f+c$ otherwise. Observe that $h$ is permutation invariant and also belongs to $\mathcal{H}$. $h$ achieves zero generalization error on training sequences of length $T$ but a nonzero error on longer sequences. Thus, in the context of high capacity deep sets, there exist solutions to \eqref{eqn: risk_min},  which do not achieve length generalization. We can construct the same argument for compositional generalization as well and say $h=f$ on the training distribution  (pink region) in Figure~\ref{fig1}a and $h=f+c$ on the testing distribution (yellow region) in Figure~\ref{fig1}a.  In order to show successful generalization (length or compositional) in Theorem~\ref{thm1}, we require all solutions to risk minimization in \eqref{eqn: risk_min} to match the predictions of the true labeling function in the data beyond the support of the training distribution.  In order to show that high-capacity models are not guaranteed to succeed, we focused on showing that there exists a solution to \eqref{eqn: risk_min} that does not generalize beyond the support of training distribution. A more nuanced argument for failure should show that there exist solutions reachable via gradient descent that do not generalize.  We leave a rigorous theoretical exploration of this to future work. However, we conduct experiments with high capacity models in the Appendix (Section~\ref{sec:failure_cases}) to illustrate failures in the high capacity regime. 

%In this section, we discussed length generalization guarantees in limited capacity settings in Theorem~\ref{thm1} (Theorem~\ref{thm_ds_gn}), i.e., a shallow $\rho$ (or diffeomorphism $\rho$) and a deep $\phi$. 

We now describe how CoT data can assist in achieving length generalization in high capacity deep sets maps, i.e., with less restrictions than ones imposed in Assumption~\ref{assm: perm_inv}, \ref{assm: perm_inv_gn_ds}.

\paragraph{Role of CoT data}  It can be useful to interpret CoT from the lens of learning using privileged information (LUPI) paradigm from \cite{vapnik2009new}. In LUPI, the teacher supplies the student with additional information that is only available at the training time.  In the previous paragraph, we saw that length generalization is not achievable for high-capacity deep sets. We now show that CoT supervision provides intermediate computations to the learner as privileged information, which helps it achieve length generalization.

Given a labeling function $f$, in typical supervised learning we are given input $x$ and label $y=f(x)$ pairs. Supervision from chain-of-thought data is meant to provide data on the intermediate thoughts/computations that lead to the output $f(x)$. If $f$ is an MLP with ReLU activationss, we can express it as $f=h_L \circ h_{L-1} \cdots h_{1}(x)$ and $h_i$ is an intermediate layer given as $h_i = \mathsf{ReLU}(W_ix + b_i)$.  It is fairly simple to show  that a learner with access to outputs of intermediate layers in addition to input, output information $(x,y)$, ends up learning the true parameters $W_i$ and $b_i$. As a result, the learner ends up matching the labeling function everywhere thus generalizing out-of-distribution. Hence, CoT provides privileged information that reduces underspecification and helps achieve generalization beyond the training distribution. In contrast, standard supervision with $(x,y)$ pairs would only learn to match $f$ on the support of the training distribution.

In the context of deep sets parametrized as $\rho(\sum_{j\leq i} \phi(x_j))$, we can assume that the learner sees the intermediate embeddings $\sum_{j\leq i} \phi(x_j)$ as CoT in addition to input, output pair $(x_{\leq i}, y_i)$.  Supervision from CoT implies $\sum_{j\leq i} \psi(x_j) = \sum_{j\leq i} \phi(x_j)$. Following similar arguments to proof of Theorem~\ref{thm1}, we get $\psi(x) = \phi(x),\; \forall x$ in the training support of each individal $x$ (See the Appendix Section~\ref{sec:cot_appendix} for a detailed argument).  Further, from equating the labels we get that $\omega (\sum_{j\leq i} \psi(x_j)) = \rho (\sum_{j\leq i} \phi(x_j))$. If the support of $\sum_{j\leq i} \phi(x_j)$ is $\mathbb{R}^{m}$, then we can argue that $\omega = \rho$. This combined with $\psi = \phi$ yields length generalization, i.e., $\omega (\sum_{j\leq i} \psi(x_j)) = \rho (\sum_{j\leq i} \phi(x_j)), \forall i \geq 1$.

% , where $\phi$ is a hidden layer and $w$ is the output layer, then under CoT supervision, the outputs of $\phi(x)$ are also available to the learner. 

% In Section~\ref{sec: ghc}, we present results that capture models between the limited-capacity models of Theorem~\ref{thm1}, Theorem~\ref{thm_ds_gn} and the high-capacity models in the discussion above. Having discussed a 

\subsection{Transformers}
\label{sec: transformers}
Ever since their introduction in \cite{vaswani2017attention}, transformers have revolutionized all domains of AI. In this section, we seek to understand the length generalization for these models. Transformer architectures are represented as alternating layers of attention and position-wise nonlinearity. We drop layer norms for tractability. Following similar notation as in the previous section, we denote positionwise nonlinearity as $\rho$ and attention layer as $\phi$. We obtain the simplest form of the causal transformer model as $\rho \Big(\sum_{j=1}^{i} \frac{1}{i} \cdot \phi(x_i, x_j) \Big)$.  This decomposition captures linear attention, ReLU attention, sigmoid attention, ReLU squared attention, which were previously studied in \cite{wortsman2023replacing, hua2022transformer, shen2023study} and were found to be quite effective in several settings. This decomposition does not capture softmax-based attention. Other works \citep{bai2023transformers} also replaced softmax with other nonlinear attention for a more tractable analysis.   We illustrate the sigmoid-based transformer from \cite{wortsman2023replacing} below. Let $W_q \in \mathbb{R}^{k \times n }$, $W_k \in \mathbb{R}^{k \times n}$, and $W_v \in \mathbb{R}^{k \times n}$ be the query, key, and value matrices.  $\rho$ is parametrized via a multilayer perceptron denoted as $\mathsf{MLP}$. 

\begin{equation}
    \begin{split}
      &  q_i = W_qx_i, \; k_j = W_k x_j, v_j = W_v x_j, \phi(x_i, x_j) = \sigma\bigg(\frac{q_i^{\top}k_j}{\sqrt{d}}\bigg)v_j,\; \mathsf{MLP} \Big(\sum_{j=1}^{i} \frac{1}{i} \cdot \phi(x_i, x_j) \Big).
    \end{split}
\end{equation}

In the above feedforward computation, the output of attention for the current query is computed and sent to the MLP to generate the label.

\begin{assumption}
\label{assm: perm_inv1}
Each function in the hypothesis class $\mathcal{H}$ takes a sequence $\{x_1,\cdots, x_i\}$ as input and outputs $h(x_1,\cdots, x_i) = \omega \Big(\sum_{j\leq i}  \frac{1}{i} \cdot \psi(x_i, x_j)  \Big)$, where $\omega$ is a single layer perceptron with continuously differentiable bijective activation (e.g., sigmoid) and $\psi$ is a map that is differentiable. 
\end{assumption}

We denote the joint support of two tokens $X_i, X_j$ as $\mathsf{supp}(X_i, X_j)$.

\begin{theorem}
    \label{thm2n}
If $\mathcal{H}$ follows  Assumption~\ref{assm: perm_inv1}, the realizability condition holds, i.e., $f\in \mathcal{H}$, $\mathsf{supp}(X_i, X_j)=[0,1]^{2n},\; \forall i \not= j$ and the regular closedness condition in Assumption~\ref{assm: reg_closed} holds, then the model trained to minimize the risk in \eqref{eqn: risk_min} (with $T\geq 2$) with $\ell_2$ loss generalizes to all sequences in the hypercube $[0,1]^{nt}, \; \forall t\geq 1$ and thus achieves length and compositional generalization. 
\end{theorem}

The proof of Theorem~\ref{thm2n} follows the same strategy as proof of Theorem~\ref{thm1}. Similar to Theorem~\ref{thm1}, we observe linear identification here too, i.e., learned attention representation denoted $\psi$ is a linear transform of the true attention representation denoted $\phi$,  i.e., $\psi(x_i,x_j) = C \phi(x_i, x_j) $, (details in Section~\ref{sec: transformers_proofs}). We now extend Theorem~\ref{thm2n} from a single layer perceptron $\omega$ to $C^{1}$-diffeomorphism $\omega$.

\begin{assumption}
\label{assm: perm_inv_gn}
 Each function in $\mathcal{H}$ takes  $\{x_1,\cdots, x_i\}$ as input and outputs  $h(x_1,\cdots, x_i)=\omega ( \sum_{j=1}^{i-1}  \frac{1}{i-1} \cdot \psi(x_i,x_j) )$, where $\omega$ is a $C^{1}$-diffeomorphism, $\omega(0)=0$.
\end{assumption}

The reader would notice that the summation is up to $i-1$ and hence it computes attention scores w.r.t all other terms in the context except $x_i$. We conjecture that the theorem that we present next extends to the more general case where summation includes the $i^{th}$ term.

\begin{assumption}
\label{assm: supp2_main}
 The support of all pairs of tokens is equal, i.e., $\mathsf{supp}(X_{i}, X_{j}) = [0,1]^{2n}$, where $i\not=j$, $i\geq 1, j\geq 1$.  The support of $[\phi(X_1,X_2), \phi(X_1,X_3)]$ is $\mathbb{R}^{2m}$, where $\phi$ is the nonlinearity used in the labeling function $\rho (\sum_{j\leq i} \phi(x_i,x_j))$.
\end{assumption}

\begin{theorem}
\label{thm2_transformers}
     If $\mathcal{H}$ follows  Assumption~\ref{assm: perm_inv_gn}, the realizability condition holds, i.e., $f\in \mathcal{H}$, and further assumptions on the support (Assumption~\ref{assm: reg_closed}, Assumption~\ref{assm: supp2_main}) hold, then the model trained to minimize the risk in \eqref{eqn: risk_min} (with $T\geq 3$) with $\ell_2$ loss generalizes to all sequences in $[0,1]^{nt}, \forall t\geq 1$ and thus achieves length and compositional generalization. 
\end{theorem}

The proof of Theorem~\ref{thm2_transformers} is in Appendix Section~\ref{sec: transformers_proofs} and follows the same strategy as the proof of Theorem~\ref{thm_ds_gn}. Observe that similar to Theorem~\ref{thm2n}, we find that the learned attention representation $\psi$ is a linear transform of the true attention representation $\phi$. 

\paragraph{Multiple attention heads and positional encoding} While the discussion in this section used a single attention head $\phi$, the results extend to multiple attention heads, as shown in Section~\ref{sec: transformers_proofs}.  The transformer model discussed so far uses the current query and compares it to keys from the past; it does not distinguish the keys based on their positions. For many arithmetic tasks such as computing the  median, maximum etc., the positions of keys do not matter but for other downstream tasks such as sentiment classification, the position of the words can be important. In Section~\ref{sec: transformers_proofs}, we adapt the architecture to incorporate relative positional encodings and show how some of the results extend.  We modify the model as $\rho(\sum_{j=1}^{i} \frac{1}{i} \cdot \phi_{i-j}(x_i,x_j))$, where $ \phi_{i-j}(x_i,x_j)$ computes the inner product of the query and key while taking into account the relative position $i-j$. We show that if $\phi_{i-j} =0$ for $i-j>T_{\max}$, i.e., two tokens sufficiently far apart do not  impact the data generation, then length generalization and compositional generalization are achieved.

\paragraph{High capacity transformers} In the above results, we operated with constraints on transformers, which limit their capacity. Similarly to the setting of deep sets, observe that Assumption~\ref{assm: perm_inv1} constrains $\omega$ to a single layer perceptron, Assumption~\ref{assm: perm_inv_gn} constrains $\omega$ to $C^1$ -diffeomorphisms.  What happens if we work with transformers without constraint on $\omega$ and $\psi$? If $\psi(x,y)=\psi(\tilde{x}, y), \forall x\not=\tilde{x}$, then the decomposition for the causal transformer $\omega \Big(\sum_{j=1}^{i} \frac{1}{i} \cdot  \psi(x_i, x_j) \Big)$ becomes $\omega \Big(\sum_{j=1}^{i} \frac{1}{i} \cdot \psi(x_j) \Big)$, which is very similar to deep sets. In such a case, we can adapt arguments similar to those of arbitrary capacity deep sets and argue that there exist solutions to \eqref{eqn: risk_min} that do not achieve length and compositional generalization.

\paragraph{Role of CoT data} Similar to the previous section, we can show that CoT supervision helps achieve length generalization even for higher capacity transformers than ones considered in Theorem~\ref{thm2n}, \ref{thm2_transformers}. Consider the transformer family parametrized as $\rho(\frac{1}{i}\sum_{j\leq i} \phi(x_i, x_j))$, we can assume that the learner sees the output of attention $\sum_{j\leq i} \phi(x_i, x_j)$ as CoT in addition to input, output pair $(x_{\leq i}, y_i)$.  Supervision from CoT implies $\sum_{j\leq i} \psi(x_i,x_j) = \sum_{j\leq i} \phi(x_i,x_j)$. Following similar arguments of proof of Theorem~\ref{thm2n}, we get $\psi = \phi,\; \forall x$ in the training support of each individal $x$. Further, using the full support conditions similar to the role of CoT in deep sets described earlier, we can arrive at length generalization.

% In this section, we discussed models with limited capacity (Theorem~\ref{thm2n}, Theorem~\ref{thm2_transformers}) and models with high capacity, as discussed above. In Section~\ref{sec: ghc}, we present results that capture models between limited-capacity models of Theorem~\ref{thm2n}, Theorem~\ref{thm2_transformers} and high-capacity models in the discussion above. We now move to state-space models and RNNs. 

\subsection{State space models}
\label{sec: ssm}
In recent years, state space models \cite{gu2021efficiently, orvieto2023resurrecting} have emerged as a promising competitor to transformers. In \citep{orvieto2023universality, orvieto2023resurrecting}, the authors used the lens of linear recurrent layer followed by position-wise nonlinearities as the main building block to understand these models. We illustrate the dynamics of these models to show the generation of $x_{\leq t}$ and $y_{\leq t}$ next.  Given the current input $x_t$, we combine it linearly with the hidden state of the past to obtain the current hidden state. The hidden state is inputted to $\rho$, which generates the label as follows.

\begin{equation}
	\begin{split}
		& h_1 = B x_1; \;\; \;h_2 = \Lambda h_1 + Bx_2; \; \cdots,  h_{t} = \Lambda h_{t-1} + B x_{t}, \; \\ 
		& y_1 = \rho(h_1); \; y_2 = \rho(h_2); \; \cdots \cdots , \;\;\;\;\; y_{t} = \rho(h_t),
	\end{split}
	\label{eqn: ssm}
\end{equation}
where $h_t \in \mathbb{R}^{k}$ is hidden state at point $t$, $\Lambda \in \mathbb{R}^{k \times k}, B \in \mathbb{R}^{ k \times n }$ and $\rho:\mathbb{R}^{k}\rightarrow \mathbb{R}^{m} $. We can succinctly write $h_{t} = \sum_{j=0}^{t-1}\Lambda^{j} B x_{t-j}$.

\begin{assumption}
	Each function in the hypothesis class $\mathcal{H}$ takes a sequence $\{x_1,\cdots, x_i\}$ as input and outputs $h(x_1, \cdots, x_{i}) = \omega \Big(\sum_{j=0}^{i-1}\Lambda^{j} B x_{i-j} \Big)$, where $\omega:\mathbb{R}^{k}\rightarrow \mathbb{R}^{m}$ is a  $C^{1}$-diffeomorphism,  $B$ and $\Lambda$ are square invertible. As a result, $m=k=n$. 
	\label{assm: ssm}
\end{assumption}

    \begin{assumption}
\label{assm: joint_int_main}

  The support of $X_1$ is $\mathbb{R}^{n}$. For some length $2 \leq i\leq T$ there exists $in$ sequences $x_{\leq i}$ such that their concatenation forms a $ni \times ni $ matrix of rank $ ni$.
    \label{assm: rank}
\end{assumption}

\begin{restatable}{theorem}{ssm}
	\label{thm: ssm} If  $\mathcal{H}$ follows  Assumption~\ref{assm: ssm}, and the realizability condition holds, i.e., $f\in \mathcal{H}$, and further conditions on the support, i.e., Assumption~\ref{assm: reg_closed}, Assumption~\ref{assm: joint_int_main} holds, then the model trained to minimize the risk in \eqref{eqn: risk_min} with $\ell_2$ loss $(T\geq 2)$ achieves length and compositional generalization. 
\end{restatable}

\begin{proof}
We start with the same steps as earlier proofs and equate the prediction of $h$ and $f$. We first use the fact $h(x_{\leq i}) = f(x_{\leq i}), \forall i \leq T$ almost everywhere in the support. We can use the continuity of $h,f$ and regular closedness of the support to extend the equality to all points in the support (from first part of Lemma~\ref{lemma1}) to obtain the following. For all $x_{\leq i} \in \mathsf{supp}(X_{\leq i})$. 
% We equate the predictions for $y_k$ and true $f(x_{\leq k})$
    \begin{equation}
        \begin{split}
           &f(x_{\leq i}) =  h(x_{\leq i})  \implies \rho( \sum_{j=0}^{i-1}\Lambda^{j} B x_{i-j} ) = \omega(\sum_{j=0}^{i-1}\tilde{\Lambda}^{j} \tilde{B} x_{i-j} ) \implies  \omega^{-1} \circ \rho( \sum_{j=0}^{i-1}\Lambda^{j} B x_{i-j}) =  \sum_{j=0}^{i-1}\tilde{\Lambda}^{j} \tilde{B} x_{i-j}   \\
           & \implies c(\sum_{j=0}^{i-1}\Lambda^{j} B x_{i-j}) =  \sum_{j=0}^{i-1}\tilde{\Lambda}^{j} \tilde{B} x_{i-j}
    \end{split}
    \label{eqn1: ssm_proof}
    \end{equation}

For $i=1$, $\forall x_1\in \mathbb{R}^{n}, c(Bx_1) = \tilde{B}x_{1}$. Substitute $Bx_1=x$, we obtain $\forall x\in \mathbb{R}^{n}, c(x) = \tilde{B}B^{-1}x = Cx$, where we use the fact that $Bx_1$ spans $\mathbb{R}^n$ as $B$ is invertible. 

From linearity of $c$, we obtain 

\begin{equation}
    \omega^{-1}\circ \rho (z) = Cz \implies \rho(z) = \omega(Cz), \forall z \in \mathbb{R}^{n}
    \label{eqn: ssm1}
\end{equation}

We use this linearity of $c$ to simplify 
        \begin{equation}
        \begin{split}
           &c(\sum_{j=0}^{i-1}\Lambda^{j} B x_{i-j}) =  \sum_{j=0}^{i-1}\tilde{\Lambda}^{j} \tilde{B} x_{i-j} \implies  C(\sum_{j=0}^{i-1}\Lambda^{j} B x_{i-j}) =  \sum_{j=0}^{i-1}\tilde{\Lambda}^{j} \tilde{B} x_{i-j} \implies \\ 
            & [CB, C\Lambda B, C \Lambda^2 B, \cdots, C \Lambda^{i-1} B] \begin{bmatrix} x_{i} \\ x_{i-2} \\ \vdots \\ x_{1} \end{bmatrix} - [\tilde{B}, \tilde{\Lambda}\tilde{B}, \tilde{\Lambda}^2 \tilde{B}, \cdots,  \tilde{\Lambda}^{i-1} \tilde{B}] \begin{bmatrix} x_{i} \\ x_{i-2} \\ \vdots \\ x_{1} \end{bmatrix} = 0 \implies \\ 
           & \Big[[CB, C\Lambda B, C \Lambda^2 B, \cdots, C \Lambda^{i-1} B]-[\tilde{B}, \tilde{\Lambda}\tilde{B},  \tilde{\Lambda}^2 \tilde{B}, \cdots, \tilde{\Lambda}^{i-1} \tilde{B}]\Big] \boldsymbol{X} = 0, \\ 
    \end{split}
    \end{equation}
    
where $\boldsymbol{X}= \begin{bmatrix} x_{i} \\ x_{i-1} \\ \vdots \\ x_{1} \end{bmatrix}$. 

Denote $R = \Big[[CB, C\Lambda B, C \Lambda^2 B, \cdots, C \Lambda^{i-1} B]-[\tilde{B}, \tilde{\Lambda}\tilde{B}, \tilde{\Lambda}^2 \tilde{B}, \cdots,  \tilde{\Lambda}^{i-1} \tilde{B}]\Big]$. 
We collect a set of points $\boldsymbol{X}^{+} = [\boldsymbol{X}^{(1)}, \cdots, \boldsymbol{X}^{(l)}]$ where $l\geq ni$ and rank of $\boldsymbol{X}^{+}=ni$ (from Assumption~\ref{assm: rank}).  Since the matrix $\boldsymbol{X}^{+}$ is full rank, we have 
$$R \boldsymbol{X}^{+} = 0 \implies R=0.$$

This yields 

\begin{equation}
    CB = \tilde{B}, C\Lambda B = \tilde{\Lambda}\tilde{B}, \cdots, C \Lambda^{i} B =  \tilde{\Lambda}^{i} \tilde{B}. 
\end{equation}

Observe that from the second equality, we get $\tilde{\Lambda} = C \Lambda C^{-1}$. Given the  parameters $(\Lambda, B)$, the set of parameters $( \tilde{\Lambda}, \tilde{B})$ that solve the first two equalities are  $\{\tilde{B} \text{ is an arbitrary  invertible}$  \newline $\text{matrix,} \;\tilde{\Lambda}=C \Lambda C^{-1},\; \;\text{where } C=\tilde{B}B^{-1}\}$.  Take any solution of the first two equalities and compute

\begin{equation}
\tilde{\Lambda}^{i} \tilde{B} = C \Lambda^{i} C^{-1} \tilde{B} =  C\Lambda^{i}B, \forall i \geq 1  
\label{eqn:ssm2}
\end{equation}

From \eqref{eqn:ssm2} and \eqref{eqn: ssm1}, we obtain that for all $x_{\leq i} \in \mathbb{R}^{n i}$
\begin{equation}
    h(x_{\leq i}) = \omega(\sum_{j=0}^{i-1}\tilde{\Lambda}^{j} \tilde{B} x_{i-j} ) = \omega(C\sum_{j=0}^{i-1}\Lambda^j B x_{i-j} ) = \rho (\sum_{j=0}^{i-1}\Lambda^j B x_{i-j} ) = f(x_{\leq i})
\end{equation}

This establishes both compositional and length generalization. Before we finish the proof, we make the observation on linear identification.  From equation \eqref{eqn:ssm2} we obtain that for all $x_{\leq i} \in \mathbb{R}^{n i}$  
\begin{equation}
    \sum_{j=0}^{i-1}\tilde{\Lambda}^{j} \tilde{B} x_{i-j}= C(\sum_{j=0}^{i-1}\Lambda^{j} B x_{i-j})
\end{equation}
Recall that $\sum_{j=0}^{i-1}\tilde{\Lambda}^{j} \tilde{B} x_{i-j}=\tilde{h}_j$ and $\sum_{j=0}^{i-1}\Lambda^{j} B x_{i-j}=h_j$. From this it follows that $\tilde{h}_j = Ch_j$, which proves that the learned hidden states are a linear transform of the hidden states underlying the labeling function.  This establishes linear identification.

\end{proof}

 Similar to previous theorems, the above proof shows how the hidden state estimated by the learned model linearly identifies the true hidden state $\tilde{h}_t$. 
 % We extend Theorem~\ref{thm: ssm} to discrete tokens in Theorem~\ref{thm3:disc}. 

\paragraph{High capacity SSMs} In the above result, we operated with certain constraints on SSMs, i.e., the input dimension, output dimension, and hidden state dimension are equal. These constraints limit their capacity.  What happens if we put no constraints on $\Lambda$, $B$, and $\omega$?   \cite{orvieto2023universality} showed that SSMs with appropriately large $\Lambda$ and $B$ matrices can approximate a sequence-to-sequence mapping up to some length with arbitrary precision. Consider the true labeling function $f$ and another function $h$, which is equal to $f$ for all sequences of length up to $T$ and $f+c$ for longer lengths.  If we use such arbitrary capacity SSMs as our hypothesis class, then this hypothesis class contains both $f$ and $h$. As a result, $h$ is a solution to \eqref{eqn: risk_min} and does not achieve length generalization. We can extend the same argument to compositional generalization as well.

\paragraph{Role of chain-of-thought data} Similar to the previous section, we can show that CoT supervision helps achieve length generalization for high capacity SSMs, i.e., going beyond the limited capacity Assumption~\ref{assm: ssm}.
In the context of SSMs, the hidden state $h_t$ serves as the intermediate supervision. By equating the true and learned hidden state, we obtain $\Lambda h_{t-1} + B x_t = \tilde{\Lambda} h_{t-1} + \tilde{B} x_t \implies (\Lambda -\tilde{\Lambda}) h_{t-1} + (\tilde{B}-B)x_t =0$. If this identity is true for a diverse set of  $[h_{t-1}, x_t]$ vectors, then we obtain that $\tilde{\Lambda} = \Lambda$, and $\tilde{B} = B$. From this we know that the recurrent dynamics of hidden state should be identical for all time steps. Since the labels are equal as well, we obtain $\rho(h_t) = \omega(h_t), \forall t\leq T$. If the support of $h_t$ across $t$ spans $\mathbb{R}^k$, then that ensures $\rho=\omega$. From these conditions, we can conclude that the predictions and the true labels match at all time steps. Further details can be found in Appendix Section~\ref{sec:cot_appendix}.

In this section, we discussed models with limited capacity (Theorem~\ref{thm: ssm}) and models with high capacity (with and without CoT). In Section~\ref{sec: ghc}, we present results that capture models that loosely speaking are in between the limited-capacity models of Theorem~\ref{thm: ssm} and high-capacity models from the discussion above.
\subsection{Vanilla recurrent neural networks}
\label{sec: vrnn}
Standard RNNs have a nonlinear recurrence unlike the linear recurrence studied in the previous section. We use the same notation as the previous section and only add an activation for non-linear recurrence. We illustrate the dynamics of data generation below.

\begin{equation}
    \begin{split}
       & h_1 = \sigma(B x_1); \;\; \;h_2 = \sigma(\Lambda h_1 + Bx_2); \; \cdots,  h_{T} = \sigma(\Lambda h_{T-1} + B x_{T}) \; \\ 
        & y_1 = \rho(h_1); \;\; \;\;\;\; y_2 = \rho(h_2); \; \cdots \cdots \cdots, \;\;\;\;\; y_{T} = \rho(h_T).
    \end{split}
    \label{eqn: vrnn}
\end{equation}

\begin{assumption}
Each function in the hypothesis class $\mathcal{H}$  is a vanilla RNN of the form \eqref{eqn: vrnn}, where the position-wise non-linearity that acts on hidden state is a single layer perceptron $\sigma \circ A$,  and $\Lambda, B$ govern the hidden state dynamics (\eqref{eqn: vrnn}). $A, \Lambda, B$ are square invertible matrices, and $\sigma$ is the sigmoid activation. 
\label{assm: vrnn}
\end{assumption}

\begin{restatable}{theorem}{vrnn}
\label{thm:vrnn}
    If  $\mathcal{H}$ follows  Assumption~\ref{assm: vrnn}, and the realizability condition holds, i.e., $f\in \mathcal{H}$ and the regular closedness condition in Assumption~\ref{assm: reg_closed} holds, then the model trained to minimize the risk in \eqref{eqn: risk_min} with $\ell_2$ loss (with $T\geq 2$) achieves length and compositional generalization. 
\end{restatable}
\begin{proof}
We start with the same steps as earlier proofs and equate the prediction of $h$ and $f$ everywhere in the support of the training distribution (using first part of Lemma~\ref{lemma1}). 
 We start with equating label at length 1, i.e., $y_1$. For all $x_1 \in \mathsf{supp}(X_1)$
    \begin{equation}
    \begin{split}
&  \sigma(A\sigma(B x_1)) = \sigma(\tilde{A} \sigma(\tilde{B}x_1)) \implies       A\sigma(B x_1) = \tilde{A} \sigma(\tilde{B}x_1) \implies     \sigma (B \tilde{B}^{-1} \tilde{B}x_1) = A^{-1}\tilde{A}\sigma (\tilde{B}x_1)
    \end{split}
    \end{equation}
    Say $y=\tilde{B}x_1$,  $A^{-1}\tilde{A}=U$, $B \tilde{B}^{-1}=V$.  We substitute these expressions in the simplificaction below.  We pick a $y$ in the interior of $\tilde{B} \cdot \mathsf{supp}(X_1)$.   

    \begin{equation}
        \sigma(Vy) = U \sigma(y)
    \end{equation}
    Take the first row of $V$ and $U$ as $v^{\top}$ and $u^{\top}$ to obtain 
     \begin{equation}
        \sigma(v^{\top}y) = u^{\top} \sigma(y)
    \end{equation}
    Suppose there is some non-zero component of $v$ say $i$ but the corresponding component is zero in $u$.

\begin{equation}
    \begin{split}
        \frac{\partial \sigma(v_i y_i + v_{-i}y_{-i})}{\partial y_i} = \sigma^{'}(v_i y_i + v_{-i}y_{-i})v_i  =  \frac{ \partial u_{-i}^{\top}\sigma(y_{-i})}{\partial y_i} = 0
    \end{split}
\end{equation}
From the above we get $\sigma^{'}(v^{\top}y)=0$. But sigmoid is strictly monotonic on $\mathbb{R}$, $\sigma^{'}(x)>0, \forall x \in \mathbb{R}$ and $v^{\top} y \in \mathbb{R}$. Hence, $\sigma^{'}(v^{\top}y)=0$ is not possible. 
Similarly, suppose some component is non-zero in $u$ and zero in $v$.

\begin{equation}
    \begin{split}
        \frac{\partial \sigma( v_{-i}^{\top}y_{-i})}{\partial y_i} = 0  =  \frac{\partial (u_i \sigma(y_i) +   u_{-i}^{\top}\sigma(y_{-i}))}{\partial y_i} = u_i \sigma^{'}(y_i)
    \end{split}
\end{equation}
Since the derivative of $\sigma$ cannot be zero, the above condition cannot be true.

From the above, we can deduce that both $u$ and $v$ have same non-zero components. 

Let us start with the case where $p\geq 2$  components of $u,v$ are non-zero. Below we equate the partial derivative w.r.t all components of $y$ that have non-zero component in $u$ (since $y$ is in the interior of the image of $\tilde{B}x_1$, we can equate these derivatives).  

\begin{align} 
\begin{split}
      \sigma(v^{\top}y) &= u^{\top} \sigma(y), \\ 
      \frac{\partial^{p} \sigma(s)}{\partial s^{p}} \Pi_{v_i\not=0} v_i &= 0 \implies \frac{\partial^{p} \sigma(s)}{\partial s^{p}}  = 0.\\ 
\end{split}
\end{align}

Since support of $X_1$ has a non-empty interior, the set of values $v^{\top}y$ takes also has a non-empty interior in $\mathbb{R}$. Hence, the above equality is true over a set of values $s$, which have a non-empty interior. Since $\sigma(s)$
 is analytic, $ \frac{\partial^{p} \sigma(s)}{\partial s^{p}}$  is also analytic. From  \citep{mityagin2015zero}, it follows that $ \frac{\partial^{p} \sigma(s)}{\partial s^{p}} =0$ everywhere. From Lemma~\ref{lemma2}, we know this condition cannot be true.

We are left with the case where $u$ and $v$ have one non-zero component each. 

$$\frac{1}{1+e^{-vy}} = \frac{u}{1+e^{-y}} \implies 1+ e^{-y} = u + ue^{-vy}$$

In the simplification above, we take derivative w.r.t $y$ to obtain $e^{-(v-1)y} = 1/uv$. We now  again take derivative again w.r.t $y$ to get $v=1$ and substitute it back to get $u=1$. Note that no other row of $U$ or $V$ can have same non-zero element because that would make matrix non-invertible. From this we deduce that $U$ and $V$ are permutation matrices. From $ \sigma(Vy) = U \sigma(y)$ it follows that $U=V=\Pi$. Thus $B =\Pi \tilde{B}$ and $\tilde{A} = A\Pi$.

Next, we equate predictions for $y_2$ to the ground truth (label $y_2$ exists as $T\geq 2$). For all $x_1,x_2 \in \mathsf{supp}(X_1,X_2)$
    \begin{equation}
    \begin{split} 
& \sigma(A\sigma(\Lambda \sigma(B x_1) + Bx_2)) = \sigma(\tilde{A} \sigma( \tilde{\Lambda} \sigma(\tilde{B}x_1) + \tilde{B}x_2)) \implies         A\sigma(\Lambda \sigma(B x_1) + Bx_2) = \tilde{A} \sigma( \tilde{\Lambda} \sigma(\tilde{B}x_1) + \tilde{B}x_2) \\ 
&  \implies \tilde{A} \sigma( \tilde{\Lambda} \sigma(\tilde{B}x_1) + \tilde{B}x_2) = A \Pi  \sigma(\tilde{\Lambda} \Pi^{\top}\sigma(Bx_1) + \Pi^{\top}Bx_2)= A \sigma(\Pi \tilde{\Lambda} \Pi^{\top}\sigma(Bx_1) + Bx_2). \\ 
    \end{split}
    \end{equation}
We use the simplification in the second step to equate to LHS in the first step as follows. 

\begin{equation}
\begin{split}
   & A \sigma(\Pi \tilde{\Lambda} \Pi^{\top}\sigma(Bx_1) + Bx_2) = A\sigma(\Lambda\sigma(B x_1) + Bx_2)  \implies (\Pi \tilde{\Lambda} \Pi^{\top} - \Lambda)\sigma(Bx_1)  = 0.  \\ 
   \label{eqn:rnnl}
\end{split}
\end{equation}

From Assumption~\ref{assm: reg_closed}, it follows that $\mathsf{supp}(X_1)$ is non-empty and regular closed, which implies $\mathsf{supp}(X_1)$ has a non-empty interior.  Since $\sigma \circ B$ is a homeomorphism, it follows that  
$\sigma(B \mathsf{supp}(X_1))$ spans a set that has a non-empty interior. From equation~\ref{eqn:rnnl} and Lemma~\ref{lemma3}, it follows that $\tilde{\Lambda} = \Pi^{\top} \Lambda \Pi$.

From the above conditions, we have arrived at $\tilde{\Lambda} =\Pi^{\top} \Lambda \Pi, \tilde{B}= \Pi^{\top}B, \tilde{A} = A\Pi$. 

We want to show that for all $k\geq 1$

\begin{equation}
     h_{k} = \Pi\tilde{h}_{k}, 
     \label{eqn: perm_id_rnn}
\end{equation}
where $h_k = \sigma(\Lambda h_{k-1} + Bx_k)$ and $\tilde{h}_k = \sigma(\tilde{\Lambda}\tilde{h}_{k-1} + \tilde{B}x_k)$ and $h_{0}=\tilde{h}_0=0$. In other words, we define $T_k$ as a mapping that takes $x_{\leq k}$ as input and outputs $h_k$, i.e., $T_k(x_{\leq k})= h_k$. Similarly, we write $\tilde{T}_k(x_{\leq k}) = \tilde{h}_k$. We want to show 
\begin{equation}
T_k = \Pi \tilde{T}_k,  \forall k
\label{eqn: tk_cont}
\end{equation}

We show the above by principle of induction. Let us consider the base case below. For all $x_1 \in \mathbb{R}^{n}$

\begin{equation}
   \tilde{A} \sigma(\tilde{B}x_1) =   A \Pi \sigma(\Pi^{\top}B x_1) = A\sigma(Bx_1)  = Ah_1 \implies h_1 = \Pi \tilde{h}_1 \implies T_1(x_1) = \Pi \tilde{T}_1(x_{1}) 
\end{equation}

Suppose $\forall j \leq k, T_j = \Pi \tilde{T}_{j}$.  Having shown the base case and assumed the condition for $j\leq k$, we now consider the mapping $\tilde{T}_{k+1}$
\begin{equation}
\Pi\tilde{T}_{k+1}(x_{\leq k+1 }) = \Pi\sigma(\tilde{\Lambda} \tilde{h}_k + \tilde{B}x_{k+1}) = \Pi \sigma(\Pi^{\top}\Lambda \Pi \tilde{h}_k + \Pi^{\top} Bx_k) =  \sigma(\Lambda h_k +  Bx_k) =  T_{k+1}(x_{\leq k+1 }).
\label{eqn: induction_k_vrnn}
\end{equation}

The prediction from the model $(\tilde{A}, \tilde{\Lambda}, \tilde{B})$ at a time step $k$ is denoted as $\tilde{y}_k$ and it relates to $\tilde{h}_k$ as follows $\tilde{y}_k=  \sigma(\tilde{A}\tilde{h}_k)$. We use the above condition in equation~\eqref{eqn: tk_cont} to arrive at the following result. For all $x_{\leq k} \in \mathbb{R}^{nk}$

$\tilde{y}_k = \sigma(\tilde{A}\tilde{h}_k) = \sigma(\tilde{A} \tilde{T}(x_{\leq k})) = \sigma(A \Pi \tilde{T}(x_{\leq k})) = \sigma(A T (x_{\leq k})) = y_k $

Before we finish the proof, we provide an argument that shows the relation between the learned and true representation.  From \eqref{eqn: induction_k_vrnn} it follows that for all $x_{\leq k} \in \mathbb{R}^{nk}$
\begin{equation}
     T_{k}(x_{\leq k })  = \Pi \tilde{T}(x_{\leq k}) \implies h_k =\Pi \tilde{h}_k
     \label{eqn: perm_id}
\end{equation}

The above implies permutation identification. 
\end{proof}

From equation \eqref{eqn: perm_id} it follows that the learned state $\tilde{h}_t$, and the true hidden state $h_t$, bear a linear relationship, where the linear relationship is a permutation map. We extend Theorem~\ref{thm:vrnn} to discrete tokens in Theorem~\ref{thm4:disc}.

\paragraph{High capacity RNNs} In our result above, similar to previous sections we showed that limited capacity RNNs can achieve length and compositional generalization. How about RNNs with arbitrary capacity, that is, without constraints on $\Lambda, B$ and $\rho$? These systems can approximate sequence-to-sequence models to arbitrary precision \citep{sontag1992neural, guhring2020expressivity}. Hence, we can use the same argument as previous sections to argue that if $\mathcal{H}$ corresponds to RNNs with arbitrary capacity, then there exist solutions to \eqref{eqn: risk_min} that do not achieve length  and compositional generalization.

\paragraph{Role of chain-of-thought data} For RNNs, we can make very similar arguments as we made for SSMs and show that CoT supervision in the form of availability of $h_t's$ at the training time ensures length generalization (provided certain diversity conditions are satisfied).  Further details on this argument can be found in the Appendix. 

In this section, we discussed models with limited capacity (Theorem~\ref{thm:vrnn}) and models with high capacity (with and without CoT), as discussed above. In Section~\ref{sec: ghc}, we present results that capture models that  loosely speaking are between the limited-capacity models of Theorem~\ref{thm:vrnn} and the high-capacity models discussed above.

\subsection{Experiments}
\label{sec:expmts}
We present the empirical evaluation of compositional and length generalization capabilities of the architectures from the previous section.   All the experiments are carried out in the realizable case where $f\in\mathcal{H}$, i.e., depending on the architecture in question, we use a random instance of the architecture to generate the labels. We train a model $h$ from the same architecture class to minimize the  $\ell_2$ loss between $h$ and $f$. Under different scenarios, we ask if $h$ achieves length generalization and compositional generalization. We also seek to understand the relationship between the hidden representations of $h$ and hidden representations of $f$.

\subsubsection{Length generalization}

We sample sequences $x_{\leq t}$ of varying length with a maximum length of $T=10$. Each token $x_i\sim \text{Uniform}[0,1]^{n}$, where $n=20$. The sequences are then fed to the labeling $f$, which comes from the hypothesis class of the architecture, to generate the labels. We minimize the empirical risk version of \eqref{eqn: risk_min} over the same hypothesis class with $\ell_2$ loss. For evaluation, we present the $\ell_2$ loss on the test datasets. We also evaluate $R^2$ of linear regression between the learned hidden representations denoted $\psi(x_i)$ and the true hidden representations $\phi(x_i)$ for all $x_i\in x_{\leq t}$ from the test dataset sequences. This metric is often used to evaluate linear identification claims \citep{khemakhem2020variational}, i.e. the higher this value, the closer the linear relationship. We present results averaged over five seeds for models with \textit{two} hidden layer MLPs for $\rho$ ($\phi$ is two hidden layer MLP for deep sets).  Figure \ref{plot: lg-2} shows a very small test loss of models on increasing sequence lengths when only trained with sequences of up to length $T=10$, which is in agreement with Theorems~\ref{thm1}-\ref{thm:vrnn}. Further, in Figure~\ref{fig:lg-transformer-2}, we show an exemplar sequence from test set and how the trained transformer tracks it.  Table \ref{tab:lg-1-R2} shows the average of $R^2$ score of $\psi(x_i), \phi(x_i)$ across the different positions $i$ at the test time. These results demonstrate a linear relationship between learned and true hidden representations, which agrees with our theoretical claims. In Section~\ref{sec: experiments_details}, we show that when the realizability condition does not hold, i.e., $f\not\in \mathcal{H}$, then length generalization is not achieved. We also present \emph{additional experiments with  failures in the high capacity settings}, and other experimental details in Section~\ref{sec: experiments_details}. 

%Due to space limitations, the experiments with discrete tokens are not presented but they present the same trend as the results in this section.

\begin{figure}
    \centering
    \includegraphics[width=1.0\textwidth]{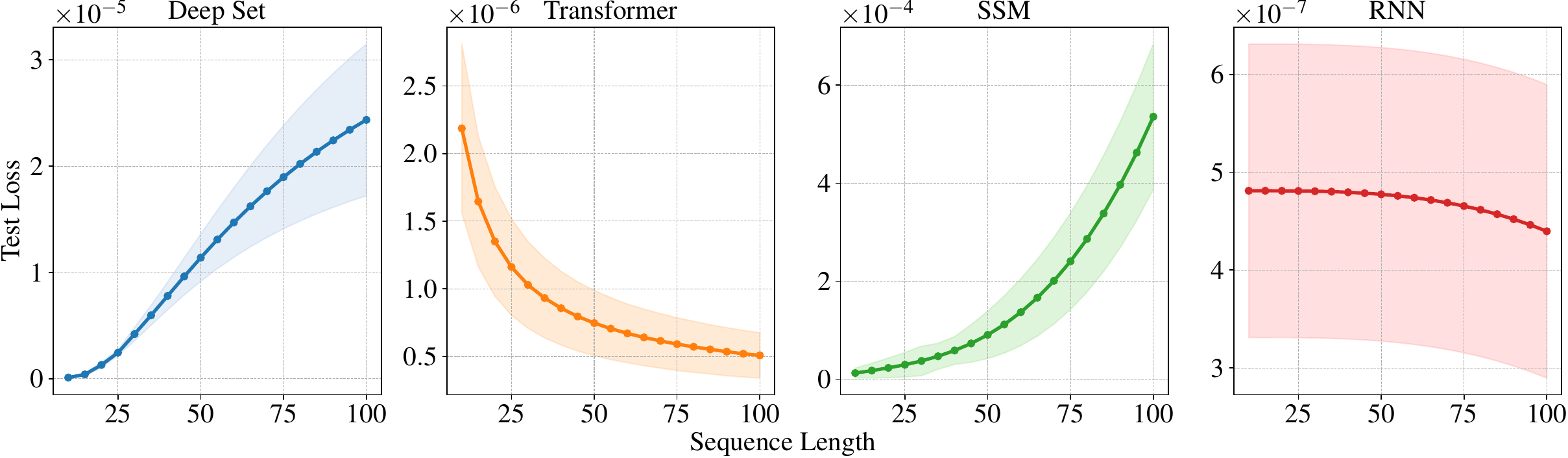}
    \caption{Length generalization: Test $\ell_2$ loss on sequences of different lengths. The models are trained only on sequences of length up to $T=10$. All models achieve small error values $\approx 10^{-4}-10^{-7}$ at all sequence lengths and thus length generalize. Since the error values are already quite small, the increasing or decreasing trends are not numerically significant.}
    \label{plot: lg-2}
\end{figure}
\begin{figure} % 'r' for right, 'l' for left
    \centering
    \includegraphics[width=0.5\textwidth]{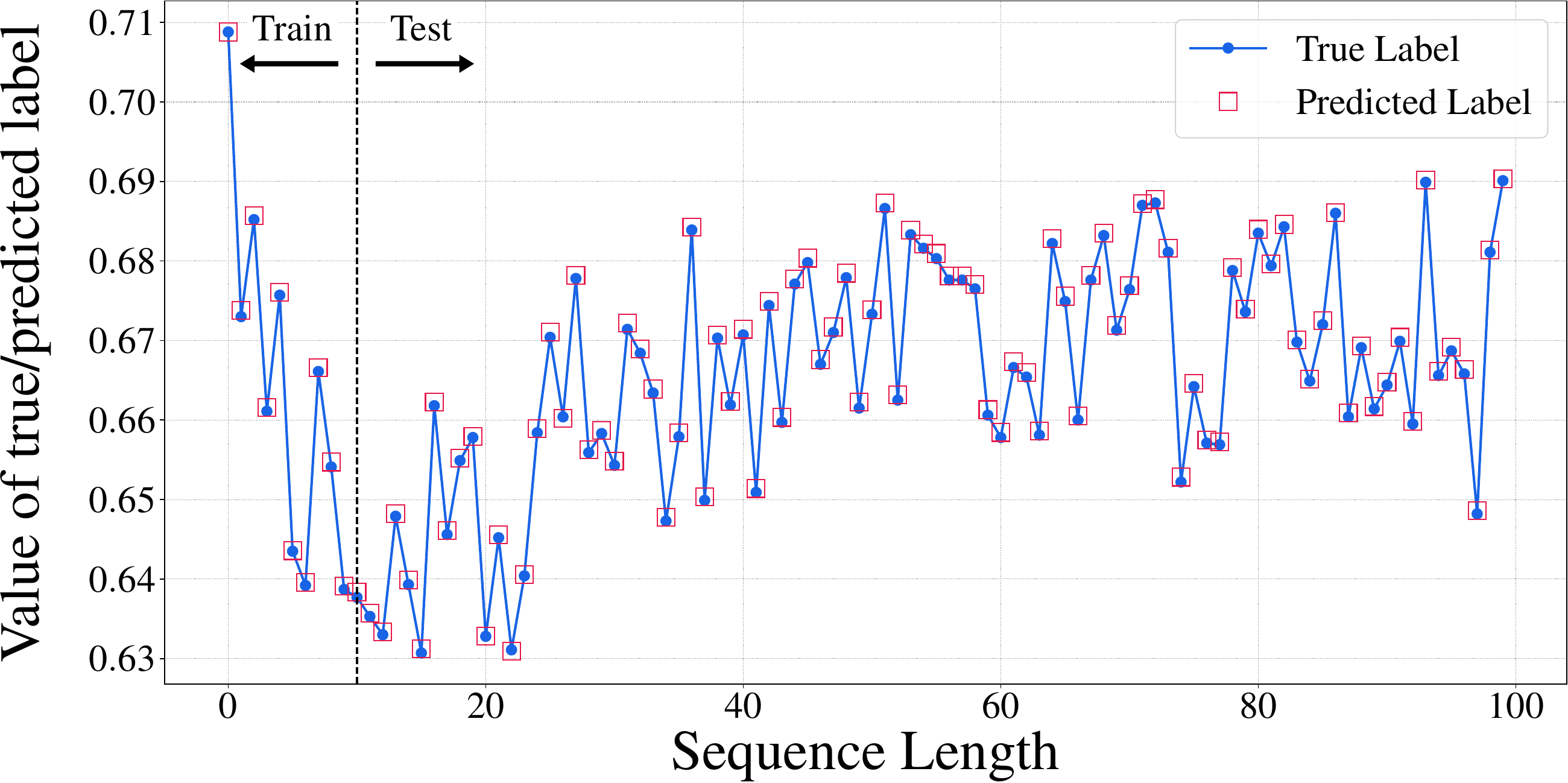}
    \caption{A transformer model with softmax attention with \textit{two} hidden layer MLP for $\omega$ trained on sequences of length up to $T=10$ length generalizes to sequences of length up to $100$.}
    \label{fig:lg-transformer-2}
\end{figure}
\begin{table}[h]
\centering
\renewcommand{\arraystretch}{1.2}

\begin{minipage}[t]{0.45\textwidth}
\centering
\begin{tabular}{lcc}
\hline
\rowcolor{lightblue}
Model & $R^2$ ($t=20$) & $R^2$ ($t=100$) \\
\hline
Deep set & $0.97 \pm 0.01$ & $0.97 \pm 0.01$ \\
Transformer & $0.99 \pm 0.01$ & $0.99 \pm 0.01$ \\
SSM & $0.99 \pm 0.01$ & $0.99 \pm 0.01$\\
RNN & $0.99 \pm 0.01$ & $0.99 \pm 0.01$\\
\hline\\
\end{tabular}
\caption{Average test $R^2$ of true and learned hidden representations $\psi(x_i),\phi(x_i)$ across all positions $i$ at various lengths unseen during training. A strong linear relationship is observed for all models across lengths.}
\label{tab:lg-1-R2}
\end{minipage}%
\hspace{0.05\textwidth} % Adjust horizontal space between the tables
\begin{minipage}[t]{0.45\textwidth}
\centering
\begin{tabular}{lcc}
\hline
\rowcolor{lightblue}
Model & Test Loss $\times 10^{ 6 }$ & $R^2$ \\
\hline
Deep set & $0.08 \pm 0.02$ & $0.96 \pm 0.01$\\
Transformer & $3.06 \pm 1.11$ & $1.00 \pm 0.00$\\
SSM & $5.92 \pm 2.47$ & $1.00 \pm 0.00$\\
RNN & $0.35 \pm 0.17$ & $0.96 \pm 0.01$\\
\hline\\
\end{tabular}
\caption{Compositional generalization: Test $\ell_2$ loss and $R^2$ score for models with \textit{two} hidden layers on sequences of length $T=10$. A strong linear relationship is observed for all models for new sequences made of unseen token combinations.}
\label{tab:cg-2}
\end{minipage}

\end{table}

\subsubsection{Compositional generalization}
\label{subsection:cg}
For compositional generalization, we generate data following the illustration in Figure~\ref{fig1}a.  During training, we sample each component $k$ of a token from $\text{Uniform}[0,1]$ and accept the sampled sequences that satisfy the following for all components $i$:
$ -0.5 \leq \sum_{j=1}^{T}  (x_j^k - 0.5) \leq 0.5 \hspace{2mm} \forall k$, where $x_j^{k}$ is the $k^{th}$ component of token $j$.   During testing, we sample $x_{\leq t}$ from the complementary set of the training set, that is, the corners of the hypercube $[0,1]^{nt}$. We present the $\ell_2$ loss on the test dataset, as well as the mean $R^2$, where the results are averaged over 5 seeds.  The rest of the details are the same as the previous section, i.e.,  $T=10$, $n=20$, $\rho$ is a two-layer MLP ($\phi$ is also a two hidden layer MLP for deep sets). Table \ref{tab:cg-2} shows the test $\ell_2$ loss and $R^2$ scores for linear identification.   
The code to reproduce the experiments in this section is available at \url{https://github.com/facebookresearch/Length-and-Compositional-Generalization}.

\section{General Hypothesis Classes}
\label{sec: ghc}

In the previous sections, we studied the different architectures separately under structural constraints (e.g., one block transformer, limit on the hidden embedding dimension for RNN and SSMs). In this section, we present a more unified approach to studying the different architectures. In this section, we focus only on length generalization guarantees. 

\subsection{Finite Hypothesis Class}

In the discussion so far, we have focused on different hypothesis classes $\mathcal{H}$ of infinite size. In this section, we focus on the finite hypothesis class, i.e., the set $\mathcal{H}$ has a finite size. We can construct such a finite hypothesis class for any architecture by restricting  the parameter vectors (weights, biases, etc.) to assume a finite set of values. Each possible parameter configuration denotes one distinct element in $\mathcal{H}$. Unlike the previous sections, we do not impose any further restrictions on $\mathcal{H}$ other than the finite size. This allows us to consider arbitrary sequence-to-sequence models -- RNNs, deep sets, transformers (e.g., with hard-coded positional encodings as in \citep{vaswani2017attention}) without restrictions on the depth and width as seen in the previous sections.

% We do not restrict ourselves to certain model class, e.g., one block transformer, deep set.

\begin{restatable}{theorem}{finitehypothesis}
\label{thm: fh} If $\mathcal{H}$ is a finite hypothesis class, the realizability condition holds, i.e., $f\in \mathcal{H}$,  then $\exists\; T_0<\infty$ such that the model trained to minimize the risk in \eqref{eqn: risk_min} with $\ell_2$ loss  and $T>T_0$ achieves length generalization. 
\end{restatable}

\begin{proof} Due to realizability assumption, observe that $\min_{h \in \mathcal{H}} R(h,T) =0$ for all $T\geq 1$. 
Consider a function $h \in \mathcal{H}$. There can be two possible scenarios for $h$. Either $R(h,T)=0$ for all $T\geq 1$ or $\exists\; T_h, \; 1 \leq T_h<\infty,   R(h,T_h)>0$. Also, due to the definition of $R(h,T)  \geq R(h,T-1)$, we can conclude that $R(h,T)>0$ for all $T\geq T_h$. If we set $T\geq T_h$, then the set of solutions to expected risk minimization cannot contain $h$ (this follows from the observation $\min_{h \in \mathcal{H}} R(h,T) =0$) . Define $\tilde{\mathcal{H}}$ to be the set of all hypothesis classes for which there is some $T_h$ beyond which $R(h,T)>0$. Define $T^{\star} = \max_{h \in \tilde{\mathcal{H}}} T_h$. Since $\tilde{\mathcal{H}}$ is finite, $T^{\star}<\infty$. If we set $T>T^{\star}$, then all the solutions to \eqref{eqn: risk_min} for $T>T^{\star}$ cannot contain any element from $\tilde{\mathcal{H}}$.  
\end{proof}

Observe that the set $\mathcal{H} \setminus \tilde{\mathcal{H}}$ contains many equivalent ways of representing $f$. This set is similar to the set of solutions in the previous section, which have different weight configurations, but generate the same output as $f$ and linearly identify the representations of $f$.  In the above result, the value of the threshold on $T$, i.e., $T_0$, can be very large. In contrast, under the structural restrictions on $\mathcal{H}$ considered in the previous section, we obtain lower thresholds on $T_0$ (e.g., two or three).

\subsection{Infinite Hypothesis Class}

We now extend the discussion from the finite hypothesis class to the infinite hypothesis class $\mathcal{H}$.  We start by explaining the key challenge that leads to the core assumption that we make in this section. Each model $h \in \mathcal{H}$ can be broken down into functions that operate on fixed-dimensional inputs corresponding to each sequence length. We denote these functions by $\{h_t\}_{t=1}^{\infty}$. Using these functions, we can construct hypothesis classes of functions operating on fixed-dimensional inputs denoted $\{\mathcal{H}_t\}_{t=1}^{\infty}$.  From Theorem~\ref{thm: fh}, we know that length generalization is achievable for finite hypothesis classes. We also want to translate similar results for infinite hypothesis classes, but we need to bound the capacity of $\mathcal{H}$. If we can construct a finite cover that covers all hypothesis classes in $\{\mathcal{H}_t\}_{t=1}^{\infty}$ simultaneously, then we can bring forward the intuition from finite hypothesis classes to infinite hypothesis classes. Next, we describe the uniform Lipschitz assumption under which such a cover is possible.

Suppose that each function $h$ is parameterized by $\theta$ written as $h(x_{\leq t} ; \theta)$.  Each $x_j \in \mathcal{X}$ and thus $x_{\leq t} \in \mathcal{X}^{t}$, where $\mathcal{X}^{t} = \mathcal{X}\times \cdots \times\mathcal{X}$ (repeated $t$ times). Suppose that $\Theta$ is the set of all the parameters that correspond to all the functions in $\mathcal{H}$.

\begin{assumption}\textbf{Uniform Lipschitzness:}
\label{asmm: Lipschitz}
    Each function $h$ is Lipschitz continuous in the parameters $\Theta$ and there is a uniform Lipschitz constant that captures all possible sequences, i.e., $\|h(x_{\leq t}; \theta) - h(x_{\leq t}; \theta^{'})\|  \leq L\|\theta-\theta^{'}\|$, $\forall x_{\leq t} \in \mathcal{X}^{t}, \forall t,$ and $\forall \theta, \theta^{'} \in \Theta$.    
\end{assumption}

 Let us construct an $\eta$-cover for $\Theta$ denoted $\Theta_c$. As a result, for every $\theta \in \Theta,$ $\exists$ $\theta_k \in \Theta_c$ such that $\|\theta - \theta_k\| \leq \eta$. If $\Theta$ is a bounded set, i.e., $\forall \theta \in \Theta, \|\theta\| \leq \theta_{\mathsf{sup}} < \infty$, then we can construct a finite cover \citep{shalev2014understanding}. Owing to Lipschitzness, if $\eta \leq \frac{\epsilon}{L}$, then we obtain an $\epsilon$-cover for all $\{\mathcal{H}_t\}_{t=1}^{\infty}$. In Section~\ref{sec: rademacher_assm10}, we show how functions that satisfy Assumption~\ref{asmm: Lipschitz} with a bounded cover have a finite Radamecher complexity that grows in the size of the cover.  In the following, we discuss the conditions under which the different architectures -- deep sets, transformers, SSMs, and RNNs, follow Assumption~\ref{asmm: Lipschitz}.

\textbf{Multiblock transformers:} Let us consider one of the transformer blocks $u(x_{\leq t}; \theta) = \omega\big( \frac{1}{i}\sum_{j=1}^{i}\psi(x_i, x_j)\big)$. We suppose $\omega$ is parametrized by parameter $\theta_{\omega}$ and $\psi$ is parametrized by parameters $\theta_{\psi}$ and $\theta= [\theta_{\omega}, \theta_{\psi}]$. If $\omega$ is Lipschitz in the inputs and the parameter and if $\psi$ is Lipschitz in the parameters, then $u(\cdot)$ satisfies Assumption~\ref{asmm: Lipschitz}. See the justification below. 

Consider the function $u(x_{\leq t}; \theta) = \omega ( \frac{1}{t}\sum_{j=1}^{t} \psi(x_j; \theta_\psi); \theta_{\psi}) $. Let us denote $z_j = \frac{1}{t}\sum_{j=1}^{t} \psi(x_j; \theta_\psi)$. To prove Lipschitzness, consider two parameters, $\theta=[\theta_\omega, \theta_\psi]$ and $\theta^{'}=[\theta_\omega^{'}, \theta_\psi^{'}]$. 

\begin{equation}
\begin{split}
   \Big\|u(x_{\leq t}; \theta) - u(x_{\leq t}; \theta^{'})\Big\| & = \Big\|\omega \Big( \frac{1}{n}\sum_{j=1}^{n} \psi(x_j; \theta_\omega); \theta_{\psi}\Big) - \omega \Big( \frac{1}{n}\sum_{j=1}^{n} \psi(x_j; \theta_\psi^{'}); \theta_{\omega}^{'}\Big)\Big\| \\
 & = \Big\| \omega(z_j; \theta_{\omega}) - \omega(z_j^{'};\theta_{\omega}^{'}) \Big\|  =  \Big\| \omega(z_j; \theta_{\omega}) -   \omega(z_j; \theta_{\omega}^{'}) + \omega(z_j; \theta_{\omega}^{'}) -  \omega(z_j^{'};\theta_{\omega}^{'}) \Big\| \\ 
 & \leq   \Big\| \omega(z_j; \theta_{\omega}) -   \omega(z_j; \theta_{\omega}^{'}) \Big\|  + \Big\|\omega(z_j; \theta_{\omega}^{'}) -  \omega(z_j^{'};\theta_{\omega}^{'}) \Big\|  \\ 
 & \leq L_{\omega}\| \theta_{\omega} - \theta_{\omega}^{'} \| + M_{\omega}\| z_j- z_{j}^{'} \| \\ 
 & \leq  L_{\omega}\| \theta_{\omega} - \theta_{\omega}^{'} \| + M_{\omega}\frac{1}{n}\sum_{j=1}^{n}\| \psi(x_j; \theta_\psi)- \psi(x_j; \theta_\psi^{'}) \| \\ 
 & \leq L_{\omega}\| \theta_{\omega} - \theta_{\omega}^{'} \| + M_{\omega} L_{\psi}\| \theta_{\psi}- \theta_{\psi}^{'} \| \\ 
 & \leq \sqrt{L_{\omega}^2 + M_{\omega}^2 L_{\psi}^2}\|\theta-\theta^{'}\|
\end{split}
\end{equation}

Further, let us consider multiblock transformers obtained by composing each block of the form $u$. Since the composition of Lipschitz functions is Lipschitz \citep{federer2014geometric}, these multiblock transformer classes also satisfy Assumption~\ref{asmm: Lipschitz}. As a special case, the above result can also be applied to deep sets and even multiple blocks of deep sets. 

% Note that this
% includes ReLU transformer and other non-linearities but it does not include softmax based transformer whose
% Lipschitz constant is known to grow with sequence length.

\textbf{Multiblock SSMs:} Next, we discuss one block SSM $u(x_{\leq t}; \theta) = \omega\big(\sum_{j=0}^{i-1}\Lambda^{j}Bx_{i-j}\big)$. If $\|B\| \leq B_{\mathsf{sup}}<\infty$, $\|\Lambda\|\leq \Lambda_{\mathsf{sup}}<1$, $\|x_{j}\|\leq x_{\mathsf{sup}}<\infty$, and $\omega$ are Lipschitz both in the parameters and in the inputs, then one block SSM described above satisfies Assumption~\ref{asmm: Lipschitz}. The proof for this follows the same steps as in the RNN setting, which we present in the next paragraph.  We can construct multiblock SSMs by composing different $u$'s and these multiblock SSMs continue to satisfy Assumption~\ref{asmm: Lipschitz}. 

\textbf{RNNs:} Let us consider RNNs in equation \eqref{eqn: vrnn}. If $\omega$ is Lipschitz in the inputs and the parameters, $\sigma$ is Lipschitz, with $L_{\sigma}$ as Lipschitz constant and $\sigma(0)=0$, $\|\Lambda\|<\Lambda_{\mathsf{sup}}$, $\Lambda_{\mathsf{sup}}L_{\sigma}<1$, $\|B\| \leq B_{\mathsf{sup}}<\infty$, and $\|x_{j}\|\leq x_{\mathsf{sup}}<\infty$,  then under these conditions the RNNs satisfy Assumption~\ref{asmm: Lipschitz}. We justify this as follows.

Recall the dynamics of RNNs: $ h_t = \sigma(\Lambda h_{t-1} + B x_t), \;    y_t = \omega(h_t; \theta_\omega) $. To prove Lipschitzness, let us consider two sets of parameters $\theta = [\theta_{\omega}, \Lambda, B]$ and $\theta^{'}= [\theta_{\omega}', \Lambda', B']$. $\theta$ is a vector comprising of $\theta_{\omega}$ concatenated with matrices $\Lambda$ and $B$ flattened out as vectors $\Lambda$ and $B$; we abuse notation here to keep things simple.

\begin{equation} 
\begin{split}
 \|y_t - y_t^{'}\| =  \|  \omega(h_t; \theta_\omega)  - \omega(h_t^{'}; \theta_\omega^{'}) \| & \leq   \|  \omega(h_t; \theta_\omega)  - \omega(h_t; \theta_{\omega}^{'}) + \omega(h_t; \theta_{\omega}^{'}) - \omega(h_t^{'}; \theta_\omega^{'}) \| \\
&   \leq \|  \omega(h_t; \theta_\omega)  - \omega(h_t; \theta_{\omega}^{'}) \| + \| \omega(h_t; \theta_{\omega}^{'}) - \omega(h_t^{'}; \theta_\omega^{'}) \| \\ 
&  \leq  L_{\omega} \|\theta_{\omega}  - \theta_{\omega}^{'}\| + M_{\omega}\|h_{t}-h_{t}^{'}\|
\label{eqn: lipschitz_ssmm}
\end{split}
\end{equation}

Denote $h_{t}-h_{t}^{'}$ as $\Delta h_{t}$. 

\begin{equation}
\begin{split}
    \| \Delta h_t \|  &  = \sigma(\Lambda h_{t-1} + B x_t) -  \sigma(\Lambda' h_{t-1}' + B^{'} x_t)   \leq L_{\sigma}\|  \Lambda h_{t-1} + B x_t -  \Lambda' h_{t-1}' - B^{'} x_t\| \\ 
                      & \leq L_{\sigma}\|  \Lambda h_{t-1}  -  \Lambda' h_{t-1}' \| + L_{\sigma}\|Bx_{t} - B^{'} x_t\| \\ 
                      & \leq L_{\sigma}\|  \Lambda h_{t-1}  - \Lambda h_{t-1}^{'} + \Lambda h_{t-1}'-  \Lambda' h_{t-1}' \| + L_{\sigma}\|B-B^{'}\| \|x_t\| \\ 
                      & \leq L_{\sigma}\|  \Lambda h_{t-1}  - \Lambda h_{t-1}^{'}\| + L_{\sigma}\| \Lambda h_{t-1}'-  \Lambda' h_{t-1}' \| + L_{\sigma}\|B-B^{'}\| \|x_t\| \\ 
                      & \leq L_{\sigma}\|\Lambda\|\|h_{t-1}-h_{t-1}^{'}\|  + L_{\sigma}\| \Lambda -\Lambda'\|\| h_{t-1}' \|+ L_{\sigma}\|B-B^{'}\| \|x_t\| \\ 
                      & \leq L_{\sigma}\|\Lambda\|\|h_{t-1}-h_{t-1}^{'}\|  + L_{\sigma}\| \Lambda -\Lambda'\|h_{\mathsf{sup}}+ L_{\sigma}\|B-B^{'}\| x_{\mathsf{sup}},
\end{split} 
\label{eqn:liprnn}
\end{equation}

where $h_{\mathsf{sup}}$ is the bound on $\|h_{t-1}\|$, which we derive in a bit. Let $\alpha^{\star}= L_{\sigma}\| \Lambda -\Lambda'\| h_{\mathsf{sup}} + L_{\sigma}\|B-B^{'}\| x_{\mathsf{sup}}$. Substitute $\alpha^{\star}$ in \eqref{eqn:liprnn} to obtain

\begin{equation}
\begin{split}
 \| \Delta h_t \| &\leq L_{\sigma}\|\Lambda\|\|\Delta h_{t-1}\| + \alpha^{\star}, \\ 
                   &\leq \frac{\alpha^{\star}}{1-L_{\sigma}\|\Lambda\|}  \leq \frac{\alpha^{\star}}{1-L_{\sigma}\Lambda_{\mathsf{sup}}}.   \\
 \end{split}
 \label{eqn:liprn2}
\end{equation}

In the above simplification, we use $L_{\sigma}\Lambda_{\mathsf{sup}}<1$.  Substitute the above \eqref{eqn:liprn2} into \eqref{eqn: lipschitz_ssmm} to obtain.

\begin{equation}
\begin{split}
   &      \|  \omega(h_t; \theta_\omega)  - \omega(h_t^{'}; \theta_\omega^{'}) \| \\ & \leq L_{\omega} \|\theta_{\omega}  - \theta_{\omega}^{'}\| + M_{\omega} \frac{\alpha^{\star}}{1-L_{\sigma}\Lambda_{\mathsf{sup}}} 
      = L_{\omega} \|\theta_{\omega}  - \theta_{\omega}^{'}\| + M_{\omega} \frac{L_{\sigma}\| \Lambda -\Lambda'\| h_{\mathsf{sup}}+ L_{\sigma}\|B-B^{'}\| x_{\mathsf{sup}}}{1-L_{\sigma}\Lambda_{\mathsf{sup}}} \\ 
& =  L_{\omega} \|\theta_{\omega}  - \theta_{\omega}^{'}\| + \underbrace{M_{\omega}L_{\sigma}^2\frac{B_{\mathsf{sup}}x_{\mathsf{sup}}}{(1-L_{\sigma}\Lambda_{\mathsf{sup}})^2}}_{\gamma_1} \| \Lambda -\Lambda'\| + \underbrace{\frac{M_{\omega}L_{\sigma}x_{\mathsf{sup}}}{1-L_{\sigma}\Lambda_{\mathsf{sup}}}}_{\gamma_2}\|B-B^{'}\| \\
& \leq L_{\omega} \|\theta_{\omega}  - \theta_{\omega}^{'}\| + \gamma_1 \|\Lambda-\Lambda^{'}\|_F + \gamma_2\|B-B^{'}\|_F \\ 
& \leq \sqrt{L_{\sigma}^2 + \gamma_1^2 + \gamma_2^2 } \|\theta-\theta' \|.  
\end{split}
\end{equation}

In the simplification above, we use $\|A\|\leq \|A\|_F $ and  the expression  $h_{\mathsf{sup}}$  obtained as follows. In the simplification below, we use $\sigma(0)=0$ and $L_{\sigma}\Lambda_{\mathsf{sup}}<1$.

\begin{equation}
\begin{split}
   & h_t = \sigma(\Lambda h_{t-1} + B x_t),  \\
   & \|h_t\| \leq L_{\sigma}\Lambda_{\mathsf{sup}} \|h_{t-1}\| + L_{\sigma} B_{\mathsf{sup}}x_{\mathsf{sup}}  \leq  \frac{L_{\sigma}B_{\mathsf{sup}}x_{\mathsf{sup}}}{1-L_{\sigma}\Lambda_{\mathsf{sup}}}= h_{\mathsf{sup}}.\\
\end{split}
\end{equation}

\begin{lemma}
\label{lemma: seq_set}
If $\{\mathcal{A}_{t}\}_{t=1}^{\infty}$ is a sequence of sets, where $\mathcal{A}_{t+1}\subseteq \mathcal{A}_t$ and each $\mathcal{A}_{t}\subseteq \mathcal{A}$. If $|\mathcal{A}|<\infty$, then there exists $T_0$ such that $\mathcal{A}_{t}=\mathcal{A}^{\star}, \forall t \geq T_0$.  
\end{lemma}

\begin{proof} The sequence $\mathcal{A}_t$  indexed by $t$ is non-increasing, the limit of the above sequence exists and is denoted $\mathcal{A}^{\star}$.  $\boldsymbol{1}_{\mathcal{A}_{t}}(h)$ is indicator function takes a value of one if $h \in \mathcal{A}_t$ and zero otherwise. From the indicator function definition of the limit, we can write $\mathcal{A}^{\star}$ as 

    $$\mathcal{A}^{\star} = \{h \in \mathcal{A} : \lim_{t \rightarrow \infty} \boldsymbol{1}_{\mathcal{A}_{t}}(h) = 1 \}$$

   Since the limit of the sequence $\mathcal{A}_t$ exists, for each $h \in \mathcal{A}$, the limit  $\lim_{t \rightarrow \infty} \boldsymbol{1}_{\mathcal{A}_{t}}(h)$ exists denoted as $p(h)$.  Each element of this sequence $\boldsymbol{1}_{\mathcal{A}_{t}}(h)$ indexed by $t$ takes a value of one or zero. From the standard definition of limit, we know that for each $\epsilon$, there exists $t(h, \epsilon)$ such that $t>t(h, \epsilon)$, $|\boldsymbol{1}_{\mathcal{A}_{t}}(h) - p(h)| <\epsilon$. Both $\boldsymbol{1}_{\mathcal{A}_{t}}(h)$ and $p(h)$ can only take a value of $0$ or $1$. If $p(h)$ takes any other value other than $0$ or $1$,  then the distance of sequence terms $\boldsymbol{1}_{\mathcal{A}_{t}}(h)$ from $p(h)$ will be bounded away from zero, which is not possible. If $\epsilon<1$, then  for all $t>t(h, \epsilon)$, $\boldsymbol{1}_{\mathcal{A}_{t}}(h) = p(h)$. 

    Define $T_0= \sup_{h \in \mathcal{A}}t(h, \epsilon)$. Since $\mathcal{A}$ is finite, $T_0 < \infty$.  
    
    We write the set $\mathcal{A}_t$ as 
    
    $$\mathcal{A}_t = \{h \in \mathcal{A}:  \boldsymbol{1}_{\mathcal{A}_{t}}(h) = 1 \}$$
    
    If $t>T_0$, then 
    $$\mathcal{A}_t = \{h \in \mathcal{A}:  p(h) = 1 \}  = \{h \in \mathcal{A}: \lim_{t \rightarrow \infty} \boldsymbol{1}_{\mathcal{A}_{t}}(h) = 1 \} = \mathcal{A}^{\star}$$

\end{proof}

\paragraph{Constrained Learner}  Define the expected loss for sequence of length $t$, $\tilde{R}(h, t) = \mathbb{E}\big[\ell(y_t, h(x_{\leq t}))\big]$. $\tilde{R}$ only computes the expected loss for fixed length and $R(h,T) = \sum_{t=1}^{T}\tilde{R}(h,t)$.  In the next result, we will use an approximate notion of length generalization, where instead of asking for $\tilde{R}(h, t)=0, \forall t\geq 1$, we require $\tilde{R}(h, t)\leq \epsilon, \forall t\geq 1$. Also, instead of operating with expected risk minimization of the form \eqref{eqn: risk_min}, we work with a constrained learner. We search for a hypothesis in the cover that achieves an error of at most $\epsilon$ for each possible sequence length in the training distribution. We define this set as follows  $\{ \theta \in \Theta_c; \;  \tilde{R}(h(;\theta),t) \leq \epsilon,  \;  1 \leq t\leq T\}$. Observe that this learning model ensures that the error for each sequence length is small. This is closely related to collaborative learning \citep{blum2017collaborative}, where we can interpret the distribution at each sequence length as a player. This is also closely related to the group-distributionally robust learner \citep{sagawa2019distributionally}.  In our result next, we prove length generalization guarantees for this alternate learning model that encourages $\epsilon$-optimal learning at all sequence lengths. This result also suggests that it would be quite exciting to study the role of different optimization objectives and their impact on length generalization.

\begin{restatable}{theorem}{infinitehypothesis}
\label{thm: ifh} Suppose $\mathcal{H}$ satisfies the uniform Lipschitz continuity assumption (Assumption~\ref{asmm: Lipschitz}) and also that the realizability condition holds, i.e., $f\in \mathcal{H}$.  Suppose $\Theta_c$ is the $\eta$-cover of $\Theta$ with $\eta \leq \epsilon/L$. $\exists\; T_0<\infty$ such that for all $T>T_0$  the solutions in the set $\{ \theta \in \Theta_c; \;  \tilde{R}(h(;\theta),t) \leq \epsilon,  \;  1 \leq t\leq T\}$  satisfy $\tilde{R}(h(;\theta),t)\leq \epsilon, \forall t \geq 1$. 
\end{restatable}

\begin{proof}
From the realizability condition we know that there is a $\theta^{\star} \in \Theta$ such $h( \cdot, \theta^{\star}) = f$. From the definition of the cover we know that there exists a $\theta \in \Theta_c$ such that $\|\theta-\theta^{\star}\| \leq \eta$. From the Assumption~\ref{asmm: Lipschitz}, it follows that $\|h(x_{\leq t}; \theta) - h(x_{\leq t};\theta^{\star})\| \leq L \| \theta - \theta^{\star}\|  \leq L\eta \leq \epsilon, \forall t\geq 1$. As a result, $\theta \in \Theta_c$ achieves  $\tilde{R}(h(;\theta), t) \leq \epsilon, \forall t\geq 1$.

Define $\Theta_c^{T} = \{ \theta \in \Theta_c; \; \tilde{R}(h(;\theta), t) \leq \epsilon,  1 \leq  t \leq T\}$. Observe that $\Theta_c^{T}$ is a non-increasing sequence of sets.  From Lemma~\ref{lemma: seq_set}, we know that there exists a $T_0$ for which $\Theta_c^{t} = \Theta_c^{\star}$ for $t\geq T_0$, where $\Theta_c^{\star}$ is the set to which the sequence of sets $\Theta_c^{t}$ converges.  Now suppose that there exists some $\tilde{\theta} \in \Theta_c^{\star}$ such that it does not satisfy 
$\tilde{R}(h(;\tilde{\theta}), t) \leq \epsilon, \forall t\geq 1$. 
For such a parameter, we know that there exists some length $\tilde{T}$ for which $\tilde{R}(h(;\tilde{\theta}), \tilde{T})>\epsilon$. As a result, $\tilde{\theta} \not \in \Theta_c^{\tilde{T}}$. Since $\Theta_c^{t}$ are non-decreasing sets, $\Theta_c^{\star}\subseteq \Theta_c^{\tilde{T}}$, which implies $\tilde{\theta} \in \Theta_c^{\tilde{T}}$, which leads to a contradiction. Thus, no such $\tilde{\theta}$ exists. Thus, $\tilde{R}(h(;\theta), t) \leq \epsilon, \forall t\geq 1$ for all $\theta \in \Theta_c^{\star}$.

\end{proof}

\paragraph{Comparison with fundamental results in PAC learning} The result above bears a parallel to the fundamental results in PAC learning, which go as follows. If the complexity of the hypothesis class, measured by the VC dimension for classification tasks, is bounded, then in-distribution generalization is possible, provided that the number of samples in the training data is sufficiently large \citep{shalev2014understanding}. In our setting, if the complexity of the hypothesis class, which we measure by the size of the cover as we deal with regression, is bounded, then length generalization is possible, provided that the length of training sequences is sufficiently large.

\section{Discussion and limitations}
\label{sec: discussion}
Our work is a step toward theoretical foundations of successes and failures of length and compositional generalization in sequence-to-sequence models. We prove simplified versions of the recently proposed RASP conjecture. In our analysis, we make certain assumptions, e.g., on the architectures considered, which motivate some of the important conjectures for future work. The main conjectures go as follows -- 
a) \textbf{Conjecture 1:} Theorem \ref{thm2n} and \ref{thm2_transformers} currently incorporate different non-linear attentions but not the softmax attention. We believe that our result in Theorem~\ref{thm: ifh} can be adapted to address softmax-based attention. More recently \cite{huang2024formal}, which came out after the initial versions of our work, also made progress on Conjecture 1.
b) \textbf{Conjecture 2:} Our results focus on the generalization properties of all possible solutions to risk minimization \eqref{eqn: risk_min}. However, in practice the optimization procedure may be biased towards a subset of those. How does accounting for this bias impact the length generalization guarantees? Does accounting for the bias allow for lower sequence length thresholds $T_0$ than those required by our results? 
Finally, most of our results have relied on the realizability assumption. Our last result in Theorem~\ref{thm: ifh} can be adapted to the non-realizable case. We believe a more general framework that captures non-realizable length and compositional generalization guarantees would be a promising future work.

\bibliography{main_jmlr}

\newpage 

\appendix

\section*{Appendix}

% \addtocontents{toc}{\protect\setcounter{tocdepth}{3}}
% \section{Appendix}
% \addcontentsline{toc}{section}{Appendix Table of Contents}

% \section*{Table of Contents for the Appendix}
% \tableofcontents  % This will list only the sections of the appendix

\label{sec:appendix}

% \section{Illustration of the test support for compositional generalization}

% The notion of compositional generalization we study requires us to evaluate the model on the Cartesian product of the support of individual token distributions. 
% In Figure~\ref{fig:illus_supp}, we give some additional examples besides the ones shown in Figure~\ref{fig1}a to illustrate the difference between the Cartesian product set and the observed support. These examples illustrate the support of the observed training data distribution can be much smaller than the Cartesian product of the support of the individual tokens. 
% \begin{figure}
%     \centering
%     \includegraphics[width=4in]{Figures/Fig_supp.pdf}
%     \caption{Illustration of observed support and its Cartesian product. These examples illustrate the support of the observed training data distribution can be much smaller than the Cartesian product of the support of the individual tokens. }
%     \label{fig:illus_supp}
% \end{figure}

\section{Proofs and Extensions} 
\label{sec: proofs}

\subsection{Proof of Lemma~\ref{lemma1}}
\label{sec: proofoflemma1}

\begin{proof}
    Let us consider the interior of $\mathcal{X}$ and denote it as $\mathcal{X}^{\mathsf{int}}$.  We first argue that the two functions $f$ and $g$ are equal at all points in the interior. Suppose there exists a point $x \in \mathcal{X}^{\mathsf{int}}$  at which $f(x)\not=g(x)$. 
    Consider a ball centered at $x$ of radius $r$ denoted as $B(x,r) \subset \mathcal{X}^{\mathsf{int}}$ (such a ball exists as this point is in the interior of $\mathcal{X}$.). We argue that there exists at least one point $x_1$ in this ball at which $f(x_1)=g(x_1)$.    If this were not the case, then the equality will not hold on the entire ball, which would contradict the condition that the equality $f(x)=g(x)$ can only be violated on a set of measure zero. Note this condition holds true for all $r>0$. Suppose the distance of $x_1$ from $x$ is $r_1\leq r$. Consider another ball with radius $r_2<r_1$ and let $x_2 \in B(x, r_2)$ where the equality holds. By repeating this argument, we can construct a sequence $\{x_k\}_{k \in \mathbb{N}}$ that converges to $x$, where $\mathbb{N}$ is the set of natural numbers. On this sequence, the following conditions hold.
    \begin{equation}
    \begin{split}
      &  f(x_k) = g(x_k), \forall k \in \mathbb{N} 
    \end{split}
    \end{equation}
        Further, from the continuity of $f$ and $g$ it follows that  
    \begin{equation}
        \begin{split}
        \lim_{k \rightarrow \infty} f(x_k) = f(x), \lim_{k \rightarrow \infty} g(x_k) = g(x) \\
    \end{split}
    \end{equation}    
    Combining the above two conditions, we get that $f(x)=g(x)$. This leads to a contradiction since we assumed that $f(x)\not=g(x)$. Thus there can be no such $x$ in the interior at which $f(x)\not=g(x)$. From this it follows that $f(x)=g(x)$ for all $x \in \mathcal{X}^{\mathsf{int}}$. Now let us consider the closure of $\mathcal{X}^{\mathsf{int}}$, which is $\mathcal{X}$ itself since it is a regular closed set. Every point $x \in \mathcal{X}$ in the closure can be expressed as limit of points in  $\mathcal{X}^{\mathsf{int}}$. 
    Consider an $x \in \mathcal{X}$ and from the definition of regular closed set it follows that $\lim_{k\rightarrow \infty} x_k  =x$, where $x_k \in \mathcal{X}^{\mathsf{int}}$. We already know from the fact that $f$ and $g$ are equal in the interior 
    \begin{equation}
        f(x_k) = g(x_k), \forall k \in \mathbb{N} 
    \end{equation}
     From the continuity of $f$ and $g$ it follows 
     \begin{equation}
    \begin{split}
        \lim_{k \rightarrow \infty} f(x_k) = f(x), \lim_{k \rightarrow \infty} g(x_k) = g(x) \\
    \end{split}
    \end{equation}
     Combining the above two we get that $f(x)=g(x)$ for all $x \in \mathcal{X}$. 
     After this we can use Lemma 6 from \citep{lachapelle2023additive} to conclude that 
    $\nabla f(x)= \nabla g(x), \forall x \in \mathcal{X}$. 
    We repeat their proof here for completeness. For all points in the interior of $\mathcal{X}$, it follows that $\nabla f(x)=\nabla g(x), \forall x \in \mathcal{X}^{\mathsf{int}}$. 
    
    Now consider any point $x \in \mathcal{X}$. Since $\mathcal{X}$ is a regular closed set, $\lim_{k \rightarrow \infty} x_k = x$. 
    Since each $x_k$ is in the interior of $\mathcal{X}$ it follows that 
         \begin{equation}
    \begin{split}
        \nabla f(x_k) = \nabla g(x_k), \forall k \in \mathbb{N} 
    \end{split}
    \end{equation}

    From the continuity of $\nabla f$ and $\nabla g$ it follows that 
    \begin{equation}
                \lim_{k \rightarrow \infty} \nabla f(x_k) = \nabla f(x), \lim_{k \rightarrow \infty} \nabla g(x_k) = \nabla g(x) 
    \end{equation}
     Combining the above conditions, we get that $\nabla f(x)=\nabla g(x)$. This completes the proof. 
    % Since $Jf$ and $Jg$ are continuous, we 

\end{proof}

\subsection{Transformers}
\label{sec: transformers_proofs}
We present an extension of Theorem~\ref{thm2n} to incorporate positional encoding in Theorem~\ref{thm2}.  In the next part of this section, we prove Theorem~\ref{thm2_transformers}.  Theorem~\ref{thm_pe_gn}  adapts Theorem~\ref{thm2_transformers} to incorporate positional encodings. Lastly, we also discuss extension of Theorem~\ref{thm2n} and Theorem~\ref{thm2_transformers} to multiple attention heads.

\subsubsection{Extension of Theorem~\ref{thm2n} to incorporate positional encodings}

In what follows, we extend Theorem~\ref{thm2n} to incorporate positional encoding. We start with extension of the hypothesis class to incorporate positional encoding. 

\begin{assumption}
Each function in the hypothesis class $\mathcal{H}$ used by the learner is given as $h(x_1,\cdots, x_i) = \omega \Big(\sum_{j\leq i} \frac{1}{i} \psi_{i-j}(x_i, x_j)  \Big)$, where $\omega$ is a single layer perceptron with continuously differentiable bijective activation (e.g., sigmoid) and each $\psi_k$ is a map that is differentiable. Also, $\psi_{k}=0$ for $k\geq  T_{\mathsf{max}}$, i.e., two tokens that are sufficiently far apart do not interact.  
\label{assm: perm_inv_pe}
\end{assumption}

In the above assumption, we incorporate relative positional encodings by making the function $\psi_{i-j}$ depend on the relative positional difference between token $x_i$ and token $x_j$. We would like to emphasize the reasons why we assume that the tokens that are sufficiently far apart do not interact. Suppose $T_{\mathsf{max}}=\infty$, which implies tokens at all positions interact. As a result, during training since we only see sequences of finite length $T$, we will not see the effect of interactions of tokens that are separated at a distance larger than $T$ on the data generation, which makes it impossible to learn anything about $\phi_{i-j}$, where $i-j \geq T-1$.

In the theorem that follows, we show that we can achieve length and compositional generalization for the above hypothesis class.

\begin{restatable}{theorem}{transformers_pe}
    \label{thm2}
If $\mathcal{H}$ follows  Assumption~\ref{assm: perm_inv_pe}, the realizability condition holds, i.e., $f\in \mathcal{H}$, $\mathsf{supp}(X_i, X_j)=[0,1]^{2n},\; \forall i \not= j \in \{1, \cdots, \infty\}$, the regular closedness condition in Assumption~\ref{assm: reg_closed} holds and $T\geq T_{\mathsf{max}}\geq 2$, then the model trained to minimize the risk in \eqref{eqn: risk_min} with $\ell_2$ loss generalizes to all sequences in the hypercube $[0,1]^{nt}, \; \forall t$ and thus achieves length and compositional generalization. 
\end{restatable}

\begin{proof}

  Consider any $h$ that solves \eqref{eqn: risk_min}. Since $\ell$ is $\ell_2$ loss and realizability condition holds, $f$ is a solution to \eqref{eqn: risk_min}. For all $i\leq T$ and for all $x_{\leq i} \in \mathsf{supp}(X_{\leq i})$ except over a set of measure zero the following condition holds

\begin{equation}
  h(x_{\leq i}) = f(x_{\leq i}) . 
\end{equation}
The above follows from the fact that $h$  solves \eqref{eqn: risk_min}, i.e., $\mathbb{E}[\|h-f\|^2]=0$ and from Theorem 1.6.6. \citep{ash2000probability}.  Since $\mathsf{supp}(X_{\leq i})$ is regular closed, $f,h$ are both continuously differentiable, we can use Lemma \ref{lemma1}, it follows that the above equality holds for all $x_{\leq i} \in \mathsf{supp}(X_{\leq i})$.
From realizability condition it follows that true $f(x_{\leq i}) = \rho\Big(\sum_{k\leq i} \phi_{i-k}(x_i, x_k)\Big)$.  We substitute the parametric forms from Assumption~\ref{assm: perm_inv1} to get

\begin{equation}
    \begin{split}
        \omega \Big(\sum_{k\leq i} \frac{1}{i} \cdot \psi_{i-k}(x_i,x_k)  \Big) =  \rho\Big(\sum_{k\leq i}\frac{1}{i} \cdot \phi_{i-k}(x_i, x_k) \Big). \\  
    \end{split}
\end{equation}

Since $\omega$ and $\rho$ are single layer perceptron with bijective activation $\sigma$. We substitute the parametric form of $\omega$ and $\rho$ to obtain the following condition. For all $x_{\leq i} \in \mathsf{supp}(X_{\leq i})$, 

\begin{equation}
    \begin{split}
       & \sigma \Big(A \sum_{k\leq i} \frac{1}{i} \cdot  \psi_{i-k}(x_i, x_k) \Big) = \sigma\Big(B \sum_{k\leq i}\frac{1}{i} \cdot \phi_{i-k}(x_i, x_k)  \Big)  \implies  \\ 
       & A \sum_{k\leq i} \psi_{i-k}(x_i, x_k)  = B \sum_{k\leq i}\phi_{i-k}(x_i, x_k). \\ 
    \end{split}
\end{equation}

The second equality follows from the fact that the activation $\sigma$ is bijective and hence the inputs to $\sigma$ are equal.  We take the derivative of the expressions above w.r.t $x_j$ to get the following (follows from Lemma~\ref{lemma1}). The equality holds true for all $i\leq T$. 

From the above, we can use $i=1$ and obtain

$$A\psi_{0}(x_1,x_1) = B \phi_{0}(x_1,x_1), \forall x_1 \in [0,1]^n.$$

From $i=2$, we obtain 

$$A\psi_{0}(x_2,x_2) + A\psi_{1}(x_2,x_1) = B \phi_{0}(x_2,x_2) + B\phi_{1}(x_2,x_1), \forall x_1\in [0,1]^n, x_2\in [0,1]^n$$

Combining the two conditions we get $$A\psi_{1}(x_2,x_1)= B\phi_{1}(x_2,x_1), \forall x_1\in [0,1]^n, x_2\in [0,1]^n .$$ 

We can use this argument and arrive at  $$A\psi_{i-1}(x_i,x_1)= B\phi_{i-1}(x_i,x_1), \forall x_i\in [0,1]^n, x_1\in [0,1]^n, \forall i \leq T.$$

Thus we obtain 
\begin{equation}
\forall i-j \leq T-1, \forall x_i \in [0,1]^n, x_j \in [0,1]^n,  \hspace{2mm}  A\psi_{i-j}(x_i,x_j) =   B\phi_{i-j}(x_i,x_j).
\label{eqn: psi_phi_eq_pe}
\end{equation}

From Assumption~\ref{assm: perm_inv_pe} and $T\geq T_{\mathsf{max}}$, we already know that 

\begin{equation}
\forall i-j\geq  T, \forall x_i \in [0,1]^n, x_j \in [0,1]^n,  \hspace{2mm}  A\psi_{i-j}(x_i,x_j) =   B\phi_{i-j}(x_i,x_j)=0.
% \label{eqn: psi_phi_eq_pe}
\end{equation}

If $A$ is left invertible, then the above condition implies that linear representation identification is necessary for both compositional and length generalization.

We now consider any sequence $x_{\leq \tilde{T}} \in [0,1]^{n\tilde{T}}$. The prediction made by $h$ is

\begin{equation}
    \begin{split}
      h(x_{\leq \tilde{T}}) =   \sigma \Big(A \frac{1}{\tilde{T}}\sum_{j\leq \tilde{T}} \psi_{\tilde{T}-j}(x_{\tilde{T}}, x_j)\Big) =  \sigma \Big(B \frac{1}{\tilde{T}}\sum_{j\leq \tilde{T}} \phi_{\tilde{T}-j}(x_{\tilde{T}}, x_j)\Big) = f(x_{\leq \tilde{T}})
    \end{split}
\end{equation}

We use \eqref{eqn: psi_phi_eq_pe} in the simplification above. From the above, we can conclude that $h$ continues to be optimal for all sequences in 
$[0,1]^{n\tilde{T}}$.

\end{proof}

\subsubsection{Proof of Theorem~\ref{thm2_transformers}}

\begin{proof}
We start with the same steps as earlier proofs and equate the prediction of $h$ and $f$. We first use the fact $h(x_{\leq i}) = f(x_{\leq i}), \forall i \leq T$ almost everywhere in the support. We can use the continuity of $h,f$ and regular closedness of the support to extend the equality to all points in the support (follows from the first part of Lemma~\ref{lemma1}) to obtain the following. For all $x_{\leq i} \in \mathsf{supp}(X_{\leq i})$ 

\begin{equation}
    \begin{split}
       &  \omega \Big(\sum_{j <  i} \frac{1}{i-1}\cdot \psi(x_i, x_j) \Big) =  \rho\Big(\sum_{j < i} \frac{1}{i-1} \cdot \phi(x_i, x_j) \Big) \implies \\ & \sum_{j < i} \frac{1}{i-1} \psi(x_i, x_j) = \omega^{-1} \circ \rho  \Big(\sum_{j < i} \frac{1}{i-1} \cdot  \phi(x_i, x_j) \Big) \implies   \\
       & \sum_{j < i} \frac{1}{i-1} \psi(x_i, x_j) = a  \Big(\sum_{j <  i} \frac{1}{i-1} \phi(x_i, x_j) \Big), 
    \end{split}
    \label{eqn1_transformers_nl}
\end{equation}

where $a= \omega^{-1}\circ \rho$.  In the above simplification, we used the parametric form for the true labeling function and the learned labeling function and use the invertibility of $\omega$. Let us consider the setting when $i=2$. In that case summation involves only one term. Substitute $x_1=y$ and $x_2=x$. We obtain $\forall x\in [0,1]^n, y \in [0,1]^n,$

\begin{equation}
    \psi(x,y) = a( \phi(x,y)). 
     \label{eqn2_transformers_nl}
\end{equation}

The above expression implies that $\psi$ bijectively identifies $\phi$. Let us consider the setting when $i=3$ (this is possible since $T\geq 3$).  We substitute $x_3=x$, $x_2=y$, $x_1=z$ and obtain

\begin{equation}
    \frac{1}{2} \Big[a (\phi(x, y)) +a (\phi(x, z))\Big]  = a  \big( \frac{1}{2}\big(\phi(x,y) + \phi(x,z)\big) \big).   
      \label{eqn3_transformers_nl}
\end{equation}
 Substitute $\phi(x,y) = \alpha$ and $\phi(x,z) = \beta$. In the simplification that follows, we use the assumption that $[\phi(x,y), \phi(x,z)]$ spans $\mathbb{R}^{2m}$, where $\phi(x,y)$ and $\phi(x,z)$ individually span $\mathbb{R}^{m}$.

\begin{equation}
   \frac{1}{2}(a (\alpha) +a (\beta))  = a  \big( \frac{1}{2}(\alpha  + \beta) \big) . 
   \label{eqn: convex_a_comb}
\end{equation}

Observe that $a(0)=0$ because $\omega^{-1} \circ \rho(0) = 0$ because $\omega^{-1}(0) = \rho(0)= 0$.

\begin{equation}
\begin{split}
&      \frac{1}{2}(a (2\alpha) +a (0))  = a  \big( \frac{1}{2}(2\alpha  + 0) \big)   \\  
& a(2\alpha) = 2 a(\alpha )
\end{split}  
\end{equation}

Next, substitute $\alpha$ with $2\alpha$ and $\beta $ with $2\beta$ in \eqref{eqn: convex_a_comb} to obtain

\begin{equation}
\begin{split}
 &   \frac{1}{2}(a (2\alpha) +a (2\beta))  = a  \big( \frac{1}{2}(2\alpha  + 2\beta) \big) \\ 
&     a(\alpha + \beta) = a(\alpha) + a(\beta)
\label{eqn: linear_comb_A}
\end{split}    
\end{equation}

We use \eqref{eqn: linear_comb_A} to show that $a$ is linear. To show that, we need to argue that $a(c \alpha) = c a(\alpha)$ as we already know $a$ satisfies additivity condition. 

Suppose $c$ is some rational number, i.e., $c = p/q$, where $p$ and $q$ are non-zero integers. 

From the identity it is clear that $a(p \alpha) = p a(\alpha)$, where $p$ is some integer.

$a(q \frac{1}{q} \alpha) = q a(\frac{1}{q}\alpha) \implies a(\frac{1}{q}\alpha) = \frac{1}{q}a(\alpha)$, where $q$ is some integer.

Now combine these $a(p/q \alpha) = p a(1/q\alpha ) = \frac{p}{q}a(\alpha)$. We have established the homogeneity condition for rationals. 

We will now use the continuity of the function $a$ and density of rationals to extend the claim for irrationals. Suppose $c$ is some irrational. Define a sequence of rationals that approach $c$ (this follows from the fact that rationals are dense in $\mathbb{R}$). 

$a(c \alpha) = a(\lim_{n\rightarrow \infty} q_n \alpha) = \lim_{n \rightarrow \infty}  a(q_n \alpha). $

In the second equality above, we use the definition of continuity ($a$ is continuous since composition of continuous functions is continuous). We can also use the property that we already showed for rationals to further simplify

$\lim_{n \rightarrow \infty}  a(q_n \alpha)  = a(\alpha) \lim_{n \rightarrow \infty}  q_n  = ca(\alpha). $

Observe that $a:\mathbb{R}^{m}\rightarrow \mathbb{R}^{m}$ and for any $\alpha, \beta \in \mathbb{R}^{m}$ $a(\alpha+\beta) = a(\alpha) + a(\beta)$ and $a(c\alpha) = c a(\alpha)$.  From the definition of a linear map it follows that $a$ is linear. As a result, we can write $\forall x \in [0,1]^{n}, y\in [0,1]^{n}$

\begin{equation}
     \psi(x,y) = A( \phi(x,y))
     \label{eqn: trans_nl_lin_id}
\end{equation}

Observe that $a$ is invertible because both $\rho$ and $\omega$ are invertible. As a result, we know that $A$ is an invertible matrix. From this we get 

\begin{equation}
  \forall x \in [0,1]^{n}, y\in [0,1]^{n},   \phi(x,y) = A^{-1}\psi(x,y)= C(\psi(x,y))
     \label{eqn: linear_id_txformer}
\end{equation}

For all $z\in \mathbb{R}^{m}$, we obtain

$$a(z) = \rho^{-1}\circ \omega (z) = Cz  \implies \omega(z) = \rho(Cz)$$

Let us consider any sequence $x_{\leq \tilde{T}} \in [0,1]^{n\tilde{T}}$. 
We use the above conditions  
$$\omega\big( \sum_{j < \tilde{T}} \psi(x_{\tilde{T}}, x_j) \big) = \rho ( C\sum_{j < \tilde{T}} \psi(x_{\tilde{T}}, x_j) ) = \rho \big( \sum_{j < \tilde{T}} \phi(x_{\tilde{T}}, x_j) \big). $$ 

 Thus we obtain length and compositional generalization. 

\end{proof}

From \eqref{eqn: trans_nl_lin_id}, we again observe a linear relationship between the learned and the true representation, which implies linear identification.  
\paragraph{On absence of labels at all lengths from $1$ to $T$} We argue that the above proof can be adapted to the setting where we do not observe labels at all lengths from $1$ to $T$. Suppose we only observe label at length $T$. Take equation \eqref{eqn1_transformers_nl} and substitute $x_i=x$ and $x_j=y$ for all $j<i$ to obtain the same condition as equation \eqref{eqn2_transformers_nl}. 
Suppose $T$ is odd and larger than or equal to $3$. Fix $x_i=x$, $x_{2j-1}=y, \forall j \in \{1, \cdots, (T-1)/2\}$, $x_{2j}=z, \forall j \in \{1, \cdots, (T-1)/2\}$. We obtain the same condition as equation \eqref{eqn3_transformers_nl}. Rest of the proof can be adapted using a similar line of reasoning.

\paragraph{Remark on Assumption~\ref{assm: perm_inv_gn}} We require that the support of $[\phi(X_1,X_2), \phi(X_1,X_3)]$ is $\mathbb{R}^{2m}$. This assumption is used in the proof in equation \eqref{eqn: linear_comb_A}. We used this assumption to arrive at $a(\alpha + \beta) =a(\alpha) + a(\beta), \forall \alpha, \beta \in \mathbb{R}^{m}$. We then used continuity of $a$ to conclude $a$ is linear. Now suppose $[\phi(X_1,X_2), \phi(X_1,X_3)]$ is some subset $\mathcal{Z} \subseteq \mathbb{R}^{2m}$. We believe that it is possible to extend the result to more general $\mathcal{Z}$, it might still be possible to arrive at the linearity of $a$. We leave this investigation to future work. 

\subsubsection{Extending Theorem~\ref{thm2_transformers} to incorporate positional encodings} We next present the result when $\omega$ is continuously differentiable and invertible.

\begin{assumption}
Each function in the hypothesis class $\mathcal{H}$ used by the learner is given as $h(x_1,\cdots, x_i) = \omega \Big(\sum_{j\leq i} \psi_{i-j}(x_i, x_j)  \Big)$, where $\omega$ is a $C^{1}$-diffeomorphism. Also, $\psi_{i-j}=0$ for $i-j> T_{\mathsf{max}}-1$, i.e., two tokens that are sufficiently far apart do not interact.  For all $k\leq T_{\mathsf{max}}-1$ each $x \in [0,1]^{n},$ $\exists\; y \in [0,1]^{n}$ where $\psi_{k}(x,y)=0$. 
\label{assm: perm_inv_pe_gn}
\end{assumption}

In the theorem that follows, we require the support of training distribution under consideration is already sufficiently diverse and hence we only seek to prove length generalization guarantees. 

\begin{assumption}
\label{assm: supp2_gn}
    The joint support $\mathsf{supp}(X_{\leq T})=[0,1]^{T}$. The support of $[\phi_1(X_1,X_2), \phi_2(X_1,X_3)]$ is  $\mathbb{R}^{2k}$, where $\phi_{i-j}$ is the embedding function for the labeling function $\rho (\sum_{j\leq i} \phi_{i-j}(x_i,x_j))$. 
\end{assumption}

\begin{restatable}{theorem}{transformers_pe_gn}
    \label{thm_pe_gn}
If $\mathcal{H}$ follows  Assumption~\ref{assm: perm_inv_pe_gn}, the realizability condition holds, i.e., $f\in \mathcal{H}$, Assumption~\ref{assm: supp2_gn}  holds and $T\geq T_{\mathsf{max}}$, then the model trained to minimize the risk in \eqref{eqn: risk_min} (with $T\geq 2$) with $\ell_2$ loss  achieves length generalization. 
\end{restatable}

\begin{proof}
We start with the same steps as earlier proofs and equate the prediction of $h$ and $f$. We first use the fact $h(x_{\leq i}) = f(x_{\leq i})$ almost everywhere in the support. We can use the continuity of $h,f$ and regular closedness of the support to extend the equality to all points in the support (follows from the first part of Lemma~\ref{lemma1}) to obtain the following. For all $x_{\leq i} \in \mathsf{supp}(X_{\leq i})$ 

\begin{equation}
    \begin{split}
       &  \omega \Big(\sum_{j <  i} \frac{1}{i-1}\psi_{i-j}(x_i, x_j) \Big) =  \rho\Big(\sum_{j < i} \frac{1}{i-1}\phi_{i-j}(x_i, x_j) \Big), \\  
       &  \sum_{j < i} \frac{1}{i-1} \psi_{i-j}(x_i, x_j) = \omega^{-1} \circ \rho  \Big(\sum_{j < i} \frac{1}{i-1} \phi_{i-j}(x_i, x_j) \Big),  \\
       & \sum_{j < i} \frac{1}{i-1} \psi_{i-j}(x_i, x_j) = a  \Big(\sum_{j <  i} \frac{1}{i-1} \phi_{i-j}(x_i, x_j) \Big), 
    \end{split}
    \label{eqn0_proof_trans_nl1}
\end{equation}

where $a = \omega^{-1}\circ \rho$.  In the above simplification, we used the parametric form for the true labeling function and the learned labeling function. We also used the invertibility of $\rho$. Let us consider the setting when $i=2$. In that case summation involves only one term. Substitute $x_1=y$ and $x_2=x$. We obtain $\forall x\in [0,1]^n, y \in [0,1]^n,$

\begin{equation}
    \psi_{1}(x,y) = a( \phi_{1}(x,y)). 
    \label{eqn1_proof_trans_nl1}
\end{equation}

For $i=3$, substitute $x_1=x$, $x_3=z$ and set $x_2=y$ in such a way that $\phi_1(x,y)=0$ (follows from Assumption~\ref{assm: perm_inv_pe_gn}). Thus we obtain

\begin{equation}
    \psi_{2}(x,y) = a( \phi_{2}(x,y)). 
    \label{eqn1_proof_trans_nl}
\end{equation}

Similarly, we can  obtain the following. For all $k\leq T_{\mathsf{max}}$
\begin{equation}
    \psi_{k}(x,y) = a( \phi_{k}(x,y)). 
    \label{eqn1_proof_trans_nl1}
\end{equation}

The above expression implies that $\psi$ bijectively identifies $\phi$. Let us consider the setting when $i=3$ (this is possible since $T\geq 3$). We substitute $x_3=x$, $x_2=y$, $x_1=z$ to give

\begin{equation}
    \frac{1}{2}\big(a (\phi_1(x, y)) +a (\phi_2(x, z))\big)  = a \big( \frac{1}{2}(\phi_1(x,y) + \phi_2(x,z)) \big).   
        \label{eqn2_proof_trans_nl1}
\end{equation}

We now use the assumption $[\phi_1(x,y), \phi_2(x,z)]$ spans $\mathbb{R}^{2k}$ and substitute $\phi_1(x,y) = \alpha$ and $\phi_2(x,z) = \beta$

\begin{equation}
  \frac{1}{2} (a (\alpha) +a (\beta))  = a  \big( \frac{1}{2}(\alpha  + \beta) \big) . 
\end{equation}

Rest of the proof follows the same strategy as proof of Theorem~\ref{thm2_transformers}.
\end{proof}

\subsubsection{Extending Theorem~\ref{thm2n}, Theorem~\ref{thm2_transformers} to incorporate multiple attention heads} Our choice of the archictecture did not use multiple attention heads. Let us consider multiple attention heads. First observe that in Theorem~\ref{thm2n} since the output layer is a single layer perceptron, adding multiple attention heads $\psi_1,\psi_2\cdots$ can equivalently be seen as increasing the output dimension of $\psi$, thus the extension to multiple heads follows trivially.  For extending Theorem~\ref{thm2_transformers}, let us consider the model class with two attention heads $\psi_1, \psi_2$ can be stated as follows  $\omega \Big(\sum_{j <  i} A [\psi_1(x_i, x_j), \psi_2(x_i, x_j)]^{\top}  \Big)$, where $A$ combines the outputs of the attention heads linearly. Following the same steps of proof of Theorem~\ref{thm2_transformers}, we obtain the following.

\begin{equation}
    \begin{split}
       &  \omega \Big(\sum_{j <  i} A [\psi_1(x_i, x_j), \psi_2(x_i, x_j)]^{\top}  \Big) =  \rho\Big(\sum_{j < i} B [\phi_1(x_i, x_j), \phi_2(x_i, x_j)]^{\top}  \Big), \\  
       &  \omega \Big(\sum_{j <  i} \tilde{\psi}(x_i,x_j) \Big) =  \rho\Big(\sum_{j < i} \tilde{\phi}(x_i,x_j)  \Big), \\   
       & \sum_{j <  i} \tilde{\psi}(x_i,x_j) = a \Big(\sum_{j < i} \tilde{\phi}(x_i,x_j)  \Big),  
    \end{split}
\end{equation}

where $a=\omega^{-1}\circ \rho$. In the above simplification, the RHS shows the labeling function and the LHS is the function that is learned. We can follow the same strategy as the proof of Theorem~\ref{thm2_transformers} for the rest of the proof. We set $i=2$ and obtain a condition similar to \eqref{eqn2_transformers_nl} and for $i=3$ we obtain a condition similar to \eqref{eqn3_transformers_nl}. Following a similar proof technique, we obtain $a$ is linear and  the proof extends to multiple attention heads.

\subsection{State space models}
\label{sec: ssm_proofs}

\subsection{Vanilla RNNs}
\label{sec: vrnn_proofs}
We provide proofs to Lemma~\ref{lemma2} and \ref{lemma3} that are used to prove Theorem~\ref{thm:vrnn}. Next, in Theorem~\ref{thm4:disc}, we present the discrete token counterpart to Theorem~\ref{thm:vrnn}.

\begin{lemma}
\label{lemma2}
   The $k^{th}$ derivative of sigmoid function denoted $\frac{\partial^k \sigma(s)}{\partial s^k}$ is not zero identically.
\end{lemma}

\begin{proof}
 The first derivative of the sigmoid function $\frac{\partial \sigma(s)}{\partial s} = \sigma(s)(1-\sigma(s))$. 
 We argue that the $\frac{\partial^k \sigma(s)}{\partial s^k}$ is a polynomial in $\sigma(s)$ with degree $k+1$. Consider the base case of $k=1$.  This condition is true as $\frac{\partial \sigma(s)}{\partial s} = \sigma(s)(1-\sigma(s))$. 
Now let us assume that $\frac{\partial^k \sigma(s)}{\partial s^k}$ is a polynomial of degree at most $k+1$ denoted as $P_{k+1}(\sigma(s))$. We simplify 
$$\frac{\partial^k \sigma(s)}{\partial s^k} = P_{k+1}(\sigma(s)) = \sum_{j=1}^{k+1} a_{j} (\sigma(s))^{j}$$

We take another derivative of the term above as follows. 
$$\frac{\partial^{k+1} \sigma(s)}{\partial s^{k+1}} = \frac{\partial P_{k+1}(\sigma(s))}{\partial s} = \sum_{j=1}^{k+1} a_{j} \frac{\partial (\sigma(s))^{j}}{\partial s} = \sum_{j=1}^{k+1} a_{j} j \sigma(s)^{j-1} (\sigma(s)(1-\sigma(s)))$$

Observe that the $\frac{\partial^{k+1} \sigma(s)}{\partial s^{k+1}}$ is also a polynomial in $\sigma(s)$. Observe that the degree $k+2$ term has one term with coefficient $-a_{k+1}\cdot (k+1)$. Since $a_{k+1}\not=0$, the coefficient of degree $k+2$, $-a_{k+1}\cdot (k+1)$, is also non-zero. Since $\frac{\partial^k \sigma(s)}{\partial s^k}$ is a polynomial in $\sigma(s)$ with degree $k+1$ and hence, it cannot be zero identically.

\end{proof}

\begin{lemma}
\label{lemma3}
Let $x\in \mathbb{R}^n$ and $A\in \mathbb{R}^{n\times n}$. Suppose $Ax=0, \forall x \in \mathcal{X}$, where $\mathcal{X}$ has a non-empty interior. Under these conditions $A=0$. 
\end{lemma}

\begin{proof}
Since $\mathcal{X}$ has a non-empty interior, we can construct a $\ell_{\infty}$ ball centered on $\theta$, defined as follows -- $\tilde{\mathcal{X}} = \{\theta + \sum_{j=1}^{n}\alpha_je_j\; | \|\alpha\|_{\infty} \leq \alpha_{\mathsf{max}} \; \}$, where $e_j$ is a vector that is zero in all components and one on the $j^{th}$ component. 
Suppose $A$ was non-zero. One of the columns say $a_j$ is non-zero. Consider two points in the ball $\tilde{\mathcal{X}}$ such that  
 $j^{th}$ coefficients are non-zero but rest of the coefficients are zero. We denote the $j^{th}$ components for the two components as  $\alpha_j$ and $\tilde{\alpha}_j$, where $\alpha_j\not=\tilde{\alpha}_j$.  
 We now plug these two points into the condition that $Ax=0$
 \begin{equation}
 \begin{split}
     &     A(\theta + \alpha_j e_j) = 0 \implies A\theta = \alpha_ja_j, \\ 
     &     A(\theta + \tilde{\alpha}_j e_j) = 0 \implies A\theta = \tilde{\alpha}_ja_j, \\ 
 \end{split}
 \end{equation} 
We take a difference of the two steps above and obtain 
$$(\alpha_j-\tilde{\alpha_j})a_j =0 \implies a_j =0$$
This is a contradiction.  Hence, $A=0$. 
\end{proof}
% \vrnn*

\subsubsection{Extending Theorem~\ref{thm:vrnn} to discrete tokens} 
In our discussion, we have focused on settings where the support of each token has a non-empty interior (Assumption~\ref{assm: reg_closed}). In practice of language modeling, we use discrete tokens and hence Assumption~\ref{assm: reg_closed} does not hold anymore. In this section, we discuss the adaptation of results for vanilla RNNs to setting when the the support of tokens is a finite set.

Define $\mathcal{S} = \{y=Bx \; | \; x \in \mathcal{X}\}$, where $\mathcal{X}$ is the marginal support of each token.

\begin{assumption}
		\label{assm: disc_vrnn}
	 a) For each component $i$ of $y$, $\mathcal{S}$ contains two pairs where the first coordinate differs by the same amount. Mathematically stated, the two pairs are $\Big((y_i, y_{-i}), (y_i+\delta, y_{-i}))\Big)$ and $\Big((y_i^{'}, y_{-i}^{'}), (y_i^{'}+\delta, y_{-i}^{'}))\Big)$. 
  
		b)For every pair of components $i,j$ of $y$, $\mathcal{S}$ contains a point $y$ that satisfies the following. There exists three points in $\mathcal{S}$ such that they only differ in $y_i, y_j$, and form a rectangle, $(y_i, y_j),$  $(y_i^{'}, y_j),$ $(y_i, y_j^{'}),$ $(y_i^{'}, y_j^{'})$. Similarly, there exists another set of points where $y_i^{'}<y_i$ and $y_j^{'}<y_j$. 
 	
\end{assumption}

\begin{theorem}
\label{thm4:disc}
  If  $\mathcal{H}$ follows  Assumption~\ref{assm: vrnn}, and the realizability condition holds, i.e., $f\in \mathcal{H}$ and regular closedness condition in Assumption~\ref{assm: reg_closed} holds, then the model trained to minimize the risk in \eqref{eqn: risk_min} with $\ell_2$ loss (with $T\geq 2$) achieves length  and compositional generalization.
\end{theorem}

\begin{proof}
	We start with the same steps as earlier proofs and equate the prediction of $h$ and $f$ everywhere in the training support. 
	We start with equating label at length 1, i.e., $y_1$. For all $x_1 \in \mathsf{supp}(X_1)$
	\begin{equation}
		\begin{split}
			&  \sigma(A\sigma(B x_1)) = \sigma(\tilde{A} \sigma(\tilde{B}x_1)) \implies       A\sigma(B x_1) = \tilde{A} \sigma(\tilde{B}x_1) \implies \\ 
			&         \tilde{A}^{-1} A\sigma (Bx_1) = \sigma (\tilde{B}  B^{-1} B x_1) 
		\end{split}
	\end{equation}
	Say $y=Bx_1$,  $\tilde{A}^{-1}A=U$, $\tilde{B} B^{-1}=V$.  We substitute these expressions in the simplificaction below.  We pick a $y$ in the interior of $\tilde{B} \cdot \mathsf{supp}(X_1)$.   
	% Take a ball and map it under an invertible transform, it leads to some form of ellipsoid. Consider $\tilde{B}^{-1}$, take inverse image of an open ball from a ball in x space and that should be an open set in $\mathbb{R}^d$
	
	\begin{equation}
		\sigma(Vy) = U \sigma(y)
	\end{equation}
	Take the first row of $V$ and $U$ as $v^{\top}$ and $u^{\top}$ to obtain 
	\begin{equation}
		\sigma(v^{\top}y) = u^{\top} \sigma(y)
		\label{eqn: sigma_vrnn_disc}
	\end{equation}
	 
	Say $v_i\not=0$ and $u_i=0$.  We consider  a $(y_i, y_{-i})$ and $(y_i^{'}, y_{-i})$ satisfying Assumption \ref{assm: disc_vrnn} a.  We substitute these points in \eqref{eqn: sigma_vrnn_disc} and take the difference of the LHS and RHS in \eqref{eqn: sigma_vrnn_disc} to obtain. 
	
	\begin{equation}
		\sigma(v_iy_i^{'} + v_{-i}y_{-i}) - 	\sigma(v_iy_i + v_{-i}y_{-i})  =0 
	\end{equation}
	
$\sigma$ is strictly monotonic and thus the above cannot be true. Similarly, we can rule out the case when $u_i\not=0$ and $v_i=0$. Thus we can deduce that both $u$ and $v$ have same non-zero components.

	Let us start with the case where $p\geq 2$  components of $u,v$ are non-zero. Without loss of generality say the first two components are among coordinates that are non-zero. Pick a $y \in \mathcal{S}$ that satisfies Assumption~\ref{assm: disc_vrnn} b. Suppose $v^{\top}y\geq 0$. We select the neighbors of $y$ that form the rectangle such that each coordinate is greater than $y$.  We substitute these points in \eqref{eqn: sigma_vrnn_disc} and the simplification procedure works as follows.  
	Let 
	
	\begin{equation}
		\begin{split}
		&	s_1 = v_1 y_1^{'} + v_2 y_2^{'} + \cdots v_n y_n,  \;\; s_3 = v_1 y_1^{'} + v_2 y_2 + \cdots v_n y_n \\ 
		&	s_2 = v_1 y_1 + v_2 y_2^{'} + \cdots v_n y_n,  \;\; s_4 = v_1 y_1 + v_2 y_2 + \cdots v_n y_n
		\end{split}
	\end{equation}

	Observe that $s_1>s_2>s_4$ and $s_1>s_3>s_4$. It is possible that $s_2\geq s_3$ or $s_3>s_2$. Suppose $s_2\geq s_3$. 
	
	We can write 
	
	\begin{equation}
		\begin{split}
			 & \sigma(s_1) = u_1 \sigma(y_1^{'}) + u_2\sigma(y_2^{'}) + \cdots + u_n \sigma(y_n), 	\sigma(s_2) = u_1 \sigma(y_1) + u_2\sigma(y_2^{'}) + \cdots + u_n \sigma(y_n)  \\ 
			 &  \sigma(s_3) = u_1 \sigma(y_1^{'}) + u_2\sigma(y_2) + \cdots + u_n \sigma(y_n), 	\sigma(s_4) = u_1 \sigma(y_1) + u_2\sigma(y_2) + \cdots + u_n \sigma(y_n) 
		\end{split}
	\end{equation}
	
	We take a difference of the first two and the latter two, and subtract these differences to get 
	
	\begin{equation}
		\Big( \sigma(s_1) - \sigma(s_2)\Big) - 	\Big( \sigma(s_3) - \sigma(s_4)\Big) = 0
	\end{equation}

	From mean value theorem, we get that $\sigma^{'}(\tilde{s}) = \sigma^{'}(s^{\dagger})$, where $\sigma^{'}$ is the derivative of $\sigma$, $\tilde{s}$ is a value between $s_1$ and $s_2$, and $s^{\dagger}$ is a value between $s_3$ and $s_4$. Since $s_1>s_2>s_3>s_4>0$, $\tilde{s}>s^{\dagger}>0$. Since $\sigma^{'}$ strictly decreases on positive values, the above equality $\sigma^{'}(\tilde{s}) = \sigma^{'}(s^{\dagger})$ is not possible.   Similarly, we can tackle the case $v^{\top} y <0$.

	We are left with the case where $u$ and $v$ have one non-zero component each. From Assumption~\ref{assm: disc_vrnn}a, we select two pairs that differe exactly in the non-zero component. We can resort to dealing with scalars as follows.  We start with first pair $(y,y+\delta)$. 
	
	\begin{equation}
		\begin{split}
		& \sigma(v y) = u \sigma(y)  \implies \frac{1}{1+e^{-vy}} = \frac{u}{1+e^{-y}}  \implies 1-u = ue^{-vy}-e^{-y}\\ 
		& 	\sigma(v (y+\delta)) = u \sigma(y+\delta) \implies  1-u = ue^{-v(y+\delta)}-e^{-(y+\delta)}
	    \end{split}
	\end{equation}
	
	By equating the RHS in the above, we obtain 
	
	\begin{equation}
		\frac{1-e^{-\delta}}{1-e^{-v\delta}} = ue^{-(v-1)y}
		\label{eqn: simp_rnn_1}
	\end{equation}
	
	For the second pair $(y^{'}, y^{'}+\delta)$, we obtain 
	
	\begin{equation}
		\frac{1-e^{-\delta}}{1-e^{-v\delta}} = ue^{-(v-1)y^{'}}
			\label{eqn: simp_rnn_2}
	\end{equation}

If we compare the RHS of \eqref{eqn: simp_rnn_1} and \eqref{eqn: simp_rnn_2}, we obtain $ ue^{-(v-1)y} =  ue^{-(v-1)y^{'}}$. Since $u$ is non-zero, we obtain that $v=1$.  Substituting this into $\sigma(vy) = u \sigma(y)$, we also obtain $u=1$.

	Note that no other row of $U$ or $V$ can have same non-zero element because that would make matrix non invertible. From this we deduce that $U$ and $V$ are permutation matrices. From $ \sigma(Vy) = U \sigma(y)$ it follows that $U=V=\Pi$. Thus $B =\Pi \tilde{B}$ and $\tilde{A} = A\Pi$.

	Next, we equate predictions for $y_2$ to the ground truth (label $y_2$ exists as $T\geq 2$). For all $x_1 \in \mathsf{supp}(X_1)$
	\begin{equation}
		\begin{split} 
			& \sigma(A\sigma(\Lambda \sigma(B x_1) + Bx_2)) = \sigma(\tilde{A} \sigma( \tilde{\Lambda} \sigma(\tilde{B}x_1) + \tilde{B}x_2)) \implies       A\sigma(\Lambda \sigma(B x_1) + Bx_2) = \tilde{A} \sigma( \tilde{\Lambda} \sigma(\tilde{B}x_1) + \tilde{B}x_2) \implies \\ 
			& \tilde{A} \sigma( \tilde{\Lambda} \sigma(\tilde{B}x_1) + \tilde{B}x_2) = A \Pi  \sigma(\tilde{\Lambda} \Pi^{\top}\sigma(Bx_1) + \Pi^{\top}Bx_2)= A \sigma(\Pi \tilde{\Lambda} \Pi^{\top}\sigma(Bx_1) + Bx_2). \\ 
		\end{split}
	\end{equation}
	We use the simplification in the second step to equate to LHS in the first step as follows. 
	
	\begin{equation}
		\begin{split}
			& A \sigma(\Pi \tilde{\Lambda} \Pi^{\top}\sigma(Bx_1) + Bx_2) = A\sigma(\Lambda\sigma(B x_1) + Bx_2) \\     
			&  \implies (\Pi \tilde{\Lambda} \Pi^{\top} - \Lambda)\sigma(Bx_1)  = 0.  \\ 
		\end{split}
	\end{equation}
	
	Since $\sigma(Bx_1)$ spans a set that has a non-empty interior, we get that $\tilde{\Lambda} = \Pi^{\top} \Lambda \Pi $ (from Lemma~\ref{lemma3}).

	From the above conditions, we have arrived at $\tilde{\Lambda} =\Pi^{\top} \Lambda \Pi, \tilde{B}= \Pi^{\top}B, \tilde{A} = A\Pi$. 
	
	We want to show that for all $k\geq 1$
	
	\begin{equation}
		h_{k} = \Pi\tilde{h}_{k}, 
		\label{eqn: perm_id_rnn}
	\end{equation}
	where $h_k = \sigma(\Lambda h_{k-1} + Bx_k)$ and $\tilde{h}_k = \sigma(\tilde{\Lambda}\tilde{h}_{k-1} + \tilde{B}x_k)$ and $h_{0}=\tilde{h}_0=0$. In other words, we define $T_k$ as a mapping that takes $x_{\leq k}$ as input and outputs $h_k$, i.e., $T_k(x_{\leq k})= h_k$. Similarly, we write $\tilde{T}_k(x_{\leq k}) = \tilde{h}_k$. We want to show 
	\begin{equation}
		T_k = \Pi \tilde{T}_k,  \forall k
		\label{eqn: tk}
	\end{equation}

	We show the above by principle of induction. Let us consider the base case below. For all $x_1 \in \mathbb{R}^{n}$
	
	\begin{equation}
		\tilde{A} \sigma(\tilde{B}x_1) =   A \Pi \sigma(\Pi^{\top}B x_1) = A\sigma(Bx_1)  = Ah_1 \implies h_1 = \Pi \tilde{h}_1 \implies T_1(x_1) = \Pi \tilde{T}_1(x_{1}) 
	\end{equation}

	Suppose $\forall j \leq k, T_j = \Pi \tilde{T}_{j}$.

	Having shown the base case and assumed the condition for $j\leq k$, we now consider the mapping $\tilde{T}_{k+1}$
	\begin{equation}
		\Pi\tilde{T}_{k+1}(x_{\leq k+1 }) = \Pi\sigma(\tilde{\Lambda} \tilde{h}_k + \tilde{B}x_{k+1}) = \Pi \sigma(\Pi^{\top}\Lambda \Pi \tilde{h}_k + \Pi^{\top} Bx_k) =  \sigma(\Lambda h_k +  Bx_k) =  T_{k+1}(x_{\leq k+1 }).
		\label{eqn: induction_k_vrnn_disc}
	\end{equation}

	The prediction from the model $(\tilde{A}, \tilde{\Lambda}, \tilde{B})$ at a time step $k$ is denoted as $\tilde{y}_k$ and it relates to $\tilde{h}_k$ as follows $\tilde{y}_k=  \sigma(\tilde{A}\tilde{h}_k)$. We use the above condition in equation~\eqref{eqn: tk} to arrive at the following result. For all $x_{\leq k} \in \mathcal{X}^{k}$
	
	$\tilde{y}_k = \sigma(\tilde{A}\tilde{h}_k) = \sigma(\tilde{A} \tilde{T}(x_{\leq k})) = \sigma(A \Pi \tilde{T}(x_{\leq k})) = \sigma(A T (x_{\leq k})) = y_k $

	This completes the proof. 
\end{proof}

\subsection{Role of chain-of-thought in length generalization}
\label{sec:cot_appendix}

\begin{assumption}
\label{assm: cot_deepset}
Each function $h\in \mathcal{H}$ is of the form $h(x_1,\cdots, x_i) =  \omega \Big(\sum_{j\leq i} \psi(x_j) \Big)$,  $\psi:\mathbb{R}^{n} \rightarrow \mathbb{R}^m$ is a map that is differentiable everywhere.  
Also, there exists $t \leq T$, such that $\mathsf{supp}(x_j)=\mathbb{R}^n$ and $\mathsf{supp}(\sum_{j\leq i} \phi(x_j))=\mathbb{R}^k$. 
\end{assumption}

\begin{theorem}

If $\mathcal{H}$ follows the first part of Assumption~\ref{assm: cot_deepset} and the support of the embeddings follows the second part of Assumption~\ref{assm: cot_deepset}, then a learner trained to minimize prediction error w.r.t both true labels and CoT achieves length generalization.
\end{theorem}

\begin{proof}
The expected risk minimization objective under both forms of supervision is $$\sum_{i\leq T}\mathbb{E}\Big[\Big(\omega \Big(\sum_{j\leq i} \psi(x_j) \Big)-y_i\Big)^2\Big] +\mathbb{E}\Big[\Big(\sum_{j\leq i} \psi(x_j)-h_i\Big)^2\Big], $$ 

where $y_i$ is the true label and $h_i$ is the intermediate embedding $\sum_{j \leq i }\phi(x_j)$. 
Under perfect minimization of the above objective, we obtain 
$\sum_{j\leq i} \psi(x_j) = \sum_{j\leq i} \phi(x_j)$ and $\omega\Big(\sum_{j\leq t} \psi(x_j)\Big) = \rho\Big(\sum_{j\leq t} \phi(x_j)\Big)$. Select $t$ from Assumption~\ref{assm: cot_deepset}, take the gradient w.r.t some $x_r$ to conclude that $\psi(x_r) = \phi(x_r), \forall x_r \in \mathbb{R}^n$ following the same steps as equation \eqref{eqn: psi_phi_eq_deepset} with $A=B= I$, where $I$ is identity matrix.
From equating the labels part, we obtain $\omega\Big(\sum_{j\leq t} \psi(x_j)\Big) = \rho\Big(\sum_{j\leq t} \phi(x_j)\Big)$. Substitute $\sum_{j\leq t} \psi(x_j) = \sum_{j\leq t} \phi(x_j) =h_t$ to obtain $\omega(h_t) = \rho(h_t)$ for all $h_t \in \mathbb{R}^n$. 

Now let us consider any new length $\tilde{T}>T$, it follows that $\sum_{j\leq \tilde{T}} \psi(x_j) = \sum_{j\leq \tilde{T}} \phi(x_j)$. Also, since $\omega = \rho$, it follows that $\omega(\sum_{j\leq \tilde{T}} \psi(x_j)) = \rho(\sum_{j\leq \tilde{T}}\phi(x_j))$.

\end{proof}

The argument for transformers can be adapted exactly on the same lines as above, which is why we straight turn to SSMs.

\begin{assumption}
\begin{itemize}
	\item Each function in the hypothesis class $\mathcal{H}$ takes a sequence $\{x_1,\cdots, x_i\}$ as input and output $h(x_1, \cdots, x_{i}) = \omega \Big(\sum_{j=0}^{i-1}\Lambda^{j} B x_{i-j} \Big)$. 
   \item  There exists $t \leq T$, such that  $\mathsf{rank}([B, \Lambda B, \cdots, \Lambda^{t}B])=k$ and $\mathsf{supp}(X_i)=\mathbb{R}^n, \forall i \leq t$. Further, a subset $\mathcal{S}$ of training data satisfies $\mathsf{rank}([h_{t-1}^{(j)}\; ; x_t^{(j)}]_{j \in \mathcal{S}})= k+n$.
\end{itemize}
	\label{assm: ssm_cot}
\end{assumption}

\begin{theorem}
If $\mathcal{H}$ satisfies the first part of Assumption~\ref{assm: ssm_cot}, and the diversity assumption in second part of Assumption~\ref{assm: ssm_cot} is satisfied, then a learner trained to minimize prediction error w.r.t both true labels and CoT achieves length generalization.

\end{theorem}

\begin{proof}
    The expected risk minimization objective under both forms of supervision is $$\sum_{i\leq T}\mathbb{E}\Big[\Big(\omega \Big(\sum_{j=0}^{i-1}\tilde{\Lambda}^{j} \tilde{B} x_{i-j} \Big)-y_i\Big)^2\Big] +\mathbb{E}\Big[\Big(\sum_{j=0}^{i-1}\tilde{\Lambda}^{j} \tilde{B} x_{i-j} -h_i\Big)^2\Big], $$ 
    where $y_i$ is the true label, and $h_i$ is the hidden state. 

Under perfect minimization of the above objective, we obtain 
$\tilde{\Lambda} h_{t-1} + \tilde{B}x_t =\Lambda h_{t-1} + Bx_t \implies  [h_{t-1}^{\top}, x_t^{\top}][\Lambda-\tilde{\Lambda} ;B -\tilde{B}] =0$. Under the rank condition in Assumption~\ref{assm: ssm_cot}, we obtain $\Lambda = \tilde{\Lambda}$ and $B=\tilde{B}$. 

From equating the labels part, we obtain $\omega\Big(\sum_{j=0}^{i-1}\tilde{\Lambda}^{j} \tilde{B} x_{i-j}\Big) = \rho\Big(\sum_{j=0}^{i-1}\Lambda^{j} B x_{i-j}\Big)$. Since $\tilde{\Lambda} = \Lambda$, $\tilde{B}=B$, we obtain $\omega(h_t) = \rho(h_t)$. From the rank condition in Assumption~\ref{assm: ssm_cot}, we obtain that the set of all possible values $h_t$ is $\mathbb{R}^k$. From this we obtain that $\omega=\rho$.

Now let us consider any length  $\tilde{T}>T$, it follows that $\sum_{j=0}^{\tilde{T}-1}\tilde{\Lambda}^{j} \tilde{B} x_{i-j} = \sum_{j=0}^{\tilde{T}-1}\Lambda^{j} B x_{i-j}$. Also, since $\omega = \rho$, it follows that $\omega(\sum_{j=0}^{\tilde{T}-1}\tilde{\Lambda}^{j} \tilde{B} x_{i-j}) = \rho(\sum_{j=0}^{\tilde{T}-1}\Lambda^{j} B x_{i-j})$.

\end{proof}

\begin{assumption}
\label{assm: rnn_cot}
\begin{itemize}
\item Each function in the hypothesis class $\mathcal{H}$  is a vanilla RNN of the form \eqref{eqn: vrnn}, where the position-wise non-linearity that acts on hidden state is a single layer perceptron $\sigma \circ A$,  and $\Lambda, B$ govern the hidden state dynamics (\eqref{eqn: vrnn}).

\item A subset $\mathcal{S}$ of training data satisfies $\mathsf{rank}([h_{t-1}^{(j)}\; ; x_t^{(j)}]_{j \in \mathcal{S}})= k+n$.
\item A subset $\mathcal{R}$ of training data satisfies $\mathsf{rank}([h_t^{(j)}]_{j \in \mathcal{R}})= k$
\end{itemize}
\end{assumption}

\begin{theorem}
    If $\mathcal{H}$ satisfies the first part of Assumption~\ref{assm: rnn_cot}, and the diversity assumption in second and third parts of Assumption~\ref{assm: rnn_cot} is satisfied, then a learner trained to minimize prediction error w.r.t both true labels and CoT achieves length generalization.
\end{theorem}

\begin{proof}
    The expected risk minimization objective under both forms of supervision is $$\sum_{i\leq T}\mathbb{E}\Big[\Big(\omega(\tilde{h}_i)-y_i\Big)^2\Big] +\mathbb{E}\Big[\Big(\tilde{h}_i -h_i\Big)^2\Big], $$ 
    where $y_i$ is the true label, and $h_i$ is the hidden state. 

Under perfect minimization of the above objective, we obtain 
$\tilde{h}_t = h_t \implies  \sigma(\tilde{\Lambda} h_{t-1} + \tilde{B}x_t) =\sigma(\Lambda h_{t-1} + Bx_t) \implies  [h_{t-1}^{\top}, x_t^{\top}][\Lambda-\tilde{\Lambda} ;B -\tilde{B}] =0$, where we use bijectivity of sigmoid activation $\sigma$.    Under the rank condition in Assumption~\ref{assm: rnn_cot}, we obtain $\Lambda = \tilde{\Lambda}$ and $B=\tilde{B}$.  From equating the labels part, we obtain $\sigma(A h_t) = \sigma(\tilde{A} \tilde{h}_t)$. Substituting $\tilde{h}_t = h_t$ leads to $(A-\tilde{A}) h_t = 0$.  From the third part of the Assumption~\ref{assm: rnn_cot}, it follows that  $A=\tilde{A}$. Owing to the exact match of the parameters between the learned and the true model length generalization follows. 
\end{proof}

\subsection{Rademacher complexity for hypothesis classes satisfying Assumption~\ref{assm: vrnn}}
\label{sec: rademacher_assm10}

Let $S = \{\{(x_{\leq \tau}^{i}, y_{\tau}^{i})\}_{\tau=1}^{T}\}_{i=1}^{m}$ be the set of $m$ sequences of length $t$.  Let $z_i^t=(x_{\leq t}^{i}, y_t^{i})$. Suppose $h$ is parametrized by $\theta$ and it follows Assumption~\ref{asmm: Lipschitz}. Next, we will compute the Rademacher complexity of $g(x_{\leq t}, y_t) = \ell(h(x_{\leq t}; \theta), y_t)$. Define $\mathcal{G}$ as the set of all the functions $g$. If $\ell$ is Lipschitz, we can show that $g$ is Lipschitz in $\theta$. 

\begin{equation}
\begin{split}
  |g(z_t; \theta) - g(z_t; \theta') |=  |\ell(h(x_{\leq t}; \theta), y_t) - \ell(h(x_{\leq t}; \theta'), y_t)|  &  \leq L_{\ell}\|h(x_{\leq t}; \theta)- h(x_{\leq t}; \theta')\| \\ 
  & \leq L_{\ell}L\|\theta-\theta'\|
\end{split}
\end{equation}

Recall that the Rademacher complexity is defined as 

$\mathcal{R}(\mathcal{G} \circ S) = \frac{1}{mT} \mathbb{E}_{\boldsymbol{\sigma}=\{\pm 1\}^m}\Big[ \sup_{g \in \mathcal{G}}\sum_{i,t} \sigma_i g(z_i^{t};\theta)\Big]$

Let us consider $\frac{1}{mT}\sum_{i,t} \sigma_i g(z_i^t)$.  Suppose that $\Theta_c$ is the $\eta$-cover of $\Theta$. Denote the set of functions in $\mathcal{G}$ corresponding to the elements in $\Theta_c$ as $\mathcal{G}_c$.

Set $\eta \leq \frac{\epsilon}{L L_{\ell}}$. For parameter $\theta$, select $\theta_c$ from the cover which is within $\eta$ distance of $\theta$.

$$|\frac{1}{mT}\sum_{i,t} \sigma_i g(z_i^t;\theta) - \frac{1}{mT}\sum_{i,t} \sigma_i g(z_i^t;\theta_c)| \leq  \frac{1}{mT}\sum_{i,t} | g(z_i^t;\theta) -  g(z_i^t;\theta')| \leq \epsilon $$

From the above, we obtain the following. For every $\theta$ there is a $\theta'\in \Theta_c$

$$\frac{1}{mT}\sum_{i,t} \sigma_i g(z_i^t;\theta) \leq \frac{1}{mT}\sum_{i,t} \sigma_i g(z_i^t;\theta') + \epsilon  $$

$$\forall \theta \in \Theta, \frac{1}{mT}\sum_{i,t} \sigma_i g(z_i^t;\theta) \leq \sup_{\theta^{'}\in \Theta_c}\frac{1}{mT}\sum_{i,t} \sigma_i g(z_i^t;\theta') + \epsilon  $$

$$ \sup_{\theta\in \Theta}\frac{1}{mT}\sum_{i,t} \sigma_i g(z_i^t;\theta) \leq \sup_{\theta^{'}\in \Theta_c}\frac{1}{mT}\sum_{i,t} \sigma_i g(z_i^t;\theta') + \epsilon  $$

From the above we obtain
\begin{equation}
\mathcal{R}(\mathcal{G} \circ S)\leq \mathcal{R}(\mathcal{G}_c \circ S) + \epsilon 
\label{eqn1:rad}
\end{equation}

Let $a = [g(z_1^1), g(z_1^2),  \cdots, g(z_m^T)]$, where $a$  is the vector of loss across all data points for hypothesis $g$. If each $h \in \mathcal{H}$ is bounded, the realizability condition holds, that is, loss $f\in \mathcal{H}$ and $\ell = \ell_2$, then $g$ is also bounded. Observe that $g(x_{\leq t},y) = \|h(x_{\leq t}) -y \| \leq \| h(x_{\leq t})\| + \|y\|$. Under realizability conditions, the label is bounded as well. As a result, $\|g\|\leq c$. From this we obtain that $\|a\| \leq c\sqrt{mT}$.    Let $\mathcal{A}$ be the set of all the possible vectors computed from $\mathcal{G}$ and let $\mathcal{A}_c$ be the set of all the possible vectors computed from $\mathcal{G}_c$.  

From Massart lemma,  $\mathcal{R}(\mathcal{G}_c \circ S) \leq \max_{a \in \mathcal{A}_c}\|a -\bar{a} \| \frac{\sqrt{\log(N)}}{mT}$, where $\bar{a} = \frac{1}{N}\sum_{i=1}^{N}a_i$ and $N$ is the size of the cover. We can simplify this as follows.

$$\mathcal{R}(\mathcal{G}_c \circ S) \leq \max_{a \in \mathcal{A}_c}(\|a\| + \|\bar{a} \|)\frac{\sqrt{\log(N)}}{mT} \leq 2c \sqrt{\frac{\log(N)}{mT}}$$

Combining the above with \eqref{eqn1:rad}

\begin{equation}
    \mathcal{R}(\mathcal{G} \circ S) \leq 2c \sqrt{\frac{\log(N)}{mT}} + \epsilon
\end{equation}

Since the size of the cover $N$ is finite, Rademacher complexity decreases with increase in the number of samples $m$.

\section{Experiments}
\label{sec: experiments_details}
Here we provide additional experimental results as well as the training details.

\paragraph{Model Architecture}
In all the architectures, there are two types of non-linearities, $\omega$ that generates the target label, $\psi$ that operates on inputs (used in deep sets and transformers).  We use MLPs to implement these non-linearities. We instantiate MLPs with $l$ hidden layers, and the input, output, and hidden dimensions are all the same $m=n=k$. Recall that under the realizability assumption $f\in \mathcal{H}$. Therefore, we need to select the labeling function from $\mathcal{H}$. To do so, the weights of MLP are initialized according to $\mathcal{N}(\mu,\sigma^2)$, where $\mu=0.0, \sigma=0.6$. For RNNs and SSMs, $A,B,\Lambda$ are initialized separately for the learner and true generating process as orthogonal matrices. All hidden layers, as well as the output layer are followed by a sigmoidal activation function.

\paragraph{Training Details and Hyperparameter Selection}
\label{appendix: hp_search}
We train all models with AdamW optimizer \citep{loshchilov2018decoupled} with a learning rate of $10^{-3}$, weight decay of $0.01$, $\epsilon=10^{-8}, \beta_1=0.9, \beta_2=0.95$. We reduce the learning rate by a factor of $0.8$ if the validation loss is not improved more than $10^{-6}$ for 1 epoch. This drop is followed by a cool-down period of 1 epoch, and the learning rate cannot decrease to lower than $10^{-7}$. For all datasets we use a streaming dataset where each epoch contains 100 batches of size 256 sampled online from the specified training and test distributions, and we train all models for 100 epochs.  Therefore, the size of the training dataset is $256\times 10^4$ and the size of the testing dataset is $256\times 10^2$. Since our models are generally small, running the experiments is rather inexpensive, and we carried out each experiment on 4 CPU cores using 20 GB of RAM. For inference, specially for SSM and RNN with very long sequences, we use RTX8000 GPUs.

\subsection{Failure Cases}
\label{sec:failure_cases}

So far we focused on  success scenarios for length and compositional generalization, here we provide examples to show how a model might fail.

\paragraph{$f$ is realizable in a high capacity $\mathcal{H}$} For a given $\mathcal{H}$, if all solutions to \ref{eqn: risk_min} achieve length generalization or compositional generalization, then we can guarantee length or compositional generalization regardless of the training procedure. When the capacity of $\mathcal{H}$ becomes very large, it continues to contain the right solutions but it starts to contain many incorrect solutions that match the true solution only on the support of training distribution. In such a case, there is no reason to presume that our learning procedure picks the right solution to \ref{eqn: risk_min} that also achieves length and compositional generalization. Figure~\ref{fig:lg-failure-arbitrary-expressive-H} show experiments illustrating the above. We experiment with the following scenarios for deep sets and transformers:
\begin{itemize}
    \item Deep set: We use the labeling function that takes the following form $f=\rho(\sum_{i\leq t}\phi(x_i))$ for $t\leq T$ and $f=\rho(\sum_{i\leq t}\phi(x_i))+c$ for $t>T$ with $c=0.2, T=5$. We use 1 hidden layer MLPs for $\rho,\phi$ (with no activation on the output of $\rho$). We use 2 hidden layer MLPs for $\omega,\psi$ for $h$ so that it can express the above labeling function. The input, hidden, and output dimensions are all equal $m=n=k=20$ for $f, h$. We train on sequences of length longer than $T$ to demonstrate this expressivity claim. When the model is trained on sequences of length less than $T$, due to the simplicity bias of the training procedure model learns $\rho(\sum_{i\leq t}\phi(x_i))$ and uses it on longer sequences and hence fails.
    \item Transformer: We use the labeling function that takes the following form $f=\rho(\sum_{j=1}^{i}\frac{1}{i}\phi(x_i,x_j))$ for $t\leq T$ and $f=\rho(\sum_{j=1}^{i}\frac{1}{i}\phi(x_i,x_j))+c$ for $t>T$ with $c=0.1, T=10$. We use 1 hidden layer MLPs for $\rho$ (with no activation on the output of $\rho$). We use a Transformer with 3 hidden layer MLPs for $\omega$ so that it can express the above labeling function. The input, hidden, and output dimensions are all equal $m=n=k=20$ for $f, h$. We train on sequences of length longer than $T$ to demonstrate this expressivity claim. When the model is trained on sequences of length less than $T$, due to the simplicity bias of the training procedure model learns $f=\rho(\sum_{j=1}^{i}\frac{1}{i}\phi(x_i,x_j))$ and uses it on longer sequences and hence fails.
\end{itemize}
The failures of such degenerate solutions can be visualized in Figure~\ref{fig:lg-failure-arbitrary-expressive-H} (right), where the predictions diverge from the true values when the model is only trained on sequences shorter than $T_0$. Figure~\ref{fig:lg-failure-arbitrary-expressive-H} (left) shows that when the model is trained on sequences longer than $T_0$, it can successfully generalize to longer lengths. Table~\ref{tab:lg-failure-arbitrary-expressive-H} further validates this observation numerically. It presents the test loss of each model at lengths shorter and longer than $T_0$ under the two training schemes: a) When trained only on sequences of length shorter than $T_0$ (rows corresponding to Fig~\ref{fig:deepset-arbitrary-expressive-failure} and Fig~\ref{fig:transformer-arbitrary-expressive-failure} which result in failure due to degenerate solution), b) when trained on sequences of length longer than $T_0$ (rows corresponding to Fig~\ref{fig:deepset-arbitrary-expressive-success} and Fig~\ref{fig:transformer-arbitrary-expressive-success} which result in successful generalization).

\begin{table}[]
\centering
\begin{tabular}{lcc}
\hline
Deep set & Loss ($t<T_0$) &	Loss ($t\geq T_0$) \\
Fig~\ref{fig:deepset-arbitrary-expressive-success} & $0.001\pm 10^{-4}$ & $0.002\pm 3\times10^{-4}$ \\
Fig~\ref{fig:deepset-arbitrary-expressive-failure}	& $0.0007\pm 10^{-4}$ & $0.007\pm0.001$ \\
\hline\hline
Transformer & Loss ($t<T_0$) &	Loss ($t\geq T_0$) \\
Fig~\ref{fig:transformer-arbitrary-expressive-success}	& $0.0006 \pm10^{-4}$ &	$0.006 \pm3\times10^{-3}$ \\
Fig~\ref{fig:transformer-arbitrary-expressive-failure}	& $10^{-5}\pm10^{-6}$ & $0.01\pm0.003$\\
\hline\\
\end{tabular}
\caption{Length generalization of different architectures when the hypothesis class $\mathcal{H}$ is highly expressive. For further details see Fig.~\ref{fig:lg-failure-arbitrary-expressive-H}}
\label{tab:lg-failure-arbitrary-expressive-H}
\end{table}

\begin{figure}
    \centering
    \begin{subfigure}{0.45\textwidth}
        \centering
        \includegraphics[width=\textwidth]{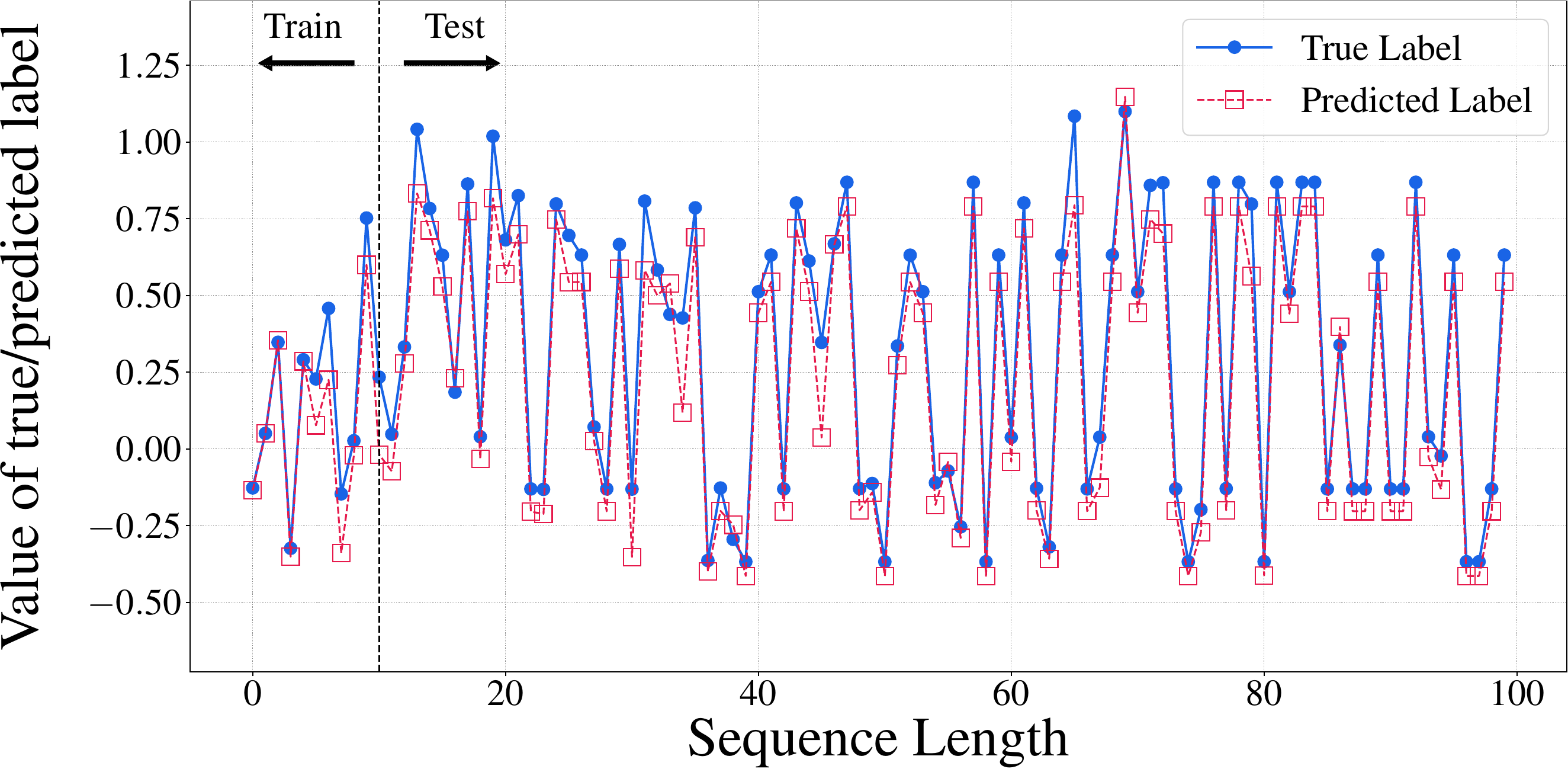}
        \caption{}
        \label{fig:deepset-arbitrary-expressive-success}
    \end{subfigure}
    \hfill
    \begin{subfigure}{0.45\textwidth}
        \centering
        \includegraphics[width=\textwidth]{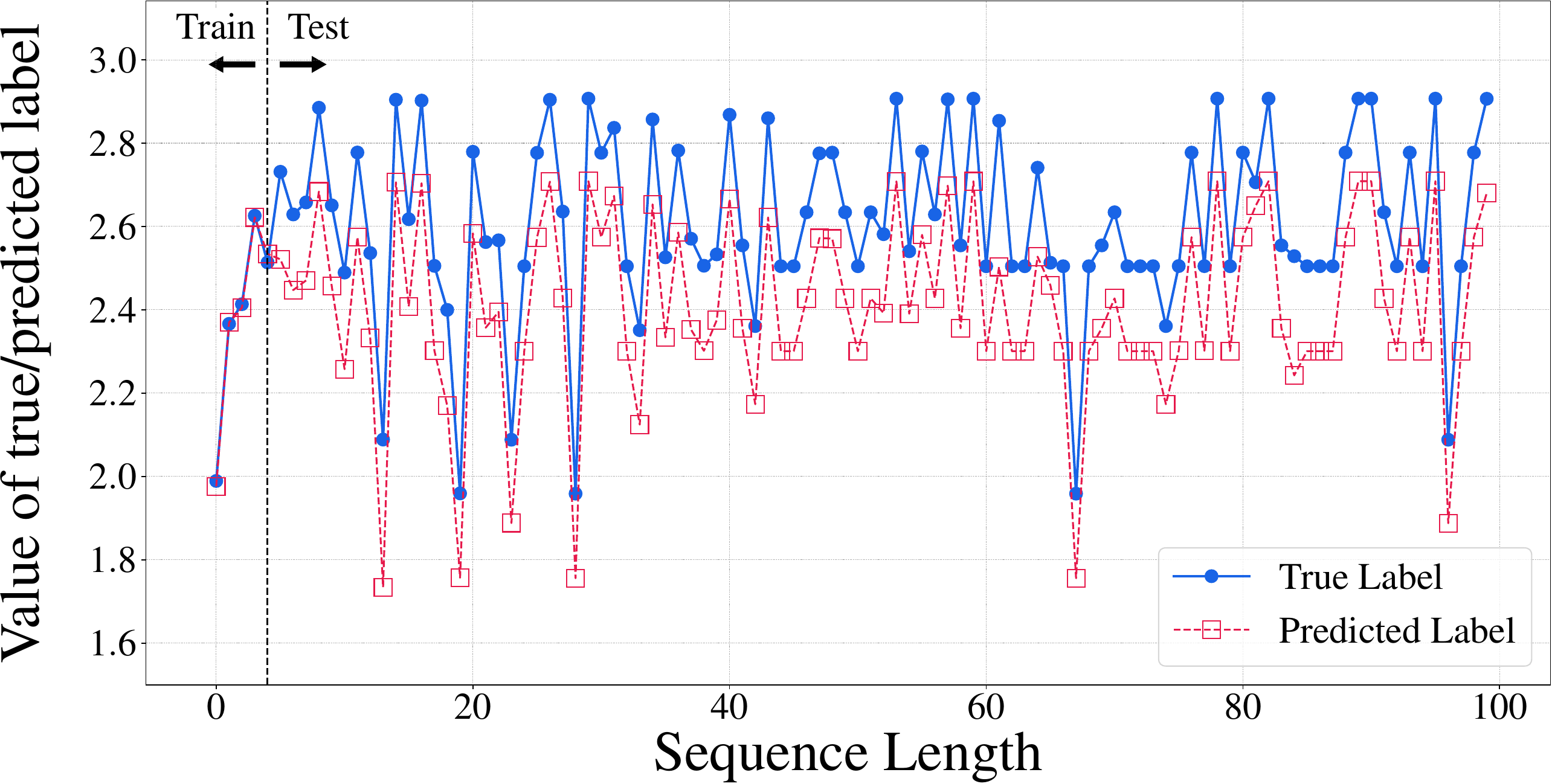}
        \caption{}
        \label{fig:deepset-arbitrary-expressive-failure}
    \end{subfigure}
    \label{fig:arbitrary-expressive-deepset}
    \begin{subfigure}{0.45\textwidth}
        \centering
        \includegraphics[width=\textwidth]{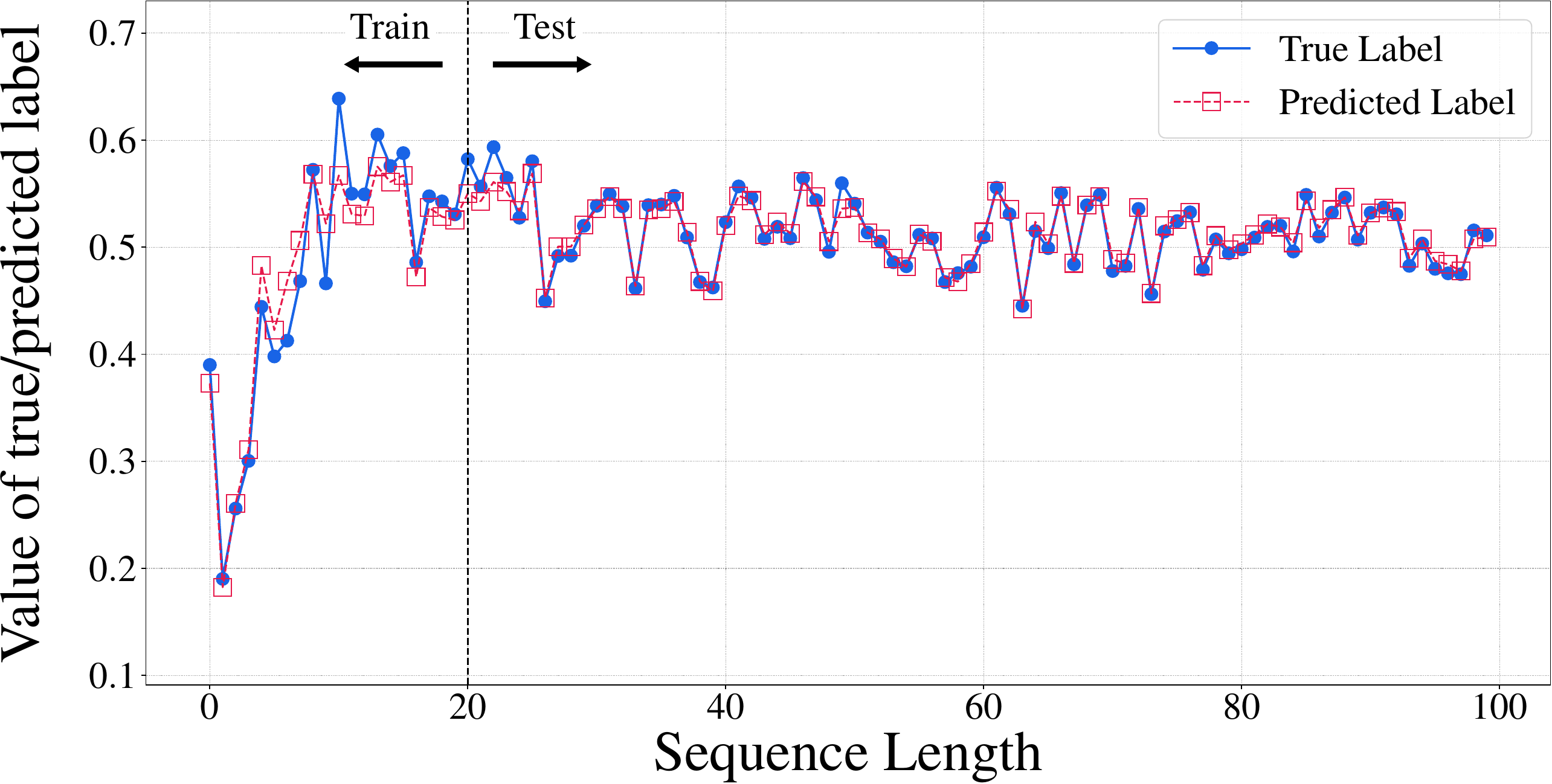}
        \caption{}
        \label{fig:transformer-arbitrary-expressive-success}
    \end{subfigure}
    \hfill
    \begin{subfigure}{0.45\textwidth}
        \centering
        \includegraphics[width=\textwidth]{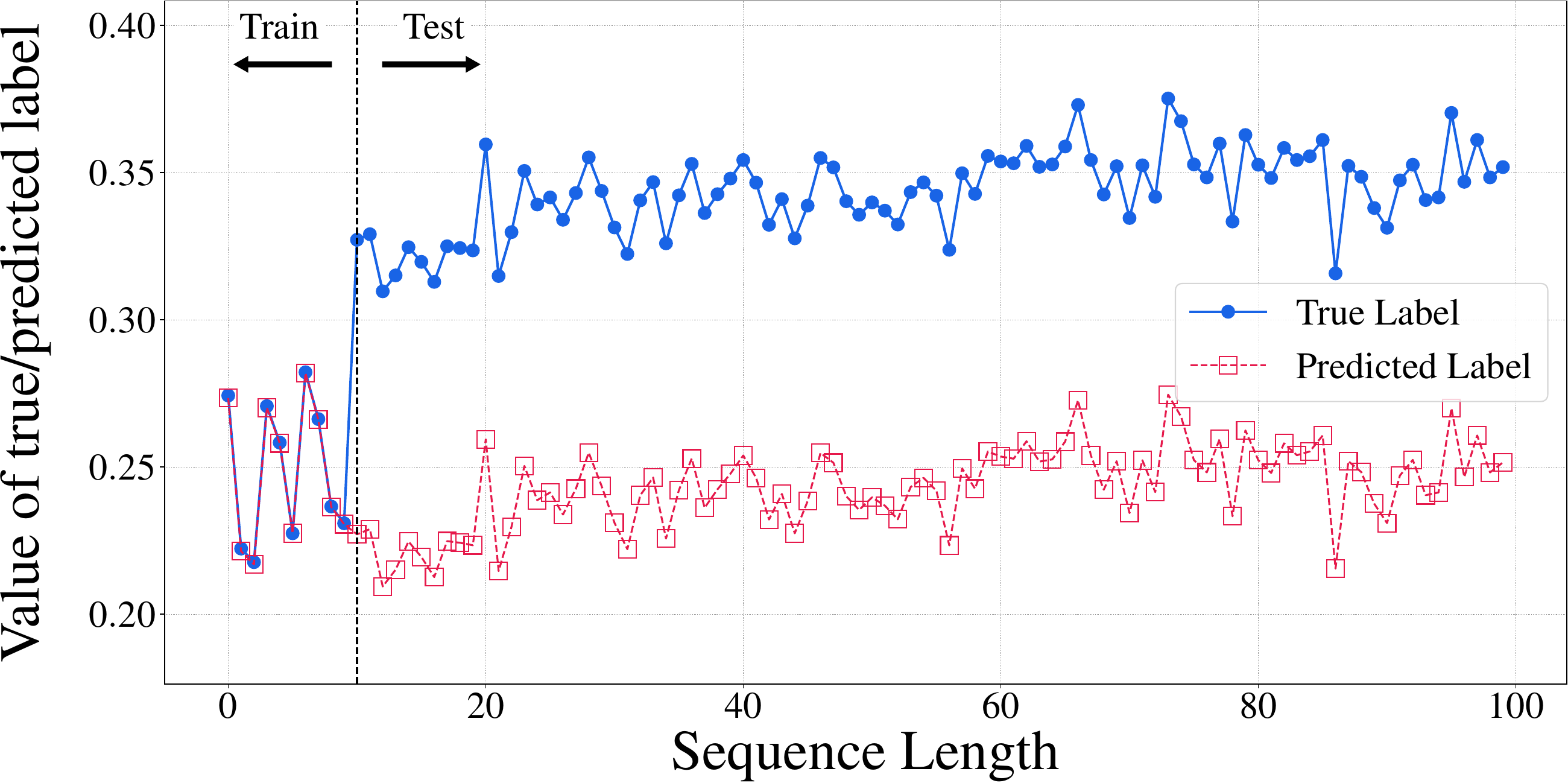}
        \caption{}
        \label{fig:transformer-arbitrary-expressive-failure}
    \end{subfigure}
    \caption{A failure case of length generalization under arbitrary expressive generative model with (a,b) Deep sets, (c,d) and Transformer. The generative function on both cases introduces an offset to sequences longer than some critical length ($T_0$). The learner is once trained on sequences longer than $T_0$ and successfully generalizes (a,c), and once is trained only on sequences shorter than $T_0$ where the offset never appears, and hence fails to generalize beyond that.}
    \label{fig:lg-failure-arbitrary-expressive-H}
\end{figure}

\end{document}